\documentclass[journal]{IEEEtran}
\usepackage{amsmath,amsfonts,amsthm}
\usepackage{algorithm}
\usepackage{algorithmic}
\usepackage{array}
\usepackage[caption=false,font=normalsize,labelfont=sf,textfont=sf]{subfig}
\usepackage{textcomp}
\usepackage{stfloats}
\usepackage{url}
\usepackage{verbatim}
\usepackage{graphicx}
\usepackage{xcolor}
\usepackage{footmisc}
\usepackage{lineno}

\usepackage{mdwmath}
\usepackage{mdwtab}
\usepackage{multirow}
\usepackage{multicol}
\usepackage{booktabs}

\usepackage[numbers]{natbib}
\def\BibTeX{{\rm B\kern-.05em{\sc i\kern-.025em b}\kern-.08em
    T\kern-.1667em\lower.7ex\hbox{E}\kern-.125emX}}
\usepackage{balance}

\newif\ifincludesupplementary
\includesupplementarytrue 

\usepackage{xr}
\newtheorem{theorem}{Theorem}
\newtheorem{lemma}{Lemma}

\newtheorem{remark}{Remark}
\newtheorem{proposition}{Proposition}
\newtheorem{definition}{Definition}
\newtheorem{assumption}{Assumption}
\newtheorem{corollary}{Corollary}

\newcommand{\bE}{{\bf E}}

\newcommand{\bA}{{\bf A}}
\newcommand{\bX}{{\bf X}}
\newcommand{\bD}{{\bf D}}
\newcommand{\bbarD}{\bar{\bf D}}
\newcommand{\btiD}{{\tilde{\bf D}}}
\newcommand{\btiX}{{\tilde{\bf X}}}
\newcommand{\bbarX}{{\bar{\bf X}}}
\newcommand{\barY}{{\bar Y}}
\newcommand{\tiX}{{\tilde X}}
\newcommand{\tiY}{{\tilde Y}}

\newcommand{\tih}{{\tilde h}}
\newcommand{\barh}{{\bar h}}

\newcommand{\calH}{{\cal H}}
\newcommand{\calX}{{\cal X}}
\newcommand{\calY}{{\cal Y}}
\newcommand{\calA}{{\cal A}}

\begin{document}
\title{Task-tailored Pre-processing: \\Fair Downstream Supervised Learning}
\author{Jinwon Sohn, Guang Lin, Qifan Song
\thanks{
GL is supported by the National Science Foundation under grants DMS-2533878, DMS-2053746, DMS-2134209, ECCS-2328241, CBET-2347401, and
OAC-2311848. The U.S. Department of Energy also supports this work through the Office of Science Advanced Scientific Computing Research program (DE-SC0023161) and the Office of Fusion Energy Sciences (DE-SC0024583).
\\
J. Sohn is with the Booth School of Business, University of Chicago, USA. Corresponding e-mail is jwsohn612@gmail.com. Q. Song is with the Department of Statistics, Purdue University, USA. E-mail: qfsong@purdue.edu. G. Lin is with the Department of Mathematics and School of Mechanical Engineering, Purdue University, USA. E-mail: guanglin@purdue.edu.
\\
``© 2026 IEEE. Personal use of this material is permitted. Permission
from IEEE must be obtained for all other uses, in any current or future
media, including reprinting/republishing this material for advertising or
promotional purposes, creating new collective works, for resale or
redistribution to servers or lists, or reuse of any copyrighted
component of this work in other works."
\\
}}



\markboth{Preprint}{}

\maketitle

\begin{abstract}
Fairness-aware machine learning has recently attracted various communities to mitigate discrimination against certain societal groups in data-driven tasks. For fair supervised learning, particularly in pre-processing, there have been two main categories: data fairness and task-tailored fairness. The former directly finds an intermediate distribution among the groups, independent of the type of the downstream model, so a learned downstream classification/regression model returns similar predictive scores to individuals inputting the same covariates irrespective of their sensitive attributes. The latter explicitly takes the supervised learning task into account when constructing the pre-processing map. In this work, we study algorithmic fairness for supervised learning and argue that the data fairness approaches impose overly strong regularization from the perspective of the Hirschfeld–Gebelein–R\'enyi correlation. This motivates us to devise a novel pre-processing approach tailored to supervised learning. We account for the trade-off between fairness and utility in obtaining the pre-processing map. Then we study the behavior of arbitrary downstream supervised models learned on the transformed data to find sufficient conditions to guarantee their fairness improvement and utility preservation. To our knowledge, no prior work in the branch of task-tailored methods has theoretically investigated downstream guarantees when using pre-processed data. We further evaluate our framework through comparison studies based on tabular and image data sets, showing the superiority of our framework which preserves consistent trade-offs among multiple downstream models compared to recent competing models. Particularly for computer vision data, we see our method alters only necessary semantic features related to the central machine learning task to achieve fairness.
\end{abstract}


\section{Introduction}
\IEEEPARstart{T}{he} widespread use of data-driven inferential models in the real world has aroused a burgeoning demand for fairness. Despite the great advances in artificial intelligence during the era of big data, algorithmic unfairness arises from various causes, such as inaccessibility to technology, tacit prejudice in society, etc. For instance, \cite{dear:etal:22} investigated fairness issues of popular examples in business analytics such as human resource operations, supply chain management, etc. To reduce such societal discrimination effects, particularly for supervised learning, it becomes a goal to fit a supervised model that produces indistinguishable outcomes across subpopulations in protection. As a concrete example, \cite{yang:etal:23} considered controlling the ethnicities of patients and hospital locations to predict the status of COVID-19 because some bias, associated with these variables, has been present when collecting medical data. To see more examples and a comprehensive overview of fairness, refer to \cite{chenp:etal:23}. Specific strategies for achieving fair supervised learning are generally categorized into pre-processing, post-processing, and in-processing \cite{baro:etal:23}. Pre-processing methods modify given data, such as training data for supervised learning, so that learned machine learning models from the modified data become non-discriminatory against some sensitive information \cite{feld:etal:15}. In-processing methods incorporate fairness constraints directly into the model training process of the tasks \cite{agar:etal:18}. Post-processing methods directly adjust the outputs of the trained model to ensure fair outcomes \cite{hard:etal:16}. 

In particular, pre-processing methodologies, which create modified data sets, can be viewed as pseudo-generative models, taking advantage of special benefits in the era of burgeoning generative artificial intelligence (AI) models that produce diverse formats of synthetic data such as images, tabular data, and so forth. Generated data can be used for other data science projects for various purposes, such as privacy protection or improving the model's performance \cite{jord:etal:22}. For instance, a generative model equipped with a fairness-aware pre-processing layer offers benefits to the data managers or distributors who want to ensure fairness protection, because end users who receive and use the pre-processed data (or generated data) on a \emph{downstream} side can unwittingly protect fairness without applying additional algorithm-fairness techniques. Because of its practical merits, there have been attempts to generate ``fair data" \cite{xu:etal:18,raja:etal:22} on which a learned supervised model does not discriminate a certain group in the population.

On the one hand, a generative model can be categorized as task-tailored or general-purpose generative modeling. A task-tailored generative model is specialized for specific applications and is therefore highly efficient \cite{ledi:etal:17,ren:etal:20}. In contrast, a general-purpose generative model is not aware of a specific task. Rather, when a user specifically requests a task, the general-purpose model starts to resolve the task from its general domain. ChatGPT \cite{open:23} is a representative example. The general-purpose model takes advantage of applying to a wide range of tasks, but it might not perform as well as specialized models for certain tasks. Recently, \cite{moha:24} reported the task-tailored model generally achieves higher performance in various tasks than the general-purpose model. 

Similarly, existing pre-processing methods fall into either task-tailored fairness or data fairness. Approaches in data fairness try to find an intermediate distribution among subpopulations categorized by a sensitive attribute, e.g., an intermediate distribution between males and females \cite{feld:etal:15,xu:etal:18,gord;etal;19,raja:etal:22}. Although the ultimate goal is to achieve fair supervised learning, their frameworks do not explicitly consider the nature of supervised learning. To address this gap, there have been specific approaches for supervised learning \cite{wang:etal:22,zeng:etal:24} that achieve a better balance between fairness and utility than data fairness. However, these previous studies do not consider the scenario in which the pre-processed data is distributed to downstream end users whose models to be supervised are not pre-specified. In this case, the existing approaches have no results of the fairness guarantee on such a downstream side to our knowledge. Tackling such downstream fairness becomes more important because distributing or generating data to third parties frequently occurs nowadays through data sharing/subscription, generative AI platforms, etc, and the end users may be bound to protect fairness.

In this work, we propose a novel fair-supervised learning framework. This work concerns distributing data pre-processed in the upstream stage to the downstream stage where end users fit user-specific supervised models (Figure~\ref{fig:overview}, notations appear in Section~\ref{sec:notion}).
\begin{figure}[h]
    \centering
    \vspace{-0.2cm}
    \includegraphics[width=0.8\linewidth]{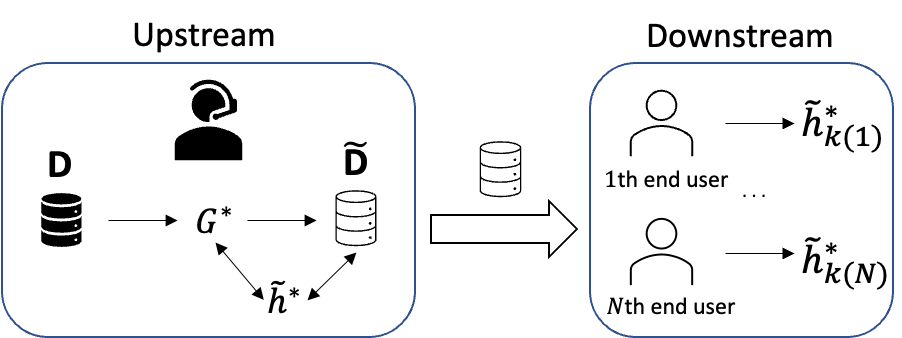}
    \caption{A data manager has the pre-processing map $G^*$ that transforms original data $\bD$ to pre-processed data $\btiD$ and the upstream supervised model $\tih^*$ that satisfies a certain level of fairness and accuracy on $\btiD$. Then $\btiD$ is distributed to end users who can fit their own downstream supervised models $\tih_{k}^*$, not sticking to $\tih^*$, but the manager wants their downstream models to satisfy similar levels of fairness and accuracy with the upstream model $\tih^*$. The subscript $k(i)$ implies the model of the $i$th end user. }
    \label{fig:overview}
\end{figure}
To our knowledge, this is the first attempt to analyze the behavior of fairness and utility of downstream models on the pre-processed data in the task-tailored framework. Here are the major contributions of our work: 
\begin{itemize}
    \item We identify the relationship between data fairness and task-tailored fairness by investigating the nonlinear Pearson correlation. 
    \item We formulate a constrained min-max bilevel optimization problem solved via neural networks to obtain a fair pre-processing model. 
    \item We analytically demystify the dynamics of fairness and utility of both upstream and downstream models learned on pre-processed data. 
\end{itemize}

\noindent {\bf Organization} Section~\ref{sec:notion} clarifies notation and explains two notions of group fairness. Section~\ref{sec:relatedworks} reviews the literature. In Section~\ref{sec:motiv}, we point out the potential inefficiency of data fairness. Based on this motivation, Section~\ref{sec:method} designs our pre-processing framework with trade-off analysis, and Section~\ref{sec:down} investigates how the generated data brings fairness improvement and utility control for arbitrary downstream models. Section~\ref{sec:algorithm} describes the optimization process of the proposed framework, and Section~\ref{sec:simul} conducts simulation studies on benchmark tabular and image data sets. Mathematical proofs and simulation details are provided in Sections ~\ref{supp:proof} and \ref{sup:simul} in the Supplementary materials (SM).

\subsection{Notions of Fairness}
\label{sec:notion}
\begin{table}[t]
    \centering
    \caption{Summary of Notations}
    \label{tab:notations}
    {\scriptsize 
    \renewcommand{\arraystretch}{1.2} 
    \begin{tabular}{l p{0.65\columnwidth}} 
        \toprule
        \textbf{Symbol} & \textbf{Description} \\
        \midrule
        $\bX, \bA, Y$ & Covariates, sensitive attributes, and outcome variable \\
        $d_{\bX}, d_{\bA}$ & Dimensionality of covariates and sensitive attributes \\
        $\mathcal{X}, \mathcal{A}, \mathcal{Y}$ & Domain spaces where $\mathcal{X} \subset \mathbb{R}^{d_{\bX}}$, $\mathcal{A} \subset \mathbb{R}^{d_{\bA}}$, $\mathcal{Y} \subset \mathbb{R}$ \\
        $\bD$ & Original dataset tuple $(\bX, \bA, Y)$ \\
        $G, G^*$ & Pre-processing map and optimal map \\
        $\bbarD = (\bbarX, \barY)$ & Transformed data, output of a generic map $G$ \\
        $\btiD = (\btiX, \tiY)$ & Transformed data, output of the optimal map $G^*$ \\
        $\Delta_{\bX}, \Delta_Y$ & Distance metrics on spaces $\mathcal{X}$ and $\mathcal{Y}$ \\
        \midrule
        $h, \mathcal{H}$ & Upstream supervised model and its class \\
        $l(\cdot, \cdot)$ & Loss function, where $l \geq 0$ \\
        $h^*, \barh^*, \tih^*$ & Risk minimizers on $\bD$, $\bbarD$, and $\btiD$ respectively \\
        $h_k, \mathcal{H}_k$ & Downstream supervised model $k$ and its hypothesis space \\
        $\tih_k^*$ & Risk minimizer for downstream model $k$ on $\btiD$ \\
        \midrule 
        $\rho(S_1,S_2)$ & Hirschfeld–Gebelein–R\'enyi correlation between $S_1,S_2$ \\
        $d(S_1,S_2)$ & $\sqrt{2-\rho(S_1,S_2)}$ \\ 
        \midrule
        SM & Supplementary material \\ 
        \bottomrule
    \end{tabular}
    }
\end{table}

For readers' convenience, Table \ref{tab:notations} summarizes the notations used throughout this paper. Denote by $\bX=(X_1,\dots,X_{d_{\bX}}) \in \mathcal{X} \subset \mathbb{R}^{d_{\bX}}$ a vector of $d_{\bX}$ covariates, $\bA = (A_1,\dots,A_{d_{\bA}}) \in {\cal A} \subset \mathbb{R}^{d_{\bA}}$ a vector of $d_{\bA}$ sensitive attributes, $Y\in\mathcal{Y}\subset\mathbb{R}$ an outcome variable, and $h\in \calH:\mathcal{X}\rightarrow \mathcal{Y}$ an upstream supervised learning. For instance, ${\cal Y}=[0,1]$ implies $h$ is a classification model, and regression with ${\cal Y}=\mathbb{R}$. Let's denote by $\Delta_{\bX}$ and $\Delta_Y$ the distance metrics on $\calX$ and $\calY$. For the model $h$, a risk is defined as $\bE[l(Y,h(\bX))]$ for a given loss function $l\geq 0$. The risk minimizer is denoted by $h^*=\arg_{h\in{\calH}}\min \bE[l(Y,h(\bX))]$. If the minimizer has a small (or high) risk, we say the model has a high (or small) utility. We define a pre-processing map $G\in {\cal G} :\calX \times \calA \times \calY \rightarrow \calX \times \calY$ and denote by $\bbarD=(\bbarX,\barY)$ the output of $G$ from the original data $\bD=(\bX,\bA,Y)$. Depending on the context, we abuse $\bD=(\bX,Y)$ if ${\bf A}$ is abandoned for the model input. For the output for the optimal pre-processing map $G^* \in {\cal G}$, we specifically denote by $\btiD=(\btiX,\tiY)$ the transformed data from $\bD$. We denote by $\barh^*$ and $\tih^*$ the risk minimizer of $\bE[l(\bar{Y},h(\bar{\bX}))]$ and $\bE[l(\tiY,h(\btiX))]$. Multiple downstream supervised models $h_k\in{\cal H}_k$ are trained independently on $(\btiX,\tiY)$ for $k=1,\dots,K$. They inherit the notations of the upstream model for $h_k^*$ and $\tih_k^*$ which are risk minimizer on $(\bX,Y)$ and $(\btiX,\tiY)$ respectively. 

In fair supervised learning, independence and separation are two key fairness criteria. Referring to \cite{baro:etal:23}, we introduce the most famous definitions in algorithmic fairness as follows.
\begin{definition}[Independence]
    A supervised model $h$ satisfies independence if $\bA \perp h(\bX)$.
\end{definition}
\begin{definition}[Separation]
    A supervised model $h$ satisfies separation if $\bA \perp h(\bX)|Y$.
\end{definition}
Independence (also known as demographic/statistical parity) requires that a model's predictions be independent of a sensitive attribute (e.g., race or gender), meaning the outcomes are consistent across groups defined by the attribute. Separation (or equalized odds/disparate mistreatment) requires conditional independence of predictions from the sensitive attribute given the true outcome, meaning individuals within the same outcome class (e.g., those who truly qualify for a loan) should have an equal likelihood of receiving a positive prediction across groups.

\subsection{Literature Review}
\label{sec:relatedworks}

To represent the dependence on the type of sensitive attributes, we subdivide notations for them. $A_{\text{bin}}$ stands for a binary univariate variable. $\bA_{\text{dis}}$ includes all one-hot encoded discrete $A_1,\dots,A_{d_{\bA}}$. $\bA$ includes a wide range of variables, such as a mix of multiple one-hot encoded discrete and continuous variables. 
 
As we briefly discussed in the introduction, there have been mainly two strategies for fair pre-processing. The first, called data fairness, is to find the intermediate distribution balancing the data distributions for different sensitive groups (i.e., $A_{\text{bin}}$ or $\bA_{\text{dis}}$) directly. For instance, \cite{feld:etal:15,gord;etal;19} focused on finding a Wasserstein-barycenter distribution among discretized sensitive groups. \cite{calm:etal:17} solved an optimization to produce a transformed data set $(\btiX,\tiY)$ by controlling a discrepancy between $Y$ and $Y|\bA_{\text{dis}}$. In the context of generative modeling, \cite{choi:etal:24} proposed switching sensitive attributes, applicable for $\bA_{\text{dis}}$, in the sampling process of diffusion models. \cite{tian:etal:24} advanced the mix-up procedure to create fair attributes that interpolate different instances from $\bA_{\text{dis}}$. 

Secondly, a task-tailored pre-processing strategy identifies a transformation map tailored to the intended use of the transformed data. Specifically for a supervised learning task, \cite{wang:etal:22} suggested generating adversarial samples that weaken a deployed and frozen model's dependency. In a similar spirit, \cite{petr:etal:22} obtained the adversarial sample weights designed to lessen the prediction power of predicting sensitive attributes from covariates. These past studies borrow the prediction-based mitigation mechanism \cite{zhan:etal:18,adel:etal:19} which may technically handle $\bA$ but possibly with extra optimization complexity. \cite{jang:etal:24} devised a feature-selection mechanism to improve the fairness in terms of $\bA_{\text{dis}}$. Recently, \cite{zeng:etal:24} devised an advanced method that harnesses randomized labels for fairness while exploiting utility by finding the optimal Bayes rule, but limited to $A_{\text{bin}}$.

The existing task-tailored strategies, however, may have restricted applicability to real-world problems. \cite{zeng:etal:24} achieves stronger performance than various past approaches under a mathematical guarantee for fairness but their framework only supports a binary classification task with $A_{\text{bin}}$. \cite{jang:etal:24} is also only applicable for $\bA_{\text{dis}}$ and they do not have a fairness guarantee when selected features are used for other prediction models. Similarly, adversarial re-weighting of \cite{petr:etal:22} also does not guarantee when the learned weights for a specific neural network are used for other machine learning models. Moreover, some learning algorithms may not adapt such instance weights during the training procedure (e.g., K-nearest neighbor regression/classification) without nontrivial modification to the original formulation. Finally, to our knowledge, there have been no studies to discuss when/where the use of pre-processed data remains effective for downstream models different from the upstream model, particularly for controlling independence and separation.  

\section{Methodology}
\label{sec:method_intro}
\subsection{Motivation}
\label{sec:motiv}
We first justify that data-fairness approaches may destroy too much utility, particularly for supervised learning. For mathematical justification, we mainly build our arguments on the Hirschfeld–Gebelein–R\'enyi correlation (HGR, \cite{reny:59:2}) that measures the nonlinear Pearson correlation between two univariate random variables. While the original definition of HGR is for two univariate random variables, it is straightforward to extend the concept that adapts random vectors as follows. 
\begin{definition}[HGR correlation]
\label{def:hgr}
For random vectors $\bX$ and $\bA$, the Hirschfeld–Gebelein–R\'enyi correlation is defined as $\rho(\bX,\bA) = \sup_{f,g} \bE[f(\bX)g(\bA)]$ where the supermum is taken over all Borel measurable functions $f:\calX \rightarrow \mathbb{R}$ and $g:\calA \rightarrow \mathbb{R}$, satisfying $\bE[f(\bX)]=\bE[g(\bA)]=0$ and $\bE[f(\bX)^2]=\bE[g(\bA)^2]=1$.
\end{definition}
This extended notion inherits the main properties of the original definition including: (P1) $0 \leq \rho(\bX,\bA) \leq 1$, (P2) $\rho(\bX,\bA)=0$ if and only if $\bX$ and $\bA$ are jointly independent, (P3) $\rho(\bX,\bA)=\rho(\bA,\bX)$, and (P4) if $\bX$ and $\bA$ have arbitrary functional relationship, then $\rho(\bX,\bA)=1$. To see other properties for the univariate case, we refer to \cite{reny:59:2}. 

From the data-fairness perspective, the HGR correlation $\rho(\bar{\bX},\bA)$ or $\rho(\bar{Y},\bA)$ can be leveraged as main objectives to be minimized as to the map $G$. For instance, \cite{feld:etal:15} suggested a way to obtain $\btiX$ such that $\rho(\btiX,\bA_{\text{dis}})\approx 0$ by solving the Wasserstein-barycenter problem. Moreover, if $\rho(\btiX,\bA)=0$ and $\rho(\tiY,\bA)=0$, then it trivially follows that any supervised models $\tih^*$ learned on $(\btiX,\tiY)$ is independent to $\bA$. There may need to be some constraints such as $\Delta_{\bX}(\bX,\btiX)\leq \delta_{\bX}$ and $\Delta_Y(Y,\tiY)\leq\delta_Y$, where $\Delta_{\bX}$ and $\Delta_Y$ measure the discrepancy between two quantities respectively, to avoid meaningless solutions of $G$ (e.g., degenerated $\tiX$ and $\tiY$) and to preserve utility as argued in \cite{calm:etal:17}. By adjusting the size of $\delta_{\bX}$ and $\delta_Y$ appropriately, this data-fairness scheme can provide a series of $\tih^*$, each of which has different levels of fairness.

Such a data fairness approach, however, may overly penalize the unfairness of the supervised model.
\begin{proposition}
\label{prop:motiv}
    For a given Borel measurable function $u$ defined on ${\cal X}$, $\rho(u(\bar{\bX}),\bA) \leq  \rho(\bar{\bX}, \bA)$ holds for any measurable $\bar{\bX}$ and $\bA$. 
\end{proposition}
This simple mathematical observation justifies that $\rho(\bar{\bX}, \bA)$ is stronger than $\rho(u(\bar{\bX}),A)$ within the same amount of $\delta_{\bX},\delta_Y$. In other words, the regularizer $\rho(\bar{\bX}, \bA)$ for finding $G^*$ which pursues data fairness may lead to an over-fair $\tih^*(\btiX)$ which implicitly over-sacrifices the utility of $\tih^*(\btiX)$ in the end. Therefore, depending on a targeted machine learning task, exploring a better regularization term should be of great interest because it potentially preserves much larger task-specific utility. In this work, motivated by Proposition~\ref{prop:motiv}, we use the penalty of the form $\rho(u(\bar{\bX}),A)$ with $u$ being $\bar{h}^*$, so it becomes specialized for supervised learning.

The degree of fairness improvement in supervised learning can be evaluated by the HGR correlation. On the upstream side, $\tilde{\Delta}_F=\rho(h^*(\bX),\bA)-\rho(\tih^*(\btiX),\bA)$ implies the fairness improvement for the upstream, that is the difference of HGR correlations on $\bD$ and $\btiD$. Note $\btiD$ is the generated data from the optimal $G^*$. The positive $\tilde{\Delta}_F$ means that $\btiD$ results in the less discriminative model $\tih^*$. Similarly, we denote by $\tilde{\Delta}_F^k=\rho(h^*_k(\bX),\bA)-\rho(\tih^*_k(\btiX),\bA)$ the fairness improvement on the downstream side.

Before introducing our main framework, Figure~\ref{fig:ex1} shows a motivating illustration that compares our pre-processing method, tailored to supervised learning, with existing data-fairness methods \cite{feld:etal:15,gord;etal;19} on a toy example. The detailed construction of our framework appears in the next section. For ${\calH}$ and ${\cal G}$, 4-layer feed-forward neural networks with 64 hidden nodes are specified for $\btiX$, $\tiY$, and $h$. As a quantitative measure of fairness, the value of $|\bE[\tih^*_k(\tiX)|A=1] - \bE[\tih^*_k(\tiX)|A=0]|$ is considered. For the accuracy evaluation, the predictive mean squared error $\bE[(Y-\tih^*_k(X))^2]$ is calculated. The smaller the two metrics, the higher the performance. For all three downstream models indexed by $h_k=\{$Random Forest Regression, Support Vector Regression, and KNN-based Regression$\}$\footnote{These are implemented by scikit-learn in Python.}, ours achieves a better balance of accuracy and fairness, as varying the degree of partial fairness. To see further details about the implementation, refer to SM~\ref{sup:fig_ex1}. 
\begin{figure}[ht!]
    \centering
    \includegraphics[width=1.0\linewidth]{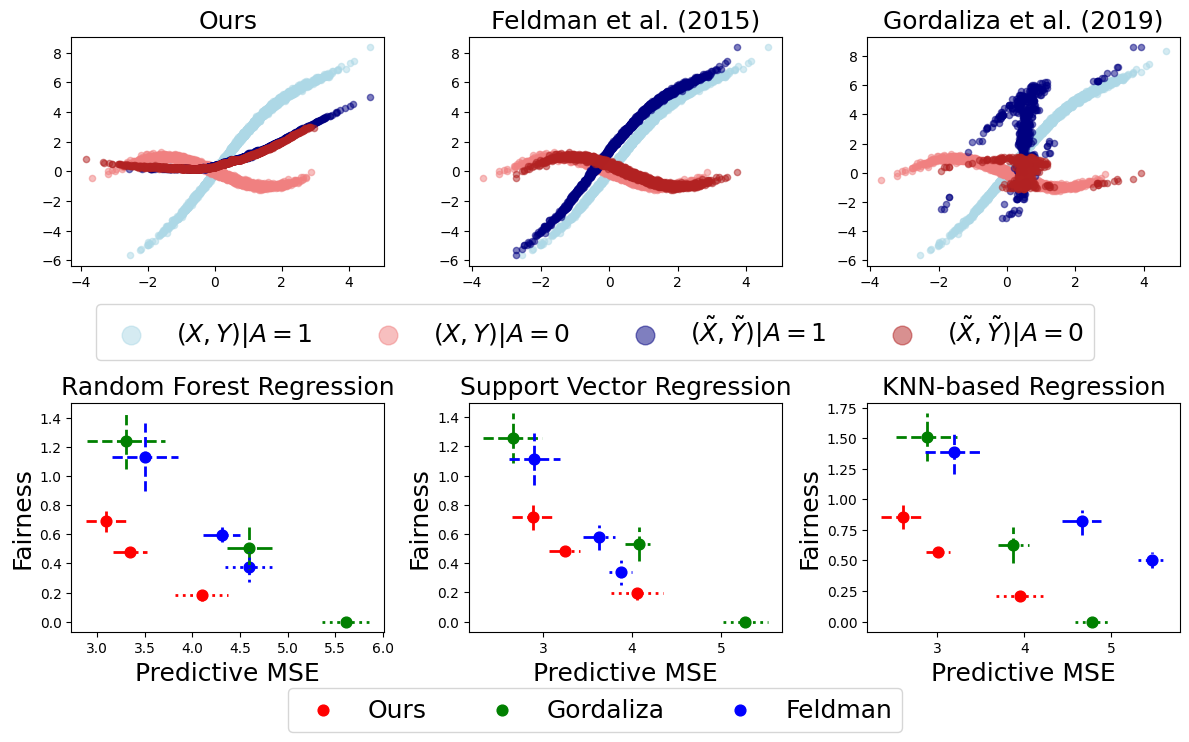}
    \vspace{-0.2in}
    \caption{A Toy example with data generated from $Y = (2A-1)\sin(X)+ 2AX + \epsilon$ where $A \sim {\rm Ber}(0.5)$, $X \sim {\rm N}(A,1^2)$, and $\epsilon \sim {\rm N}(0,0.1^2)$. The first row illustrates the original and transformed data. For both Feldman et al. (2015) \cite{feld:etal:15} and Gordaliza et al (2019) \cite{gord;etal;19}, $X$ is transformed to $\tilde X$ to maximally remedies discrimination, while $\tilde{Y}=Y$. Our method yields $(\tiX,\tiY)$ having a more efficient accuracy and fairness trade-off curve, as shown in the second row. For figures in the second row, the dot is the average, with the bars of  2$\times$standard error. The line styles distinguish the choices of fairness degree: weak (dashed line), medium (dash-dot line), and strong (dotted line). In our case, the transformed data under different $A$ overlap with each other, both showing a smooth intermediate functional relationship in contrast to the existing methods.}
    \label{fig:ex1}
\end{figure}

To incorporate the HGR correlation in the theoretical analysis, this work leverages a distance-like metric $d(S_1,S_2)=\sqrt{2-2\rho(S_1,S_2)}$ for generic random variables or vectors $S_1,S_2$. Trivially, $d(S_1,S_2)\leq \sqrt 2$ and the equality holds when $S_1$ and $S_2$ are independent. In particular, $d$ possesses the triangle inequality $d(S_1,S_2)\leq d(S_1,S_3)+d(S_2,S_3)$ when the pivot variable $S_3$ is binary (see Lemma~\ref{lem:triangle} in SM~\ref{supp:triangle}).
 
\subsection{Pre-processing via in-processing in upstream}
\label{sec:method}

To devise the pre-processing map aware of supervised learning, we borrow a popular in-processing structure \cite{lee:etal:22,sohn:etal:23}. Recall that  $\bbarD=(\bbarX,\barY)$ is the output of $G(\bX,\bA,Y)$. The ideal pre-processing map $G^*$ leads to a supervised downstream model with higher utility and fairness. Based on Proposition~\ref{prop:motiv}, the fairness regularization onto $G$ can be imposed by  the penalty term $\rho(\barh^*(\bbarX),\bA)$ w.r.t. $G$ where $\barh^*$ is the risk minimizer on $\bbarD$. Meanwhile, the functional relationship between $\bbarX$ and $\barY$ through $\barh^*$ should be maximally preserved, leading to minimizing the risk of $\barh^*$ w.r.t. $G$. Upon this consideration, the desired pre-processing map can be found by solving the following bilevel-constrained optimization
\begin{align}
\label{opt:min}
    G^* = \arg_G \min \bE[l(\barY,\barh^*(\bbarX))] + \lambda_F \rho(\barh^*(\bbarX),\bA),
\end{align}
 subject to $\Delta_{\bX}(\bX,\bbarX)\leq \delta_{\bX}$, $\Delta_Y(Y,\barY)\leq \delta_Y$ with pre-specified $\delta_{\bX}$, $\delta_Y$, $\lambda_F$, where $\barh^*=\arg_{h \in \calH }\min\bE[l(\barY,h(\bbarX))]$. The outer optimization for $G$ relies on the solution of the inner optimization for $h$. The optimization \eqref{opt:min} draws a line from in-processing since the inner optimization for $\barh^*$ does not pertain to the fairness penalty. 
 To impose the fairness on $G^*$ only, $\barh^*$ must not directly take $\bA$ as its input. Recall that $\btiD=(\btiX,\tiY)$ is the output of $G^*$. This solution finds $\btiD$ not too distant from $\bD$ due to the marginal constraints of $\delta_{\bX}$ and $\delta_{Y}$. Ideally, the solution pair $(\tih^*,G^*)$ have balanced size of $\bE[l(\tiY,\tih^*(\btiX))]$ and $\rho(\tih^*(\btiX),\bA)$ by the values of $\delta_{\bX},\delta_Y,\lambda_F$. The structure of $G$ is further subdivided by two children functions $G_{\bX}:{\cal X}\times {\cal A}\rightarrow {\cal X}$ and $G_Y:{\cal X}\times {\cal A}\times{\cal Y}\rightarrow {\cal Y}$, so $G_{\bX}^*$ and $G_Y^*$ produce $\btiX$ and $\tiY$ respectively. The transformation $G_{\bX}$ does not depend on $Y$ to maintain the nature of supervised learning over $\btiD$. 

We first investigate the relationship between $\bE[l(\tiY,\tih^*(\btiX))]$ and  $\rho(\tih^*(\btiX),\bA)$ to figure out if the utility and fairness terms trade off each other. For simplified discussion, we first deal with the solution pair $(\tih^*,G^*)$ obtained under $\delta_Y=0$, i.e., $\tiY=Y$. Let's denote by  $\check{\epsilon}=\inf_{\check{h}\in \check{{\calH}}} \bE[l(Y,\check{h}(\bX,\bA))]$ where $\check{h}\in\check{{\calH}}:{\cal X}\times {\cal A}\rightarrow \mathcal{Y}$, and $\check{{\calH}}$ expands the function space $\calH$ by allowing $A$ as an additional input. Thus $\check{\epsilon}$ is the minimal risk when using $\bX$ and $\bA$ jointly to predict $Y$. Let $e(\bA)=\bE[l(Y,h^*(\bX))]-\check{\epsilon}$ be the extra contribution of $\bA$ for predicting $Y$.
Then the fairness improvement on the upstream side is characterized as follows.
\begin{theorem}
\label{thm:uplwh0}
For any given $\delta_X, \lambda_F \geq 0$, if $\delta_Y=0$, the fairness improvement $\tilde{\Delta}_F=\rho(h^*(\bX),\bA)-\rho(\tih^*(\btiX),\bA))$ is bounded below by $\tilde{\Delta}_F \geq \lambda_F^{-1} (\bE[l(Y,\tih^*(\btiX))]-e(\bA)-\check{\epsilon})$. Especially, if the output of $h$ is binary, the improvement is bounded above by $\tilde{\Delta}_F \leq d(Y,h^*(\bX))+d(Y,\tih^*(\btiX))$.
\end{theorem}
Derived bounds address two important remarks. The lower bound implies that strict fairness improvement comes in when $\bE[l(Y,\tih^*(\btiX))]>e(\bA)+\check{\epsilon}$. If $\bX$ and $\bA$ are highly correlated, i.e., $e(\bA)\approx 0$, then the fairness improvement gets stronger. Also, both bounds relate to the coherency between $Y$ and $\tih^*$ (refer to the next paragraph), which means that the fairness improvement involves the risk of false prediction for the intact label $Y$. This trade-off inclination becomes more evident when a strong functional relationship between $Y$ and $\bX$ presents. If $d(Y,h^*(\bX))\approx 0$, the upper bound is mostly explained by $d(Y,\tih^*(\btiX))$ which aligns, as elaborated in the next paragraph, with $\bE[l(Y,\tih^*(\btiX))]$ in the lower bound.

Both  $\bE[l(Y,h^*(\bX))]$ and $d(Y,h^*(\bX))$ measures the degree of coherency between $Y$ and $h^*(\bX)$. 
For instance, let's assume that both $h^*(\bX)$ and $Y$ are standardized. Then we observe $d(Y,h^*(\bX))=\sqrt{2-2\rho(Y,h^*(\bX))}\leq \sqrt{\bE[(Y-h^*(\bX))^2]}$ since $2-2\text{Corr}(Y,h^*(\bX))=\bE[(Y-h^*(\bX))^2]$. For binary $Y$ and $h^*(\bX)$, the HGR correlation has a much clearer relation with the accuracy of $h^*$ since $\rho(Y,h^*(\bX))=|\text{Corr}(Y,h^*(\bX))|$ and $\text{Corr}(Y,h^*(\bX))$ is regarded as a desirable evaluation measure in classification also called the Matthews correlation coefficient. Therefore, it is reasonable to expect that a supervised model achieves a higher HGR correlation as the model becomes more accurate by minimizing a certain target objective, such as $L_2$, cross-entropy, hinge loss, etc. 

Now we discuss the scenario harnessing the auxiliary label $\tiY$ with $\delta_Y>0$. Using $\tiY$ brings fundamental differences in the relationship between the fairness and risk of $\tih^*$. While $\tih^*$ trained under $\delta_Y=0$ sacrifices utility as shown in Theorem~\ref{thm:uplwh0}, the additional flexibility of $G_Y$ with $\delta_Y>0$ yields $\tiY$ that would not compromise either $\bE[l(\tiY,\tih^*(\btiX))]$ or $\rho(\tih^*(\btiX),A)$. The following corollary characterizes the lower and upper bounds of  $\tih^*$'s fairness improvement. 
\begin{corollary}
\label{cor:uplwh0}
For any given $\delta_X, \delta_Y, \lambda_F \geq 0$, it follows that $d(\tih^*(\btiX),\bA) - d(h^*(\bX),\bA) \geq d(\tiY,\bA) - d(\tiY,\tih^*(\btiX)) - d(h^*(\bX),\bA)$.
    Especially, if the outputs of $h^*$ and $\tih^*$ are binary, we can show $\tilde{\Delta}_F \leq d(\tiY,h^*(\bX))+d(\tiY,\tih^*(\btiX))$.
\end{corollary}
Compared to Theorem~\ref{thm:uplwh0}, the lower bound is characterized mainly by the distance-like metric $d$; for the upper bound, the original label is replaced by $\tiY$. Note the upstream fairness improvement can also be represented as $d(\tih^*(\btiX),\bA) - d(h^*(\bX),\bA)$ which conceptually corresponds to $\tilde{\Delta}_F$. 

Although the lower bound is a direct consequence of the triangle inequality, it helps figure out how the fairness improvement under $\delta_Y > 0$ is differentiated in contrast to the case of $\delta_Y = 0$. By the result of the lower bound, a larger improvement in $\tih^*$ is expected as $\tiY$ gets a smaller dependence on $\bA$, i.e. $d(\tiY,\bA)\rightarrow \sqrt{2}$, and the model $\tih^*$ should be sufficiently accurate for predicting $\tiY$, i.e. $d(\tih^*(\btiX),\tiY)\rightarrow 0$. Meanwhile, the upper bound increases as $\tiY$ becomes more distant from the original model $h^*(\bX)$. This important observation is highlighted in the following remark.
\begin{remark}
\label{remark:opt}
With $\delta_Y > 0$, the optimization \eqref{opt:min} creates an auxiliary label $\tiY$ that is away from $\bA$ by pushing  $\tih^*(\btiX)\perp \bA$ and at the same time $\tiY\approx \tih^*(\btiX)$. However, if $\delta_Y=0$, encouraging $\tih^*(\btiX)\perp \bA$ tends to make $\tih^*(\btiX)$ away from $Y$.
\end{remark}
This investigation of $\tiY$ provides a stepping stone for insightful analysis in later sections when studying the behavior of an arbitrary downstream model $\tih_k^*\in \calH_k$ learned on the output of $G^*$. 

In the optimization \eqref{opt:min}, the budget and strength parameters $\delta_{\bX}$, $\delta_Y$, and $\lambda_F$ take a central role in adjusting the level of fairness and utility. On the application side, the type of measurement $\Delta_{\bX}$ and $\Delta_Y$ and the key parameters should be carefully chosen based on data policy or domain knowledge. There can be situations where distorting specific variables in $\bX$ or $Y$ is not allowed due to uncompromisable reasons, depending on the context of its application. In this work, we assume that the learned $G_{\bX}^*$ is also used to transform the test data $(\bX',\bA')$ into $\btiX'$ where $(\bX',\bA')\overset{d}{=}(\bX,\bA)$ so that the level of fairness and utility obtained on $(\btiX,\tiY)$ remains consistent for the test data. This can be seen as an effort to overcome covariate shift in domain adaptation.

\subsection{Behavior of downstream models learned on $G^*$}
\label{sec:down}

\subsubsection{Fairness and utility control of downstream models}
\label{sec:fair_util_control}

Our analysis for the downstream fairness hinges on a major premise for the relationship between $\calH$ and $\calH_k$. The downstream analysis aims to scrutinize when $G^*$ paired with $\tih^*$ is also effective for mitigating an arbitrary downstream model $\tih_k^*$ and preserving its utility. At first appearance, there is no clear reason that the fairness improvement of $\tih^*$ should appear for $\tih_k^*$ since \eqref{opt:min} does not consider $\tih_k^*$. For tractable analysis, we assume that $\calH$ is large enough to include the class $\calH_k$.  
\begin{assumption}
\label{assm:dominating}
    The upstream class $\calH$ is a superset of downstream class $\calH_k$, i.e., $\calH_k \subset \calH$.
\end{assumption}
For instance, if $\calH$ is the class of feed-forward neural networks, the class of linear models $\calH_k$ necessarily belongs to it. In fact, if one chooses $\cal H$ to be neural networks with sufficient width or depth, this assumption approximately holds for a continuous function class ${\cal H}_k$ due to the universal approximation property of neural networks. Based on this assumption, the remaining part shows the fairness improvement of $\tih^*$ can lead to the improvement of $\tih_k^*$.

For effective discussion, we use functional notations, for the variable $h$, $L(h;\bD)=\bE[l(Y,h(\bX))]$ and $L(h;\btiD)=\bE[l(\tiY,h(\btiX))]$. For instance, $L(\tih_k^*;\btiD)$ indicates the risk of $\tih_k^*$ trained on the output of $\btiD$. We denote by $\Delta_{F}^k=\rho(h^*(\bX),\bA)-\rho(h^*_k(\bX),\bA)$ the difference of unfairness. As $\tilde{\Delta}_F=\rho(h^*(\bX),\bA)-\rho(\tih^*(\btiX),\bA)$ for the improvement in upstream fairness, we denote by  $\tilde{\Delta}_L=L(h^*;\bD)-L(\tih^*;\btiD)$ the utility improvement in $\tih^*$ and $\Delta_L^k=L(h_k^*;\bD)-L(h^*;\bD)$ the risk difference on $\bD$. 

We first investigate how the risk of $\tih^*$ implicitly captures the risk of $\tih_k^*$ on $\btiD$. To simplify the discussion, we assume that $h$ and $h_k$ are Lipschitz and employ the same Lipschitz loss $l$. The function $h:\calX\rightarrow \calY$ is said to be $M$-Lipschitz if $\Delta_Y(h(x_1),h(x_2))\leq M \Delta_{\bX}(x_1,x_2)$ for all $x_1,x_2\in {\cal X}$. A loss function $l:\calY \times \calY \rightarrow \mathbb{R}$ is said to be $M^l$-Lipschitz if $|l(y,\hat{y}_1)-l(y,\hat{y}_2)|\leq M^l \Delta_{Y}(\hat{y}_1,\hat{y}_2)$ for all $\hat{y}_1,\hat{y}_2,y \in {\calY}$, and $|l(y_1,\hat{y})-l(y_2,\hat{y})|\leq M^l \Delta_{Y}(y_1,y_2)$ for all ${y}_1,{y}_2,\hat{y} \in {\calY}$. By choosing $\Delta_{\bX}$ and $\Delta_Y$ are proper metrics in $\calX$ and $\calY$, one can bound the size of $\tih_k^*$'s risk by the one of $\tih^*$ up to additive constants relying on the choice of $\delta_{\bX},\delta_Y$.
\begin{lemma}
\label{lem:util}
    Let's suppose the loss function $l$ is $M^l$-Lipschitz, and $h_k^*$ is $M_{k}$-Lipschitz. We further assume that the minimizer of $\bE[l(Y,h(\btiX))]$ in $\calH$ is $\tilde{M}$-Lipschitz. Then it follows that 
    $L(\tih^*;\btiD)\leq L(\tih_k^*;\btiD) \leq L(\tih^*;\btiD) + 2M^l \delta_{Y} + M^l(M_k+\tilde{M}) \delta_{\bX} + \Delta_L^k$.
\end{lemma}
Under the moderate use of $\delta_{\bX},\delta_Y$, the risk of $\tih^*$ is able to capture the downstream risk of $\tih_k^*$ if the utility gap $\Delta_L^k$, reflecting the discrepancy of capacity between $\calH$ and $\calH_k$ on $\bD$, is not too large. Many popular loss functions are Lipschitz, such as the mean absolute error, hinge loss, logistic loss, etc. For some non-Lipschitz losses, a similar bound can be derived. As an example, the mean-squared error is decomposed into $\bE[(\tiY - \tih^*(\btiX))^2]\leq \bE[(\tiY - \tih_k^*(\btiX))^2]+\bE[(\tih_k^*(\btiX) - \tih^*(\btiX))^2]$, so a similar bound as Lemma~\ref{lem:util} is obtained. For (generalized) cross-entropy loss, one can clip the probabilities to prevent them from diverging in the boundary so that the loss becomes Lipschitz.

The downstream fairness improvement $\tilde{\Delta}_F^k$ is characterized by the aforementioned informative quantities associated with $\tih^*$ and $h_k^*$. Our analysis aims to explain how the improvement of $\tilde{\Delta}_F$ leads to improving $\tilde{\Delta}_F^k$ and what affects the accurate relation between those. Recall that $\tilde{\Delta}_F^k=\rho(h^*_k(\bX),\bA)-\rho(\tih^*_k(\btiX),\bA)$ and $\tilde{\Delta}_F=\rho(h^*(\bX),\bA)-\rho(\tih^*(\btiX),\bA)$. The derived bound, as follows, asserts that the accuracy of $\tih^*$ is an important component to guarantee the fairness improvement of downstream models.
\begin{theorem}
\label{thm:uplwhk}
Suppose the assumptions in Lemma~\ref{lem:util} hold. Let $C_{M,\delta_{\bX},\delta_Y}=2M^l \delta_Y + M^l(M_k+\tilde{M}) \delta_{\bX} + \Delta_L^k$. The downstream fairness improvement $\tilde{\Delta}_F^k$ is lower bounded by $\tilde{\Delta}_F^k \geq \tilde{\Delta}_F - \Delta_F^k - d(\tih^*(\btiX),\tih_k^*(\btiX))-\lambda_F^{-1}C_{M,\delta_{\bX},\delta_Y}$ and upper bounded by $\tilde{\Delta}_F^k \leq \tilde{\Delta}_F - \Delta_F^k + d(\tih^*(\btiX),\tih_k^*(\btiX))+\lambda_F^{-1}C_{M,\delta_{\bX},\delta_Y}$.
\end{theorem}
To see how much the upstream improvement $\tilde{\Delta}_F$ should be made to guarantee $\tilde{\Delta}_F^k\geq 0$, a sufficient condition derived from the lower bound in Theorem~\ref{thm:uplwhk} are
{\small 
\begin{align}
\label{eqn:lowercondi}
    \tilde{\Delta}_F \geq \Delta_F^k + d(\tih_k^*(\btiX),\tiY)+d(\tih^*(\btiX),\tiY)+ \lambda_F^{-1} C_{M,\delta_{\bX},\delta_Y}.
\end{align}}
To guarantee a fairness improvement of $\tih_k^*$, the fairness improvement of $\tih^*$ has to be larger than the sum of the terms in the right-hand side of (\ref{eqn:lowercondi}): $\Delta_F^k$, the difference of unfairness on $\bD$, which measures how much $h^*$ experiences more severe discrimination than $h_k^*$; the degree of accuracy of $\tih_k^*$ and $\tih^*$ measured by $d$ metric, which indirectly relate to their utilities (or losses) on $\btiD$ as discussed in the previous section; and the utility deviation scaled by $\lambda_F^{-1}$. Thus, a large improvement of $\tilde{\Delta}_F$ should precede if $\rho(h^*(\bX),\bA) > \rho(h^*_k(\bX),\bA)$ and $\tih_k^*$ and $\tih^*$ are too distant or suboptimal in utility. 

Our theories underline the role of the utility of $\tih^*$ and $\tih_k^*$ to guarantee the downstream fairness improvement. This emphasis on the utility helps establish a general guideline for using our framework \eqref{opt:min} more effectively. 
\begin{remark}
\label{remark:ex}
     If the given data $\bD$ involves a difficult supervised learning task, i.e., $L(h^*;\bD)$ is substantially large, $L(\tih^*;\btiD)$ could get much worse if $\delta_Y=0$ by the trade-off between utility and fairness discussed after Theorem~\ref{thm:uplwh0}. Thus to satisfy \eqref{eqn:lowercondi}, we recommend using 
     a large $\lambda_F$ or $\delta_Y>0$. For the latter choice ($\delta_Y>0$), it can improve the fairness and reduce $L(\tih^*;\btiD)$ simultaneously, as discussed in Remark~\ref{remark:opt}. 
     On the other hand, if the given data $\bD$ has a high-accuracy supervised learning task on $\calH$ (i.e., $L(h^*;\bD)\approx 0$) and $L(\tih^*;\btiD)$ tends to be small under small $\delta_X$ and $\delta_y=0$, thus \eqref{eqn:lowercondi} is relatively easily satisfied.
     That is, 
     one can avoid using $\delta_Y>0$ unless $L(\tih^*;\btiD)$ is too large. Note that using the auxiliary label $\tiY(\neq Y)$ in the training of $\tih_k^*$ would deteriorate the model's predictive accuracy of test data. 
\end{remark}
The predictive accuracy mentioned in this remark means evaluating the degree of $\tih_k^*$'s accuracy on the test data $(\btiX',\bA',Y')$. We suppose $G_{\bX}^*$ is available in the inference phase. By denoting the test data $\btiD'=(\btiX',Y')$ for evaluation and assuming that there are no covariates and label shifts in the test data, utility control becomes less accurate when having $\tih_k^*$ on $\btiD=(\btiX,\tiY)$ than $(\btiX,Y)$. The proof of Theorem~\ref{lem:util} includes this part as a subsequential result.

We acknowledge certain limitations in our analysis regarding practical downstream application models. First, some popular supervised models, such as tree-based models, may not satisfy the Lipschitz condition. However, it is common to assume classification or regression models belong to a Lipschitz class (e.g., \cite{gao:etal:20}) for theoretical analysis. For binary classification, the binary assumption of $\tih^*$ and $\tih_k^*$ can be satisfied through a threshold-based decision on continuous outputs. Secondly, employed supervised models in the downstream stage usually have different built-in loss functions for optimization. Nevertheless, we expect this impact will not be overwhelming since it is unrealistic that a model performs too differently for a different loss function out of nowhere.

\subsubsection{Consistent fairness-utility balance among end users}
\label{sec:consistency}
Lemma~\ref{lem:util} provides deeper insights into our task-tailored pre-processing mechanism. Assuming multiple end users with different downstream models, the data distributor, without prior knowledge of the selected models, aims to ensure similar levels of utility and fairness across users. Modeling end users as randomly selecting downstream models, the following proposition demonstrates that these models achieve comparable risk levels.
\begin{proposition}[Consistency of utility]
\label{prop:coherent}
    Suppose there are $N$ downstream models each of which randomly picks among ${\calH}_k\subset{\calH}$ for $k=1,\dots,K$ with the assumptions in Lemma~\ref{lem:util}.
    Let's define $M^*=\max \{M_k\}_{k=1}^K$ and $\Delta_L^*=\max \{\Delta_L^k\}_{k=1}^K$ Let's denote by $h_{k(i)}\in {\calH}_{k(i)}$ the selected model by the $i$th downstream user. Then the variance of downstream risk is bounded above by $\text{Var}(L(\tih_{k(i)}^*;\btiD))\leq {(2M^l (\delta_Y + (M^*+\tilde{M}) \delta_{\bX} + \Delta_L^*))^2}/4$.
\end{proposition}
This bound explains that the degree of inconsistent risk values, rephrased by the variance of the random downstream model, can be limited if the size of $\delta_{\bX}$ and $\delta_{Y}$ are not too large or the downstream model shows similar performance with the upstream model. The consistency of utility eventually affects the consistency of fairness in the sense that $d(\tih^*_k(\btiX),\tiY)$ can influence the range of $\tilde{\Delta}_F^k$, $2d(\tih^*(\btiX),\tih_k^*(\btiX))+\lambda_F^{-1}C_{M,\delta_{\bX},\delta_Y}$, through the relationship $d(\tih^*(\btiX),\tih_k^*(\btiX))\leq d(\tih^*_k(\btiX),\tiY) + d(\tih^*(\btiX),\tiY)$. Specifically, if the values of $d(\tih^*_k(\btiX),\tiY)$ for $k=1,\dots,K$ are similar, the ranges of $\tilde{\Delta}_F^k $ for different $\tih_k^*$ may not vary much, accordingly reducing the variance of $\tilde{\Delta}_F^k$ under the selection mechanism of Proposition~\ref{prop:coherent}.

\subsubsection{Invariant fairness guarantee for feature engineering}
\label{sec:invariance}

In addition to the consistency property, we further observe that the fairness guarantee of the downstream model holds on post-processed features created from $\btiX$. A downstream user who receives the transformed data $\btiX$ can explore additional features using $\btiX$ before they fit a supervised model. For instance, one may add extra nonlinear features to a linear model for both interpretability and flexibility. Unless the end user imports data other than $\btiX$, the fairness improvement can be controlled by $\tih^*$ as the following proposition implies.  
\begin{proposition}[Post-processing]
    Let $\calH_k(f_e)$ be a composite function class defined as $\{h_k\circ f_e: h_k \in \calH_k\}$ with the post-processed feature-engineering map $f_e: \calX \rightarrow {\cal X}'$ and the supervised model $h_k:{\cal X}'\rightarrow \calY$. If $\calH_k(f_e)\subset \calH$, then the fairness improvement and utility preservation of $\tih_{k|f_e}^*(\btiX)$ can be controlled where $\tih_{k|f_e}^*=\arg_{h_{k}}\min\bE[l(\tiY,h_k(f_e(\btiX)))]$.
\end{proposition}
In a nutshell, if the use of customized features $f_e$ does not exceed the capacity over the auxiliary class $\calH$, the fairness improvement when using post-processed features can also be controlled by $\tih^*$. Hence, as long as $\calH$ in solving \eqref{opt:min} is sufficiently large to cover a wide range of downstream models, the end users have great flexibility and the data provider can foresee their degree of fairness and utility in advance. 

\subsubsection{Extension to separation} The above trade-off and fairness improvement analysis for independence can be extended to the notion of separation, i.e., $\tih^*(\btiX)\perp \bA|Y$, by the conditional HGR correlation. See SM~\ref{supp:sep} to see more details.

\section{Algorithm}
\label{sec:algorithm}

\subsection{Approximating HGR correlation}
\label{sec:approx_hgr}
The HGR correlation is not directly accessible in general due to the nonlinear nature of potential functions. Estimation of $\rho$ has been readily attempted in various areas \cite{brei:frie:85,wang:etal:19}. Particularly in the fairness literature, \cite{lee:etal:22} employed a relaxed version of HGR computation via neural networks. Especially, \cite{mary:etal:19} exploited the well-known approximation property of the HGR correlation for fair supervised learning via in-processing. See \cite{reny:59} for more details as well.
\begin{proposition}
\label{thm:hgr}
For discrete random variables $S_1,S_2$, let's denote by $P_{S_1,S_2}$, $P_{S_1}$, and $P_{S_2}$ the joint distribution of $S_1$ and $S_2$, and marginal mass functions of $S_1$ and $S_2$ respectively. Then $\rho(S_1,S_2)^2 \leq \chi^2(P_{S_1,S_2},P_{S_1}\otimes P_{S_2})$.
\end{proposition}
The equality holds when both $S_1$ and $S_2$ are binary. In the work of \cite{mary:etal:19}, they approximated the $\chi^2$-divergence by density estimation with the Gaussian kernel to deal with continuous variables. Recently, \cite{sohn:etal:23} approached approximating the $\chi^2$-divergence via the dual representation as follows, 
\begin{align*}
    \chi^2(P_{\barh^*(\bbarX),\bA},P_{\barh^*(\bbarX)}\otimes P_{\bA}) \geq \sup_{V\in{\cal V}} R_V(\barh^*(\bbarX),\bA), 
\end{align*}
where the dual of $\chi^2$-divergence on the right term comes with $R_V(\barh^*(\bbarX),\bA) = \bE_{P_{\barh^*(\bbarX),\bA}}[V(\barh^*(\bbarX),\bA)] - \bE_{P_{\barh^*(\bbarX)}\otimes P_{\bA}}[f^*(V(\barh^*(\bbarX),\bA))]$ where $f^*(x)=x^2/4+x$. Under regular assumptions, the approximation is tight with optimal $V^*$. Based on this background, $\rho(\barh^*(\bbarX),\bA)$ can be approximated by $R_{V^*}(\barh^*(\bbarX),\bA)$ while $V^*$ is found through an additional maximization step for $R_V(\barh^*(\bbarX),\bA)$ w.r.t. $V\in {\cal V}$. Likewise, we can find a dual term for separation, following the idea of \cite{roma:etal:20} that measures discrepancy between $(\barh^*(\bbarX),\bA,Y)$ and $(\barh^*(\bbarX),\bA',Y)$ where $\bA'$ independently follows $A|Y$. See SM~\ref{supp:sep-dual}.

\subsection{Min-max bilevel optimization}
The critical premise of the proposed framework is that $\calH$ should be sufficiently large such that the class includes as many downstream supervised models as possible. In this regard, we suggest to choose $\calH$, $\cal G$ and $\cal V$ as families of neural networks for sufficient flexibility and convenient optimization (i.e., with the use of gradient descent/ascent algorithms). With the extra maximization step for the approximation of the HGR correlation, the original optimization \eqref{opt:min} is now written as a constrained min-max bilevel optimization 
\begin{align}
\label{opt:minmax}
    \min_{G \in {\cal G}}\max_{V\in {\cal V}} \bE[l(\barY,\barh^*(\bbarX))] + \lambda_F R_V(\barh^*(\bbarX),\bA),
\end{align}
where $G$ and $\barh^*$ are subjected to the constraints in \eqref{opt:min}. In this work, the constraints $\Delta_{\bX}(\bX,\bbarX)\leq \delta_{\bX}$ and $\Delta_{Y}(Y,\barY)\leq \delta_{Y}$ are handled under a Lagrangian formulation because it can be easily updated by the dual descent approach as \cite{fior:etal:21}. 

Bilevel optimization is widely applied in machine-learning problems such as meta-learning and hyperparameter optimization. Typically, its solution is obtained through alternating gradient-based updates between inner and outer variables due to the lack of a closed-form solution for the inner variable. This work employs an iterative differentiation scheme for sequential parameter updates, offering simple implementation and supported by its decent performance in various applications \cite{finn:etal:17,ji:etal:21}. To write the algorithm, let's denote by $L(G_{\bX};G_Y,\bar{h}^*,V,\lambda_{\bX})=\bE[l(\barY,\barh^*(\bbarX))] + \lambda_F R_V(\bar{h}^*(\bbarX),\bA)+\lambda_{\bX}\tau(\Delta_{\bX}(\bX,\bbarX)-\delta_{\bX})$ for $G_{\bX}$ and $L(G_Y;G_X,\bar{h}^*,\lambda_Y)=\bE[l(\barY,\barh^*(\bbarX))]+\lambda_Y\tau(\Delta_{Y}(Y,\barY)-\delta_{Y})$ for $G_Y$, respectively where $\tau(x)=\max\{x,0\}$. The suggested algorithm to solve \eqref{opt:minmax} appears in Algorithm~\ref{alg}.

Algorithm~\ref{alg} shows the update step for each component in detail. In each step, it updates the lower (max-)minimization for $T'$ times, the upper minimization, and updates the Lagrangian multipliers $\lambda_{\bX}$ and $\lambda_Y$ for the constraints of $\Delta_X$ and $\Delta_Y$ respectively. The superscript $(t)$ denotes the $t$th iterate of its base expression. If $\Delta_{\bX} > \delta_{\bX}$ and $\Delta_Y > \delta_Y$, the multipliers will gradually increase until the constraints are satisfied. As needed, one may force every single variable $X_j$ to have each $\Delta_{X_j}\leq \delta_{X_j}$ respectively, which then requires $d_{\bX}$ number of independent multipliers. 

For coherency with our previous argument, Algorithm~\ref{alg} is written in terms of the population level. Since the joint distribution $(\bX,\bA,Y)$ is unknown, the expectation is commonly estimated by its empirical version based on $n$ size random samples $(\bX_i,\bA_i,Y_i)$ for $i=1,\dots,n$. For instance, $\frac{1}{n}\sum_{i=1}^n L(\bar{Y}^{(t)}_i, \bar{h}^{(t+1)}(\bar{\bX}^{(t)}_i))$ is an unbiased estimator of $\bE[L(\bar{Y}^{(t)}, \bar{h}^{(t+1)}(\bar{\bX}^{(t)}))]$ for a given $\bar{h}^{(t+1)}$ if the random samples are identically distributed. For computational efficiency, the algorithm adopts the mini-batch scheme. Other population quantities are estimated in the same manner. 

\begin{algorithm}[ht!]
\scriptsize
\caption{Task-tailored Pre-processing}
\label{alg}
\begin{algorithmic}[1]
\REQUIRE Data $(\mathbf{X}, \mathbf{A}, Y)$; initialized neural nets $\bar{h}^{(0)}$, $V^{(0)}$, $G^{(0)} = (G_{\mathbf{X}}^{(0)}, G_Y^{(0)})$; variables $\lambda_{\mathbf{X}}^{(0)} \gets 0$, $\lambda_Y^{(0)} \gets 0$; learning rate $r_{\cdot} > 0$ for each component; total number of iterations $T$. Note $\bar{h}^{(t,1)}=\bar{h}^{(t)}$ and $\bar{h}^{(t+1)}=\bar{h}^{(t,T')}$, and for $V^{(t)}$.
\FOR{$t = 1$ to $T$}
    \FOR{$t'=1$ to $T'$}
    \vspace{0.2cm}
    \STATE $(\bar{\mathbf{X}}^{(t)}, \bar{Y}^{(t)}) = (G_{\mathbf{X}}^{(t)}(\mathbf{X}, \mathbf{A}), G_Y^{(t)}(\mathbf{X}, \mathbf{A}, Y))$
    \vspace{0.2cm}
    \STATE $\bar{h}^{(t,t'+1)} \gets \bar{h}^{(t,t')} - r_h \partial {\bE}[L(\bar{Y}^{(t)}, \bar{h}(\bar{\mathbf{X}}^{(t)}))]/\partial \bar{h}^{(t,t')}$
    \vspace{0.2cm}
    \STATE $V^{(t,t'+1)} \gets V^{(t,t')} + r_V \partial R_{V^{(t)}}(\bar{h}^{(t,t'+1)}(\bar{\mathbf{X}}^{(t)}), \mathbf{A})/\partial V^{(t,t')}$
    \vspace{0.2cm}
    \ENDFOR
    \vspace{0.2cm}
    \STATE $\lambda_{\mathbf{X}}^{(t+1)} \gets \lambda_{\mathbf{X}}^{(t)} + r_{\mathbf{X}} \tau(\Delta_{\mathbf{X}}(\mathbf{X}, \bar{\mathbf{X}}^{(t)}) - \delta_{\mathbf{X}})$
    \vspace{0.2cm}
    \STATE $\lambda_Y^{(t+1)} \gets \lambda_Y^{(t)} + r_Y \tau(\Delta_Y(Y, \bar{Y}^{(t)}) - \delta_Y)$
    \vspace{0.2cm}
    \STATE $ G_{\mathbf{X}}^{(t+1)} \gets G_{\mathbf{X}}^{(t)} - r_G \partial L(G_{\mathbf{X}}^{(t)};G_Y^{(t)},\bar{h}^{(t+1)},V^{(t+1)},\lambda_{\bX}^{(t+1)})/\partial  G_{\mathbf{X}}^{(t)}$ 
    \vspace{0.2cm}
    \STATE $G_Y^{(t+1)} \gets G_Y^{(t)} - r_G\partial L(G_Y^{(t)};G_{\mathbf{X}}^{(t)},\bar{h}^{(t+1)},\lambda_Y^{(t+1)})/\partial G_Y^{(t)}$
    \vspace{0.2cm}
\ENDFOR
\RETURN $G^{(T)}$
\end{algorithmic}
\end{algorithm}

\section{Simulation}
\label{sec:simul}
This section presents our empirical results. All methods are implemented 5 times independently. We report averaged metric values and standard errors (or standard deviation unless otherwise mentioned) in parentheses based on the five runs. 
SM~\ref{sup:simul} includes more details about the simulation setups and implementation of competing methods, such as benchmark data sets, network architectures, and hyperparameters. Our implementation is available at the author's Gihub\footnote{\url{https://anonymous.4open.science/r/fair-downstream-learning/}}.

\subsection{Application for Tabular Data}

\subsubsection{Setting}
Our arguments are validated through comparison studies with existing pre-processing approaches. Competing methods include: reflecting instance weights for training (Reweighing) \citep{kami:cald:12}, 
generating fairness-aware synthetic data by a generative adversarial network (FairTabGAN) \citep{raja:etal:22}, linearly interpolating data instances across different sensitive groups (MultiFair) \citep{tian:etal:24}, and label randomization via Bayes classifier (FairRR) \citep{zeng:etal:24}. The FairRR is considered a primary competitor, since it explicitly employs a supervised learning model, similar to the upstream model in our context, to obtain a fair label. Other competing methods, in contrast, can be seen to belong to data fairness that directly modifies the data distribution. FairRR is implemented using logistic regression (FairRR-LR) and random forest (FairRR-RF). 

The simulation considers five downstream models readily available in {\it scikit-learn}\footnote{scikit-learn is one of the most popular machine learning (ML) packages in Python. 
This open-source platform provides easy-to-use functions for end users to construct and implement a machine-learning pipeline. For more details, refer to https://scikit-learn.org/.} for comparison: RandomForestClassifier, LinearSVC with the kernel trick Nystroem, KNeighborsClassifier, MLPClassifier, and HistGradientBoostingClassifier. 
The Kernel expanded LinearSVC model fits a linear support vector machine after expanding the feature space of $\btiX$ via the kernel trick to see whether the post-processing property holds (Section~\ref{sec:invariance}).

Three popular benchmark data sets are considered: (\texttt{A}) Adult\footnote{We use the data used in the work of \cite{zeng:etal:24}.}, (\texttt{E}) ACSEmployment\footnote{https://github.com/socialfoundations/folktables \label{data:folk}}, and (\texttt{P}) ACSPublicCoverage\footref{data:folk}. For Adult, the task is to predict whether an individual income exceeds \$50,000 or not, where $A_{\text{bin}}$ is `Gender'; ACSEmployment data is to predict whether an individual is employed or not, where $A_{\text{bin}}$ is `Disability'; and ACSPublicCoverage is to predict whether an individual has public insurance or not, where $A_{\text{bin}}$ is `Disability'. Because all data sets include Race (white vs non-white) information, we create ${\bf A}_{\text{dis}}$ (with 4 levels) by concatenating the Race as an additional sensitive attribute in addition to the prespecified $A_{\text{bin}}$. We abuse the notation $A$ for both $A_{\text{bin}}$ and ${\bf A}_{\text{dis}}$ if there is no confusion.

Since competing methods only apply to discrete-format sensitive attributes or classification problems, our simulations focus on handling $A_{\text{bin}}$ or $\bf A_{\text{dis}}$. However, we note that our framework is universally applicable to any format of sensitive attributes and outcome variables. 

To assess the classification performance of the models, the area under the curve (AUC) is measured. Higher AUC implies better predictive accuracy. The deterministic label $\hat{Y}$ of each model is obtained by maximizing the AUC value. Fairness evaluation is based on the following metrics: $\text{SP} = \sum_{a} |P(\hat{Y}=1|A=a)/P(\hat{Y}=1)-1|$ and $\text{EO} = \sum_{y,a} |P(\hat{Y}=y|A=a,Y=y)/P(\hat{Y}=y|Y=y)-1|$
for statistical parity (independence) and equalized odds (separation), respectively. In addition, we use Kolmogorov-Smirnov (KS) statistics $\text{KS-SP} = \sum_{a} \text{KS}(\tih_k^*(\btiX)|_{A=a},\tih_k^*(\btiX))$ and $\text{KS-EO} = \sum_{y,a} \text{KS}(y+(-1)^y\tih_k^*(\btiX)|_{A=a,Y=1-y},y+(-1)^y\tih_k^*(\btiX)|_{Y=1-y})$. For instance, $1-\tih_k^*(\btiX)|_{A=a,Y=0}$ denotes the one-minus score of the $k$th downstream model given ${A=a,Y=0}$. The models have better fairness satisfaction as (KS-)SP and (KS-)EO are lowered. Table~\ref{tab:basic_score} shows the predictive precision and degree of discrimination for MLPClassifier and RandomForestClassifier for $A_{\text{bin}}$ that are trained and validated in the original data $\bD$ without fairness control (Refer to Table~\ref{tab:performance-metrics-all} to see the scores of all supervised models in the study).
\begin{table}[ht!]
\centering
\caption{Prediction/fairness scores of MLPClassifier (MP) and RandomForestClassifier (RF) on $\bD$: standard deviations appear in parentheses.}
\label{tab:basic_score}
\begin{tabular}{c|c|ccc}
\hline
${\calH}_k$  &  Score  & \texttt{A}  & \texttt{E}    &  \texttt{P}    \\
\hline
\multirow{5}{*}{\textsc{MP}} 
 & AUC    & 0.879 (0.002)   & 0.878 (0.002)  & 0.725 (0.002)        \\
 & SP     & 0.832 (0.012)   & 0.654 (0.017)  & 1.073 (0.175)        \\
 & EO     & 0.446 (0.028)   & 0.532 (0.025)  & 0.918 (0.125)        \\
 & KS-SP  & 0.349 (0.009)   & 0.367 (0.007)  & 0.483 (0.013)        \\
 & KS-EO  & 0.429 (0.005)   & 0.521 (0.007)  & 0.784 (0.029)        \\
\hline
\multirow{5}{*}{\textsc{RF}} 
 & AUC    & 0.891 (0.001)   & 0.862 (0.001)  & 0.727 (0.002)        \\
 & SP     & 0.759 (0.040)   & 0.580 (0.010)  & 0.937 (0.022)        \\
 & EO     & 0.298 (0.025)   & 0.447 (0.011)  & 0.870 (0.035)        \\
 & KS-SP  & 0.308 (0.007)   & 0.340 (0.003)  & 0.434 (0.012)        \\
 & KS-EO  & 0.387 (0.024)   & 0.474 (0.014)  & 0.702 (0.033)        \\
\hline
\end{tabular}
\end{table}

\begin{figure*}[ht!] 
    \centering
    \subfloat[Trade-off plots of Adult with $\delta_{X}\in \{0.01,0.10,0.30\}$, $\delta_Y=0$, and $\lambda_F = 10$]{%
        \includegraphics[width=\linewidth]{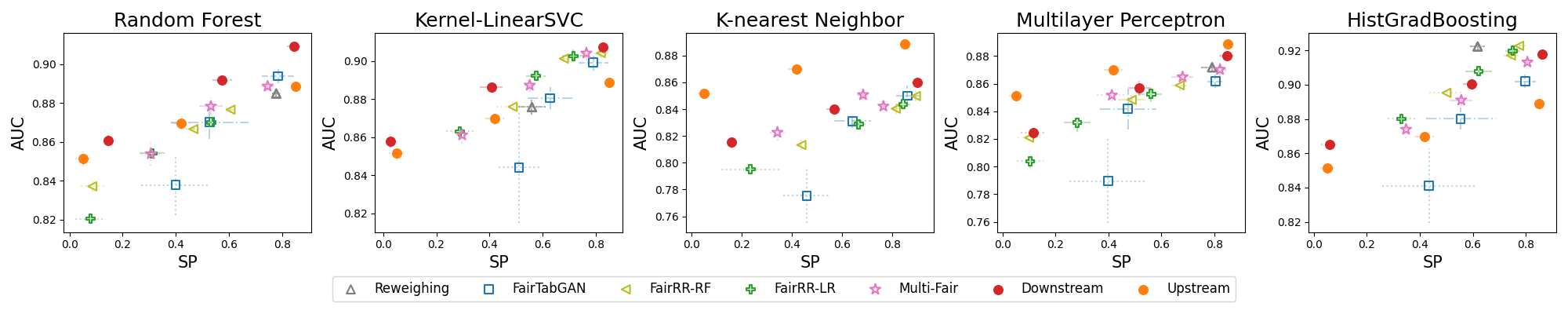}
    }\par
    \subfloat[Trade-off plots of ACSEmployment  with $\delta_{X}\in \{0.01,0.05,0.10\}$,  $\delta_Y=0$, and $\lambda_F = 10$]{%
        \includegraphics[width=\linewidth]{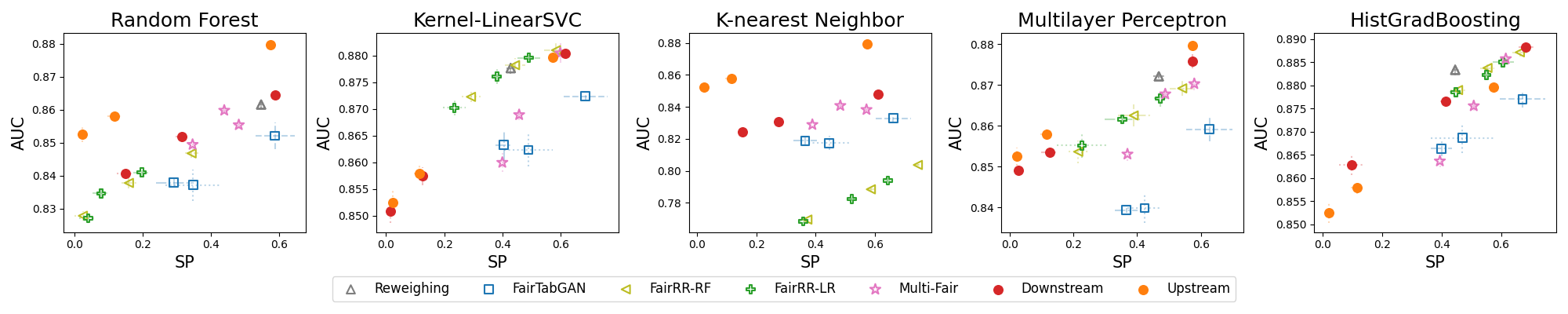}
    }\par
    \subfloat[Trade-off plots of ACSPublicCoverage with $\delta_{X}\in \{0.01,0.05,0.10\}$, $\delta_Y=0.1$, and $\lambda_F = 10$]{%
        \includegraphics[width=\linewidth]{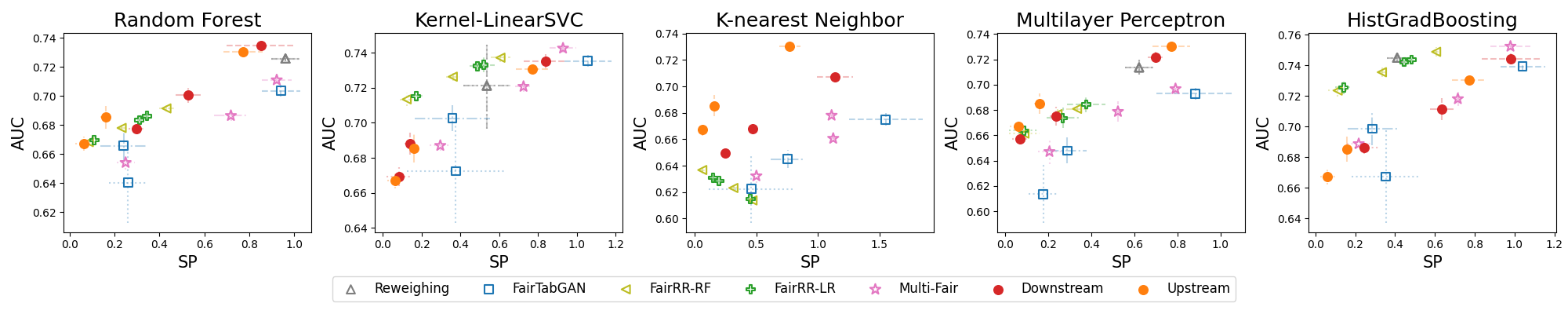}
    }\par
    \caption{Comparison for statistical parity in handling $A_{\text{bin}}$: The closer to the upper left the scores are, the higher performance is achieved. Downstream and Upstream indicate $\tih_k^*$ and $\tih^*$ respectively. The dot and each bar imply the average and $2\times$standard error in each axis. The line styles distinguish Budgets 1 (dashed line), 2 (dash-dot line), and 3 (dotted line).}
    \label{fig:sp}
\end{figure*}

\begin{figure*}
    \centering
    \includegraphics[width=1.0\linewidth]{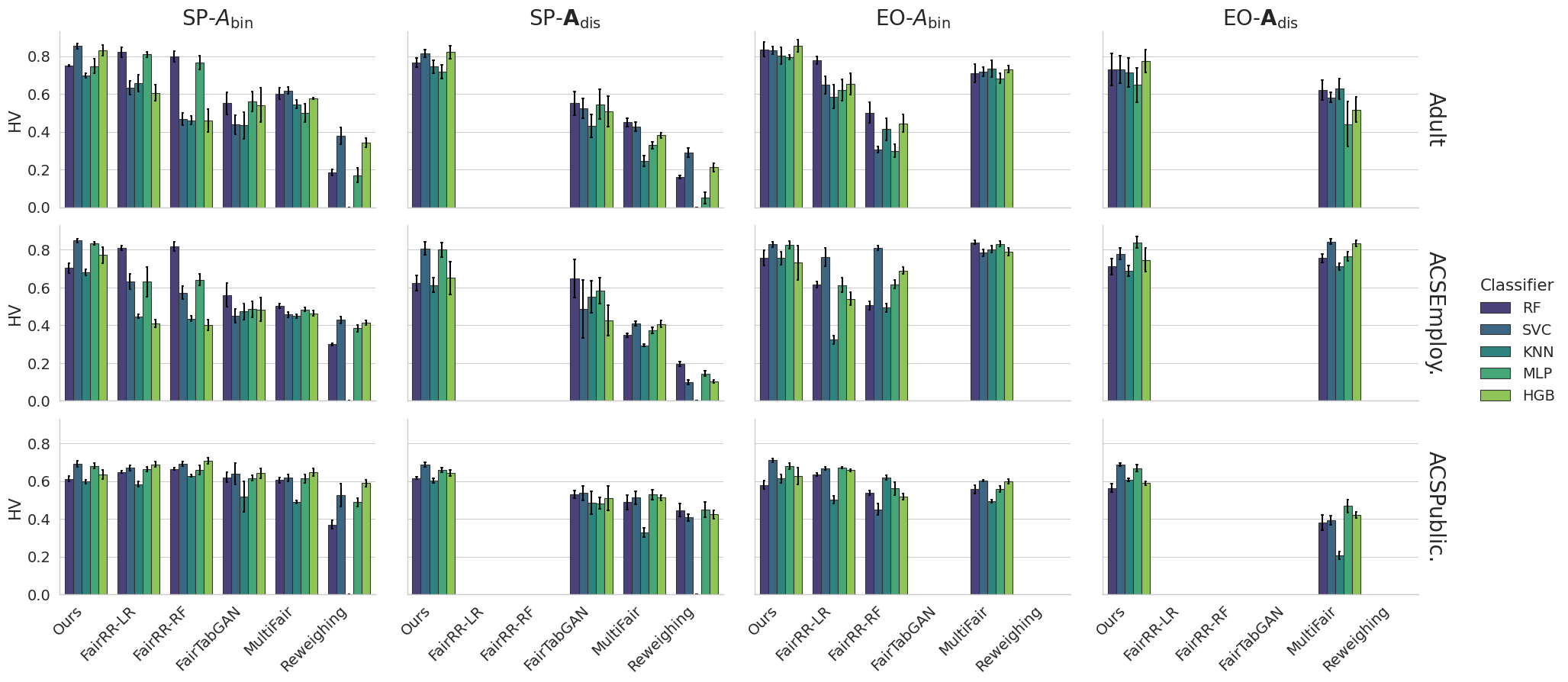}
    \caption{Comparison of HV scores: The hypervolume indicators are calculated based on the three score pairs (by the three budgets) for each competing method and downstream model. 2$\times$Standard errors are drawn from the 5 independent runs. A higher score implies that its corresponding method tends to have a more efficient and diverse trade-off curve. See also Figure~\ref{fig:hv-indicator-ks} for the KS-type fairness score.}
    \label{fig:hv-indicator}
    \vspace{-0.2in}
\end{figure*}

The specific configuration of solving \eqref{opt:minmax} is as follows. The auxiliary model $h$ is trained by the generalized cross-entropy loss. We set $T'=1$ in Algorithm~\ref{alg}. The metrics $\Delta_{\bX}$ and $\Delta_{Y}$ are defined as the mean-absolute error and the categorical hinge loss, for continuous and categorical variables, respectively. The fairness budget $\delta_X=\delta_{X_i}=\cdots=\delta_{X_{d_{\bX}}}$ is assigned to each variable $i=1,\dots,d_{\bX}$, i.e., each $X_i$ has $\Delta_{X_i}(X_i,\tiX_i)\leq \delta_{X_i}$ and $\delta_{X_i}$ is the same for all variables. For each data, we consider three fairness budgets with respect to $\delta_{X}$; Budgets 1,2, and 3 in ascending order of $\delta_X$. Supported by Remark~\ref{remark:ex}, $\delta_{Y}>0$ is used for PublicCoverage since it has relatively weak predictive accuracy as shown in Table~\ref{tab:basic_score}. We choose $\lambda_F\in \{0.1,1,10,100\}$ that achieves a better fairness-utility trade-off that varies with the size of $\delta_{X}$.  

Except for Reweighing, other competing methods use three hyperparameters to balance the degree of fairness-utility trade-off. The hyperparameters are chosen such that their algorithms find computationally feasible solutions that illustrate trade-off phenomena over a wide score range as much as possible. For clarity, we also use the term ``budget'' to denote the hyperparameter settings of other methods (e.g., in Fig. \ref{tab:consistency}).

\subsubsection{Qualitative and Quantitative Analysis}\label{simul:quali-quant-analysis}
We first illustrate trade-offs between utility and fairness. Figure~\ref{fig:sp} visualizes the trade-off between AUC and SP of the methods in comparison (See also figures in SM~\ref{supp:trade-off-curves} for KS-SP, EO, KS-EO, and the ones for $\bf{A}_{\text{dis}}$, respectively). These trade-off plots show that ours achieves a slightly better trade-off inclination for each downstream model than competitors in general. 
FairRR significantly fails for the KNN model and doesn't consistently show a utility-fairness trade-off inclination (e.g., fairness and utility decrease simultaneously for the ACSPublicCoverage data). Nevertheless, it works well on the linear model with the kernel trick and the boosting model. FairTabGAN has far larger uncertainty bounds than others and compromises more accuracy than ours in most cases. MultiFair shows better trade-off tendency than FairTabGAN overall, but it covers a relatively narrower range of fairness than ours. Its linear interpolation structure may fail to tackle more nonlinear and complex discriminative effects among variables. Reweighing shows stronger performance for the boosting model, but it is defeated by other downstream models, and also does not support controlling the level of fairness. While our method achieves more efficient and diverse trade-off curves in general, we acknowledge that it involves relatively more expensive computation, e.g., GPU dependency.

To further support our arguments on the trade-off curves, we evaluate the Hypervolume (HV) indicator, a widely adopted metric in multi-objective optimization for assessing the quality of a Pareto front \citep{zitz:thie:02,guer:etal:21}. A Pareto front is defined as a set of solutions that are not dominated by any other solution. For instance, considering tuples of (1-AUC, Scaled SP), i.e., scaling SP to lie in $[0,1]$,  where lower is better, if we have $r_1=(0.1,0.1)$, $r_2=(0.2,0.2)$, and $r_3=(0.05,0.15)$, $r_2$ is dominated by both $r_1$ and $r_3$. However, $r_1$ and $r_3$ do not dominate each other; thus, $\{r_1,r_3\}$ constitutes the Pareto front. The HV indicator quantifies the volume of the objective space dominated by an estimated front $S$ and bounded by a reference point $r_*$ (e.g., $r_*=(1,1)$, representing a totally inaccurate and discriminatory model). Formally, it is defined as the Lebesgue measure $\lambda(\cup_{p\in S}\{q: p \leq q \leq r_*\})$ \citep{guer:etal:21}. In our context, a higher HV score signifies that a method achieves better accuracy across a wider range of fairness levels. Figure~\ref{fig:hv-indicator} shows that our method achieves higher HV values in general. Notably, Reweighing yields the lowest score because its frontier consists of only a single point. This diversity aspect of the HV indicator is practically important for providing flexible options for decision-making.

In addition to evaluating the trade-off for each model, we further assess the consistency among downstream models. Let $s_{\text{AUC-SP/EO},k,i}$ and $s_{\text{(KS-)SP/EO},k,i}$ be AUC and (KS-)SP/EO values respectively under the SP/EO penalization, for the $i$th run of the $k$th downstream model. Based on the discussion of Section~\ref{sec:consistency}, we define consistency scores for the utility and fairness as the sample standard deviation $\sigma_{\text{AUC-SP/EO},i}=\sigma(s_{\text{AUC-SP/EO},k,i})$ and $\sigma_{\text{(KS-)SP/EO},i}=\sigma(s_{\text{(KS-)SP/EO},k,i})$ across $k=1,\dots, 5$. Then we report the average and standard deviation across $i=1,\dots,5$ in Figure~\ref{tab:consistency}. 
Intuitively, the higher consistency scores imply that the transformed data set $\btiD$ would produce substantially different levels of fairness and utility for different downstream models. 
Figure~\ref{tab:consistency} summarizes the consistency scores of all competing methods for SP in $A_{\text{bin}}$. We observe that when the substantially non-Lipschitz boosting model is omitted (i.e., hollow boxplots in Fig.\ref{tab:consistency}), the consistency scores generally improve (see Figure~\ref{fig:consistency-comparison-hgb}  together for other cases in SM~\ref{supp:const-scores}). This emphasizes the importance of Lipschitzness of downstream models discussed in Section~\ref{sec:consistency}. 
Our approach tends to have lower consistency scores than FairRR, the major competing method, but it seems comparable to the data-fairness methods. That is possibly because the data-fairness methods do not involve the complexity of the supervised model to enforce fairness.
\begin{figure}
    \vspace{-0.2in}
    \centering
    \includegraphics[width=1.0\linewidth]{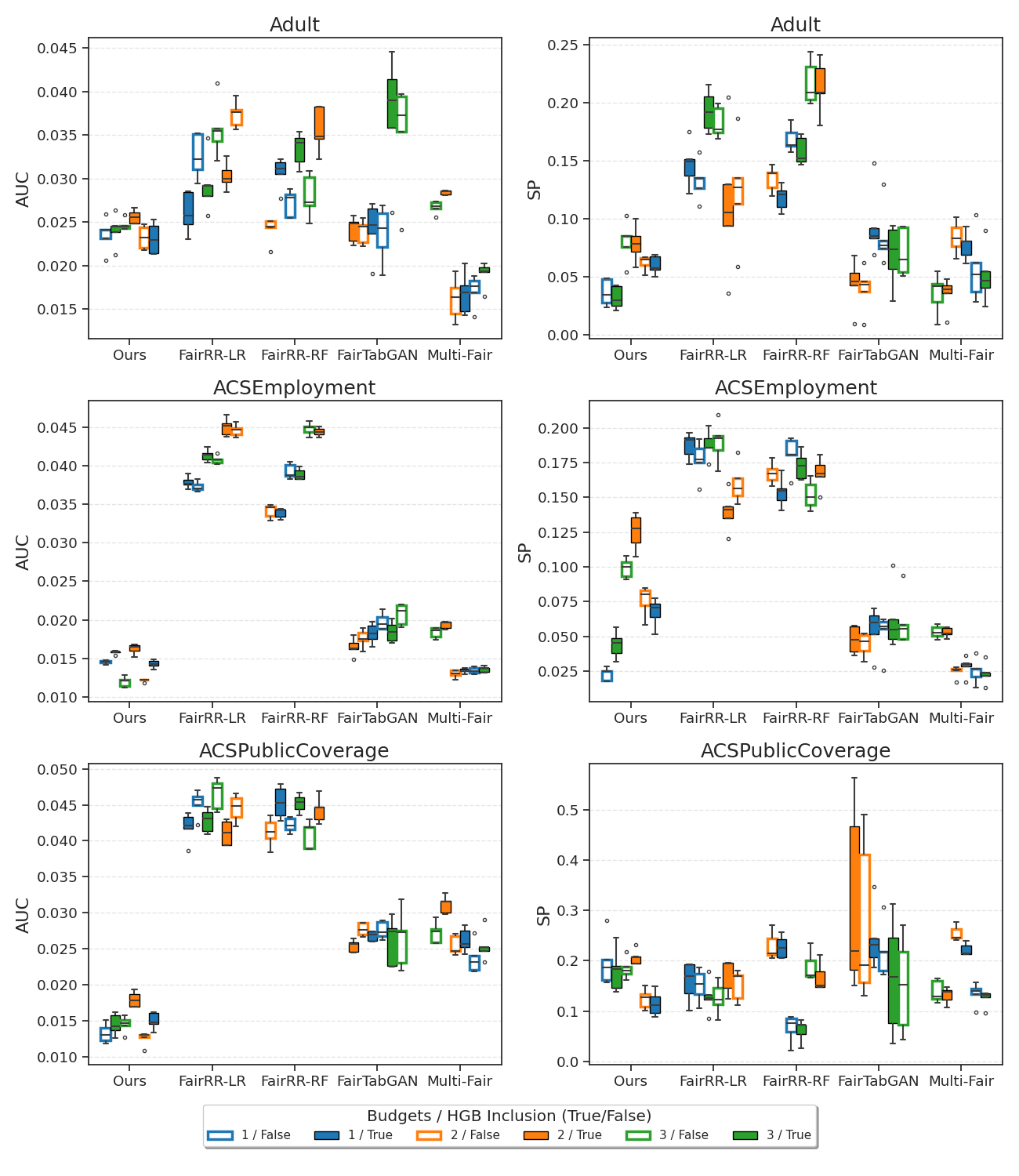}
    \vspace{-0.2in}
    \caption{Consistency of AUC and SP in $A_{\text{bin}}$: Boxplots compare consistency scores across competing methods and the inclusion of the boosting model in calculating the scores.}
    \label{tab:consistency}
    \vspace{-0.2in}
\end{figure}

\subsubsection{Effects of using $\delta_Y>0$}
\label{sec:using_deltay} Also, using $\tiY$ can improve consistency when the target data has an inherently difficult prediction task. Recall that the degree of accuracy of the upstream and downstream models influences the degree of fairness improvement (Section~\ref{sec:down}). Table~\ref{tab:consistency_public} compares the consistency scores under $\delta_Y=0$ and $\delta_Y=0.1$ of ACSPublicCoverage that has a weaker prediction power than others (Table~\ref{tab:basic_score}). Interestingly, the table shows that using $\tiY(\neq Y)$ in training yields better consistency of utility and fairness among downstream models with the improved utility on the transformed label.

\subsubsection{Tuning Guidance}
For users' convenience in hyperparameter tuning, we recommend using, or searching from, a relatively higher fixed $\lambda_F$ and altering $\delta_{\bX}>0$ to capture visually well interpretable trade-off curves or higher HV scores. Then try to tune $\delta_Y>0$ when adjusting $\lambda_F$ and $\delta_{\bX}>0$ is not satisfactory, referring to discussion the above Section~\ref{sec:using_deltay}.

\subsubsection{Comparison to in-processing}
Finally, we close this section with a brief mention of the comparison to in-processing frameworks. Overall, in-processing methods achieve slightly more efficient trade-off curves than ours (see SM~\ref{supp:in-proc}), not surprisingly, because they directly optimize a classification model itself to be most accurate and fair. Notwithstanding, pre-processing enjoys fundamental advantages over in-processing. It enables fairness in legacy systems where modifying an objective or model is impossible. Furthermore, it empowers distributors to guarantee fairness proactively, ensuring that fairness criteria have to be met regardless of the downstream user's ethicality or capability. 

\begin{table}[ht!]
\centering
\caption{Comparison of $\delta_Y=0$ and $\delta_Y=0.1$ in ACSPublicCoverage: AUC-MP and AUC-RF stand for AUC values found from transformed labels $\tiY'$ of test data. Consistency scores of utility are also made from the transformed label. This result justifies the argument in Remark~\ref{remark:ex}.}
\label{tab:consistency_public}
\setlength{\tabcolsep}{2pt} 
{\scriptsize 
\begin{tabular}{c|ccc|ccc}
\hline 
 & \multicolumn{3}{c|}{$\delta_Y=0$} & \multicolumn{3}{c}{$\delta_Y=0.1$} \\
\hline
 & Budget 1 & Budget 2 & Budget 3 & Budget 1 & Budget 2 & Budget 3 \\
\hline
AUC-MP & .715 (.004) & .674 (.008) & .658 (.005) & .804 (.004) & .761 (.006) & .748 (.005) \\ 
AUC-RF & .731 (.002) & .706 (.004) & .678 (.002) & .816 (.004) & .781 (.006) & .762 (.004) \\  
\hline 
$\sigma_{\text{AUC-SP}}$ & .023 (.001) & .028 (.001) & .024 (.001) & \textbf{.015 (.001)} & \textbf{.018 ($\approx 0$)} & \textbf{.015 (.001)} \\
$\sigma_{\text{SP}}$ & .287 (.109) & .292 (.008) & .128 (.011) & \textbf{.181 (.019)} & \textbf{.207 (.007)} & \textbf{.115 (.011)} \\
$\sigma_{\text{KS-SP}}$ & .050 (.003) & .123 (.004) & .054 (.005) & \textbf{.042 (.006)} & \textbf{.087 (.003)} & \textbf{.044 (.003)} \\
\hline 
\end{tabular}
}
\end{table}

\subsection{Application for CelebA}

We further apply our framework to an image classification problem on CelebA data. CelebA is a collection of face images, widely used in diverse areas of machine learning, such as face recognition, image synthesis, etc. This simulation study chooses binary variables Smiling ($Y=1$ for smile faces) and Blond Hair ($Y=1$ for blond hair) as the target variables for SP and EO, respectively, where $A$ is whether the person in the image is male ($A=1$) or female ($A=0$). For the upstream classification model, a ResNet18 model \cite{he:etal:16} is trained with the generalized cross-entropy. The downstream task scales down all hidden filters of the same model by half and trains it by the cross-entropy loss. To generate synthesized images, $G_{\bX}$ adopts the U-net structure \cite{ronn:etal:15} that bridges the encoder and decoder layers. For the approximation of HGR correlation, $V$ employs a common feed-forward model whose hidden node size is set to 256. With the adoption of the $L_1$ distance for $\Delta_{\bX}$, three different levels of $\delta_{\bX}\in\{0,0.007,0.01\}$ for SP and $\delta_{\bX}\in\{0,0.010,0.015\}$ for EO are considered with $\delta_Y=0$. Further details in this section appear in SM~\ref{supp:simul:celeba}. 

Interestingly, our results find that the converter $G_{\bX}^*$ tends to alter semantic features, balancing the classification probabilities of each group. To see this, we focus on images for which at $\delta_{\bX}=0$ the upstream model predicts correctly, but at $\delta_{\bX}>0$ of the upstream model predicts incorrectly. Figures~\ref{fig:celeba_sp} and~\ref{fig:celeba_eo} draw one such image example within each group of $(Y,A)$ under the three choices of $\delta_{\bX}$, for SP and EO, respectively. In Figure~\ref{fig:celeba_sp} with the control of SP, generated female image examples become less smiling and the male examples become more smiling as $\delta_{\bX}$ increases, and Table~\ref{tab:celeb_info} supports such evolution of semantic features in each group. Likewise, Figure~\ref{fig:celeba_eo} shows that the hair color in the male image examples becomes lighter blond and in the female examples darker. The changes in probability for each group are summarized in Table \ref{tab:celeb_info_eo}.

\begin{table}[ht!]
\caption{Results of the upstream model with different $\delta_{\bX}$ for predicting Smiling. standard deviations appear in parentheses.}
\label{tab:celeb_info}
\centering
\begin{tabular}{c|cc}
\toprule
& \multicolumn{2}{c}{Conditional Probabilities on $A$}              \\
\midrule
$\delta_{\bX}$ & {$P(\tih^*(\btiX)=1|A=1)$} & {$P(\tih^*(\btiX)=1|A=0)$} \\
\midrule
.0 & {.397 (.013)}    &  {.531 (.014)}    \\
.007 & {.403 (.010)}    &  {.523 (.006)}    \\
.010 & {.427 (.016)}    &  {.511 (.023)}    \\
\bottomrule
\end{tabular}
\end{table}

\begin{figure}[!t] 
    \centering
    \subfloat[ $Y=1$ and $A=1$]{%
        \includegraphics[width=\linewidth]{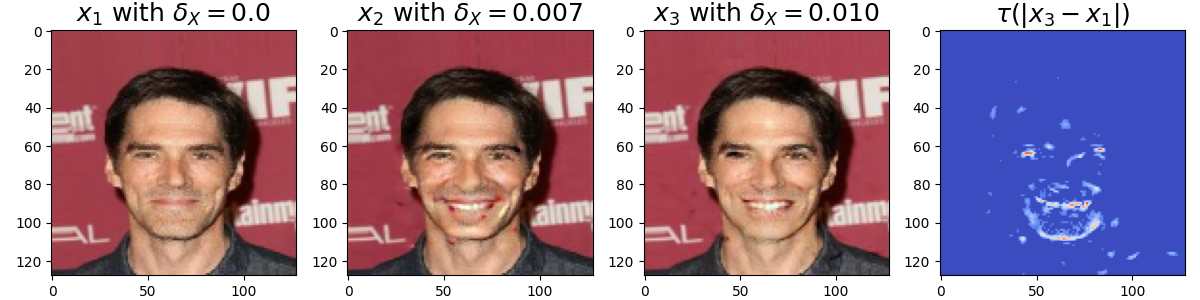}
    }\par
    \subfloat[$Y=1$ and $A=0$]{%
        \includegraphics[width=\linewidth]{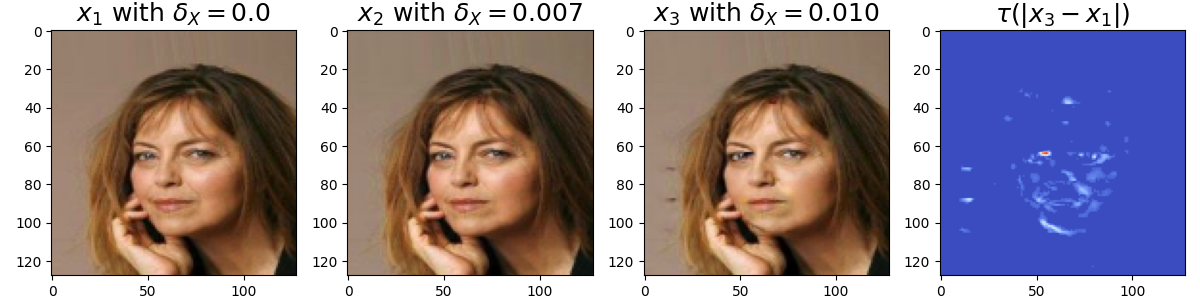}
    }\par
    \subfloat[$Y=0$ and $A=1$]{%
        \includegraphics[width=\linewidth]{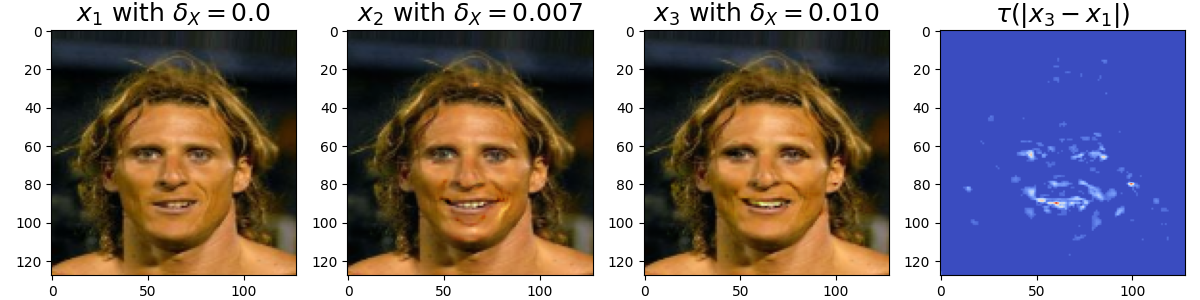}
    }\par
    \subfloat[$Y=0$ and $A=0$]{%
        \includegraphics[width=\linewidth]{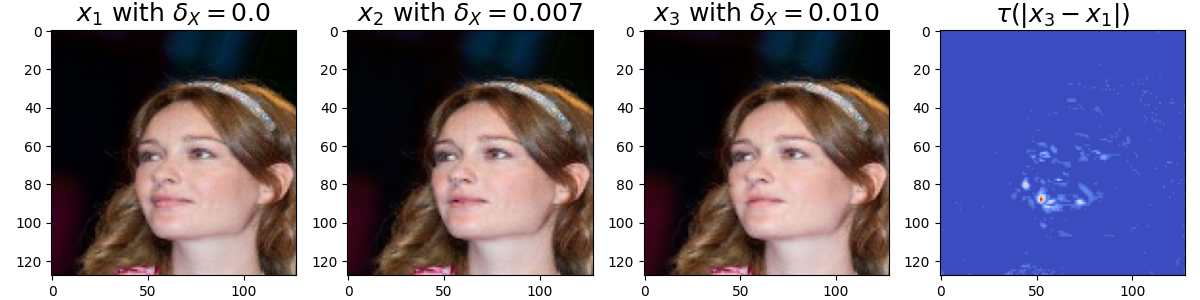}
    }
    \caption{Results of $\btiX$ trained under statistical parity by differing $\delta_{\bX}$. The filtering function $\tau$ replaces pixels by 0 that are smaller than the size of $95\%$ quantile of $|\textsc{x}_3-\textsc{x}_1|$. Male/female images become more/less smiling to balance the smiling probability of each subgroup in Table~\ref{tab:celeb_info}.}
    \label{fig:celeba_sp}
\end{figure}

\begin{figure}[!t] 
    \centering
    \subfloat[$Y=1$ and $A=1$]{%
        \includegraphics[width=\linewidth]{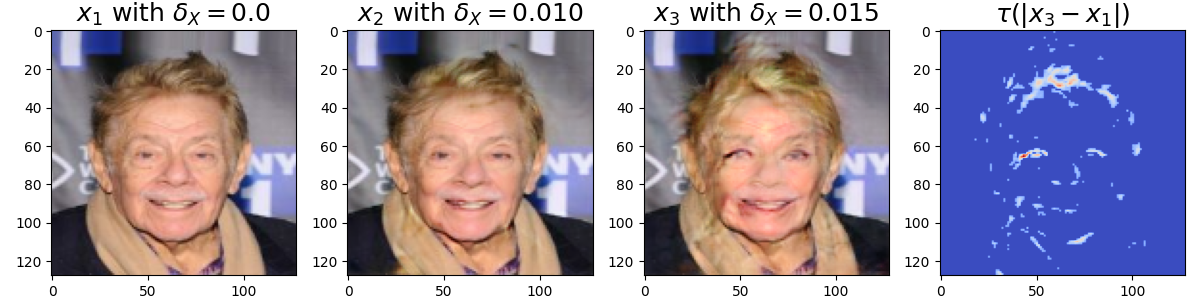}
    }\par
    \subfloat[$Y=1$ and $A=0$]{%
        \includegraphics[width=\linewidth]{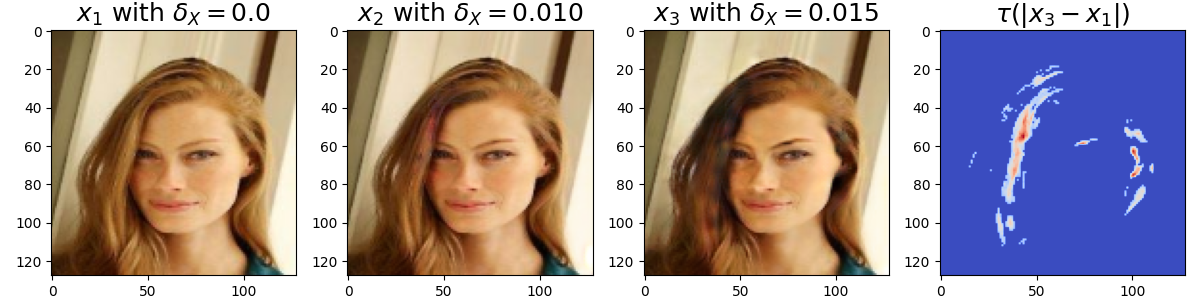}
    }\par
    \subfloat[$Y=0$ and $A=1$]{%
        \includegraphics[width=\linewidth]{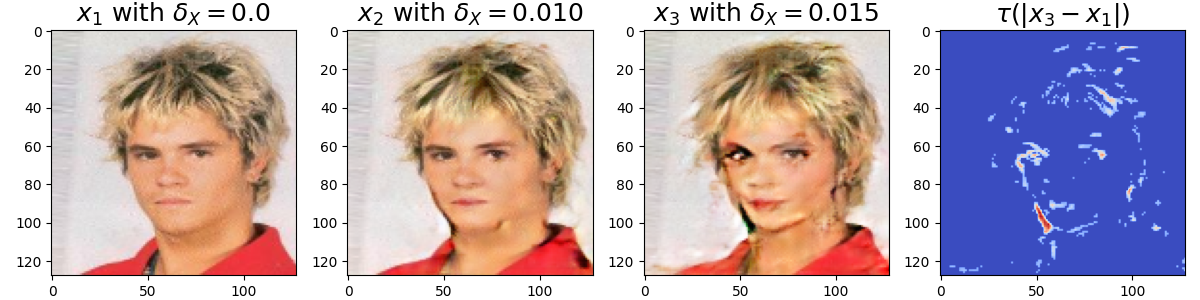}
    }\par
    \subfloat[$Y=0$ and $A=0$]{%
        \includegraphics[width=\linewidth]{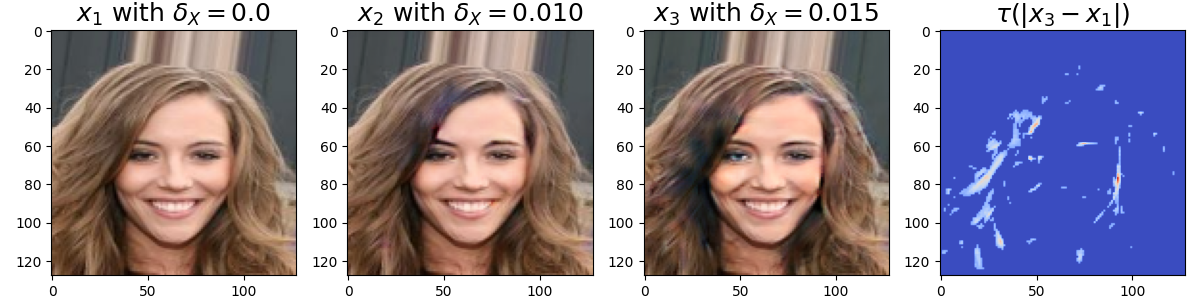}
    }
    \caption{Results of $\btiX$ trained under equalized odds. $\tau$ is the same as Figure~\ref{fig:celeba_sp}. As $\delta_{\bX}$ increases, Male/female images tend to have lighter/darker hair color to balance the conditional probability of each subgroup in Table~\ref{tab:celeb_info_eo}.}
    \label{fig:celeba_eo}
\end{figure}

\begin{table}[ht!]
\caption{Results of the upstream model for different $\delta_{\bX}$ for predicting Blond Hair. Denote $P(\tih^*(\btiX)=y|Y=y,A=a)$ as $P_{y,a}(\hat{Y}=y)$. standard deviations appear in parentheses.}
\label{tab:celeb_info_eo}
\centering
\begin{tabular}{c|cccc}
\toprule
&\multicolumn{4}{c}{Conditional Probabilities on $A$ and $Y$}                  \\
\midrule
$\delta_{\bX}$ & $P_{1,1}(\hat{Y}=1)$  & $P_{1,0}(\hat{Y}=1)$ & $P_{0,1}(\hat{Y}=0)$  & $P_{0,0}(\hat{Y}=0)$ \\
\midrule
.0& .663  (.060)   & .943 (.009)   &  .962 (.008)    &   .882  (.013) \\
.010& .738  (.049)   & .934 (.006)   &  .940 (.009)    &   .885  (.015) \\ 
.015& .784  (.057)   & .921 (.011)   &  .912 (.020)    &   .884  (.011) \\ 
\bottomrule
\end{tabular}
\end{table}
Table~\ref{tab:celeb_down} compares the downstream performance with the existing model (FAAP, \cite{wang:etal:22}). In contrast to ours, FAAP finds a fair pre-processing model for a pre-trained upstream model. Their training does not update the upstream model during the optimization of $G$, so the generated data may fail to consider the downstream performance. For a fair comparison, we implement FAAP by replacing $\bar{h}^{(t,t')}$ with the pre-trained $h^*$ and fix it during iterations of Algorithm~\ref{alg}. The  FAAP model produces data for fitting the downsized ResNet18 model. As Table~\ref{tab:celeb_down} implies, FAAP does not improve the downstream fairness despite of the upstream improvement, while showing our downstream improvement is statistically significant at both stages. As the computational scalability of FairRR highly depends on the implementation of the underlying classifier, we do not address it in this high-dimensional example.

\begin{table}[ht!]
\caption{Comparison between ours and FAAP: Up and Down stand for ResNet18 and downsized ResNet18, respectively. $\Delta$EO and $\Delta$KS-EO denote the average score difference from Original for each metric. Values in parentheses beside the average difference are standard errors. FAAP does not meaningfully improve the downstream fairness (marked $\ddagger$) as much as the upstream fairness in contrast to ours (marked $\dagger$). }
\label{tab:celeb_down}
\centering
\begin{tabular}{c|c|c|ccc}
\toprule
  $\delta_{\bX}$        &  Type           &  $h$    & AUC & EO & KS-EO \\
\midrule
\multirow{2}{*}{.0} &\multirow{2}{*}{Original} & Up & .975 (.003)  & .389 (.058) &  .651 (.050)  \\
                          &                    & Down & .981 (.005)  & .342 (.070) &  .796 (.052)  \\
\midrule
$\delta_{\bX}$        &  Type           &  $h$    & AUC & $\Delta$EO & $\Delta$KS-EO \\
\midrule
\multirow{4}{*}{.010} &\multirow{2}{*}{Ours} & Up   &  .968 (.003)  &  .117 (.038) &  .126 (.039) \\
                      &                      & Down   &  .973 (.010)  &  {$.229~(.034)^{\dagger}$} & {$.182~ (.058)^{\dagger}$} \\
&\multirow{2}{*}{FAAP}                       & Up    & .974 (.002)  &  .190 (.043) & .215 (.054) \\
&                                            & Down    & .981 (.006)  &  {$\approx 0~(.075)^{\ddagger}$}  &  {$-.005~(.046)^{\ddagger}$} \\
\midrule
\multirow{4}{*}{.015} &\multirow{2}{*}{Ours} & Up    &  .963 (.004)   & .207 (.043)  &  .248 (.042) \\
&                                            & Down    &  .971 (.004)   &  {$.231~(.039)^{\dagger}$}  & {$.172~(.034)^{\dagger}$}  \\
&\multirow{2}{*}{FAAP}                       & Up    &  .975 (.002)   &  .160 (.030) &  .204 (.026) \\
&                                            & Down    &  .979 (.005)  & {$-.067~(.079)^{\ddagger}$} & {$.107~(.073)^{\ddagger}$} \\
\bottomrule
\end{tabular}
\end{table}

To yield a large fairness improvement, a large size of $\delta_{\bX}$ is needed. However, the processed image $\btiX$ under large $\delta_{\bX}$ may look distorted and not realistic, because the constraint $\Delta_{\bX}\leq \delta_{\bX}$ does not necessarily preserve semantic features of real face images. As a remedy, one can add a variational loss, such as $f$-divergence or Wasserstein distance between $\bX$ and $\btiX$, to the main objective. Figure~\ref{fig:celeba_eo2} in SM~\ref{supp:variational loss} shows that this variant can produce improved images, but the main text displays the results without using the variational loss for the sake of consistency with the established theories.

\section{Discussion}

While the proposed framework appeals to remarkable advantages such as the control of fairness and utility in the downstream stage, it opens interesting directions to further improvement and extensions. First, our framework requires converting the test data through $G_{\bX}^*$ to align the shifted domain with the training data when evaluating $\tih_k^*$ or $\tih^*$ on the test data. From the perspective of domain adaptation, this conversion process is an effort to alleviate the covariate shift. This adaptation step for the test data may complicate a policy or procedure of distributing data from the centralized manager to downstream users in real applications, where accessing $G_{\bX}^*$ is limited for some reasons. Versatile ideas in domain adaptation would tackle such issues.

Our framework can be extended to other domains, such as text classification \cite{li:etal:22} with a large language model (LLM). We set aside this application for future research since training LLMs immediately following Algorithm~\ref{alg} is computationally too expensive. Therefore, more technical variations, e.g., modifying network structure or foundational optimization for lighter computation, would be required, which is worthy of an independent study from our perspective. In this case as well, $\Delta_{\bX}$ or $\Delta_Y$ should be carefully chosen such that the transformed text is grammatically correct and sounds semantically meaningful.

\bibliographystyle{IEEEtranN}
\bibliography{sample}

@inproceedings{agar:etal:18,
  title={A reductions approach to fair classification},
  author={Agarwal, Alekh and Beygelzimer, Alina and Dud{\'\i}k, Miroslav and Langford, John and Wallach, Hanna},
  booktitle={International conference on machine learning},
  pages={60--69},
  year={2018},
  organization={PMLR}
}

@article{hard:etal:16,
  title={Equality of opportunity in supervised learning},
  author={Hardt, Moritz and Price, Eric and Srebro, Nati},
  journal={Advances in neural information processing systems},
  volume={29},
  year={2016}
}

@inproceedings{feld:etal:15,
  title={Certifying and removing disparate impact},
  author={Feldman, Michael and Friedler, Sorelle A and Moeller, John and Scheidegger, Carlos and Venkatasubramanian, Suresh},
  booktitle={proceedings of the 21th ACM SIGKDD international conference on knowledge discovery and data mining},
  pages={259--268},
  year={2015}
}

@inproceedings{gord;etal;19,
  title={Obtaining fairness using optimal transport theory},
  author={Gordaliza, Paula and Del Barrio, Eustasio and Fabrice, Gamboa and Loubes, Jean-Michel},
  booktitle={International conference on machine learning},
  pages={2357--2365},
  year={2019},
  organization={PMLR}
}

@article{calm:etal:17,
  title={Optimized pre-processing for discrimination prevention},
  author={Calmon, Flavio and Wei, Dennis and Vinzamuri, Bhanukiran and Natesan Ramamurthy, Karthikeyan and Varshney, Kush R},
  journal={Advances in neural information processing systems},
  volume={30},
  year={2017}
}

@inproceedings{xu:etal:18,
  title={Fairgan: Fairness-aware generative adversarial networks},
  author={Xu, Depeng and Yuan, Shuhan and Zhang, Lu and Wu, Xintao},
  booktitle={2018 IEEE International Conference on Big Data (Big Data)},
  pages={570--575},
  year={2018},
  organization={IEEE}
}

@inproceedings{fior:etal:21,
  title={Lagrangian duality for constrained deep learning},
  author={Fioretto, Ferdinando and Van Hentenryck, Pascal and Mak, Terrence WK and Tran, Cuong and Baldo, Federico and Lombardi, Michele},
  booktitle={Machine Learning and Knowledge Discovery in Databases},
  pages={118--135},
  year={2021},
  organization={Springer}
}

@article{chen:etal:23,
  title={On triangle inequalities of correlation-based distances for gene expression profiles},
  author={Chen, Jiaxing and Ng, Yen Kaow and Lin, Lu and Zhang, Xianglilan and Li, Shuaicheng},
  journal={BMC bioinformatics},
  volume={24},
  number={1},
  pages={40},
  year={2023},
  publisher={Springer}
}

@inproceedings{mary:etal:19,
  title={Fairness-aware learning for continuous attributes and treatments},
  author={Mary, J{\'e}r{\'e}mie and Calauzenes, Cl{\'e}ment and El Karoui, Noureddine},
  booktitle={International Conference on Machine Learning},
  pages={4382--4391},
  year={2019},
  organization={PMLR}
}

@article{sohn:etal:23,
  title={Fair Supervised Learning with A Simple Random Sampler of Sensitive Attributes},
  author={Sohn, Jinwon and Song, Qifan and Lin, Guang},
  journal={arXiv preprint arXiv:2311.05866},
  year={2023}
}

@inproceedings{lee:etal:22,
  author       = {Joshua K. Lee and
                  Yuheng Bu and
                  Prasanna Sattigeri and
                  Rameswar Panda and
                  Gregory W. Wornell and
                  Leonid Karlinsky and
                  Rog{\'{e}}rio Feris},
  title        = {A Maximal Correlation Approach to Imposing Fairness in Machine Learning},
  booktitle    = {{IEEE} International Conference on Acoustics, Speech and Signal Processing},
  pages        = {3523--3527},
  publisher    = {{IEEE}},
  year         = {2022},
}

@article{roma:etal:20,
  title={Achieving equalized odds by resampling sensitive attributes},
  author={Romano, Yaniv and Bates, Stephen and Candes, Emmanuel},
  journal={Advances in neural information processing systems},
  volume={33},
  pages={361--371},
  year={2020}
}

@article{zeng:etal:24,
  title={FairRR: Pre-Processing for Group Fairness through Randomized Response},
  author={Zeng, Xianli and Ward, Joshua and Cheng, Guang},
  journal={arXiv preprint arXiv:2403.07780},
  year={2024}
}

@article{petr:etal:22,
  title={FAIR: Fair adversarial instance re-weighting},
  author={Petrovi{\'c}, Andrija and Nikoli{\'c}, Mladen and Radovanovi{\'c}, Sandro and Deliba{\v{s}}i{\'c}, Boris and Jovanovi{\'c}, Milo{\v{s}}},
  journal={Neurocomputing},
  volume={476},
  pages={14--37},
  year={2022},
  publisher={Elsevier}
}

@book{baro:etal:23,
  title={Fairness and machine learning: Limitations and opportunities},
  author={Barocas, Solon and Hardt, Moritz and Narayanan, Arvind},
  year={2023},
  publisher={MIT Press}
}

@article{reny:59,
  title={On measures of dependence},
  author={R{\'e}nyi, Alfr{\'e}d},
  journal={Acta mathematica hungarica},
  volume={10},
  number={3-4},
  pages={441--451},
  year={1959},
  publisher={Akad{\'e}miai Kiad{\'o}, co-published with Springer Science+ Business Media BV~…}
}

@inproceedings{wang:etal:22,
  title={Fairness-aware adversarial perturbation towards bias mitigation for deployed deep models},
  author={Wang, Zhibo and Dong, Xiaowei and Xue, Henry and Zhang, Zhifei and Chiu, Weifeng and Wei, Tao and Ren, Kui},
  booktitle={Proceedings of the IEEE/CVF Conference on Computer Vision and Pattern Recognition},
  pages={10379--10388},
  year={2022}
}

@inproceedings{jang:etal:24,
  title={Adversarial Fairness Network},
  author={Jang, Taeuk and Wang, Xiaoqian and Huang, Heng},
  booktitle={Proceedings of the AAAI Conference on Artificial Intelligence},
  volume={38},
  number={20},
  pages={22159--22166},
  year={2024}
}

@inproceedings{choi:etal:24,
  title={Fair Sampling in Diffusion Models through Switching Mechanism},
  author={Choi, Yujin and Park, Jinseong and Kim, Hoki and Lee, Jaewook and Park, Saerom},
  booktitle={Proceedings of the AAAI Conference on Artificial Intelligence},
  volume={38},
  number={20},
  pages={21995--22003},
  year={2024}
}

@article{tian:etal:24,
  title={MultiFair: Model Fairness With Multiple Sensitive Attributes},
  author={Tian, Huan and Liu, Bo and Zhu, Tianqing and Zhou, Wanlei and Philip, S Yu},
  journal={IEEE Transactions on Neural Networks and Learning Systems},
  year={2024},
  publisher={IEEE}
}

@article{yang:etal:23,
  title={An adversarial training framework for mitigating algorithmic biases in clinical machine learning},
  author={Yang, Jenny and Soltan, Andrew AS and Eyre, David W and Yang, Yang and Clifton, David A},
  journal={NPJ Digital Medicine},
  volume={6},
  number={1},
  pages={55},
  year={2023},
  publisher={Nature Publishing Group UK London}
}

@article{chenp:etal:23,
  title={AI fairness in data management and analytics: A review on challenges, methodologies and applications},
  author={Chen, Pu and Wu, Linna and Wang, Lei},
  journal={Applied sciences},
  volume={13},
  number={18},
  pages={10258},
  year={2023},
  publisher={MDPI}
}

@article{dear:etal:22,
  title={Algorithmic fairness in business analytics: Directions for research and practice},
  author={De-Arteaga, Maria and Feuerriegel, Stefan and Saar-Tsechansky, Maytal},
  journal={Production and Operations Management},
  volume={31},
  number={10},
  pages={3749--3770},
  year={2022},
  publisher={SAGE Publications Sage CA: Los Angeles, CA}
}

@inproceedings{adel:etal:19,
  title={One-network adversarial fairness},
  author={Adel, Tameem and Valera, Isabel and Ghahramani, Zoubin and Weller, Adrian},
  booktitle={Proceedings of the AAAI Conference on Artificial Intelligence},
  volume={33},
  number={01},
  pages={2412--2420},
  year={2019}
}

@inproceedings{zhan:etal:18,
  title={Mitigating unwanted biases with adversarial learning},
  author={Zhang, Brian Hu and Lemoine, Blake and Mitchell, Margaret},
  booktitle={Proceedings of the 2018 AAAI/ACM Conference on AI, Ethics, and Society},
  pages={335--340},
  year={2018}
}

@article{reny:59:2,
  title={New version of the probabilistic generalization of the large sieve},
  author={R{\'e}nyi, A},
  journal={Acta Math. Hung},
  volume={10},
  number={1-2},
  pages={217--226},
  year={1959}
}

@article{gao:etal:20,
  title={Towards convergence rate analysis of random forests for classification},
  author={Gao, Wei and Zhou, Zhi-Hua},
  journal={Advances in neural information processing systems},
  volume={33},
  pages={9300--9311},
  year={2020}
}

@article{moha:24,
  title={THE RISE OF TASK-TAILORED GENERATIVE MODELS: REDEFINING SPECIALIZATION IN ARTIFICIAL INTELLIGENCE},
  author={Mohammad, Roshan},
  journal={INTERNATIONAL JOURNAL OF ENGINEERING AND TECHNOLOGY RESEARCH (IJETR)},
  volume={9},
  number={2},
  pages={19--28},
  year={2024}
}

@inproceedings{ledi:etal:17,
  title={Photo-realistic single image super-resolution using a generative adversarial network},
  author={Ledig, Christian and Theis, Lucas and Husz{\'a}r, Ferenc and Caballero, Jose and Cunningham, Andrew and Acosta, Alejandro and Aitken, Andrew and Tejani, Alykhan and Totz, Johannes and Wang, Zehan and others},
  booktitle={Proceedings of the IEEE conference on computer vision and pattern recognition},
  pages={4681--4690},
  year={2017}
}

@misc{open:23,
  title = {ChatGPT: GPT-4},
  author = {OpenAI},
  year = {2023},
  url = {https://chat.openai.com/},
  note = {Accessed: 2024-09-28}
}

@article{ren:etal:20,
  title={Fastspeech 2: Fast and high-quality end-to-end text to speech},
  author={Ren, Yi and Hu, Chenxu and Tan, Xu and Qin, Tao and Zhao, Sheng and Zhao, Zhou and Liu, Tie-Yan},
  journal={arXiv preprint arXiv:2006.04558},
  year={2020}
}

@article{jord:etal:22,
  title={Synthetic Data--what, why and how?},
  author={Jordon, James and Szpruch, Lukasz and Houssiau, Florimond and Bottarelli, Mirko and Cherubin, Giovanni and Maple, Carsten and Cohen, Samuel N and Weller, Adrian},
  journal={arXiv preprint arXiv:2205.03257},
  year={2022}
}

@article{brei:frie:85,
  title={Estimating optimal transformations for multiple regression and correlation},
  author={Breiman, Leo and Friedman, Jerome H},
  journal={Journal of the American statistical Association},
  volume={80},
  number={391},
  pages={580--598},
  year={1985},
  publisher={Taylor \& Francis}
}

@inproceedings{wang:etal:19,
  title={An efficient approach to informative feature extraction from multimodal data},
  author={Wang, Lichen and Wu, Jiaxiang and Huang, Shao-Lun and Zheng, Lizhong and Xu, Xiangxiang and Zhang, Lin and Huang, Junzhou},
  booktitle={Proceedings of the AAAI Conference on Artificial Intelligence},
  volume={33},
  number={01},
  pages={5281--5288},
  year={2019}
}

@inproceedings{finn:etal:17,
  title={Model-agnostic meta-learning for fast adaptation of deep networks},
  author={Finn, Chelsea and Abbeel, Pieter and Levine, Sergey},
  booktitle={International conference on machine learning},
  pages={1126--1135},
  year={2017},
  organization={PMLR}
}

@inproceedings{ji:etal:21,
  title={Bilevel optimization: Convergence analysis and enhanced design},
  author={Ji, Kaiyi and Yang, Junjie and Liang, Yingbin},
  booktitle={International conference on machine learning},
  pages={4882--4892},
  year={2021},
  organization={PMLR}
}

@article{raja:etal:22,
  title={Tabfairgan: Fair tabular data generation with generative adversarial networks},
  author={Rajabi, Amirarsalan and Garibay, Ozlem Ozmen},
  journal={Machine Learning and Knowledge Extraction},
  volume={4},
  number={2},
  pages={488--501},
  year={2022},
  publisher={MDPI}
}

@article{li:etal:22,
  title={A survey on text classification: From traditional to deep learning},
  author={Li, Qian and Peng, Hao and Li, Jianxin and Xia, Congying and Yang, Renyu and Sun, Lichao and Yu, Philip S and He, Lifang},
  journal={ACM Transactions on Intelligent Systems and Technology (TIST)},
  volume={13},
  number={2},
  pages={1--41},
  year={2022},
  publisher={ACM New York, NY}
}

@inproceedings{he:etal:16,
  title={Deep residual learning for image recognition},
  author={He, Kaiming and Zhang, Xiangyu and Ren, Shaoqing and Sun, Jian},
  booktitle={Proceedings of the IEEE conference on computer vision and pattern recognition},
  pages={770--778},
  year={2016}
}

@inproceedings{ronn:etal:15,
  title={U-net: Convolutional networks for biomedical image segmentation},
  author={Ronneberger, Olaf and Fischer, Philipp and Brox, Thomas},
  booktitle={Medical image computing and computer-assisted intervention},
  pages={234--241},
  year={2015},
  organization={Springer}
}

@article{yu:18,
  title={On Conditional Correlations},
  author={Yu, Lei},
  journal={arXiv preprint arXiv:1811.03918},
  year={2018}
}

@article{jang:etal:16,
  title={Categorical reparameterization with gumbel-softmax},
  author={Jang, Eric and Gu, Shixiang and Poole, Ben},
  journal={arXiv preprint arXiv:1611.01144},
  year={2016}
}

@article{kami:cald:12,
  title={Data preprocessing techniques for classification without discrimination},
  author={Kamiran, Faisal and Calders, Toon},
  journal={Knowledge and information systems},
  volume={33},
  number={1},
  pages={1--33},
  year={2012},
  publisher={Springer}
}

@article{guer:etal:21,
  title={The hypervolume indicator: Computational problems and algorithms},
  author={Guerreiro, Andreia P and Fonseca, Carlos M and Paquete, Lu{\'\i}s},
  journal={ACM Computing Surveys (CSUR)},
  volume={54},
  number={6},
  pages={1--42},
  year={2021},
  publisher={ACM New York, NY, USA}
}

@article{zitz:thie:02,
  title={Multiobjective evolutionary algorithms: a comparative case study and the strength Pareto approach},
  author={Zitzler, Eckart and Thiele, Lothar},
  journal={IEEE transactions on Evolutionary Computation},
  volume={3},
  number={4},
  pages={257--271},
  year={2002},
  publisher={IEEE}
}

\ifincludesupplementary
    \renewcommand{\thesection}{S\arabic{section}}
    \renewcommand{\theequation}{S\arabic{equation}}
    \renewcommand{\thefigure}{S\arabic{figure}}
    \renewcommand{\thetable}{S\arabic{table}}

    
\clearpage
\onecolumn 

\setlength{\parindent}{0pt} 
\setlength{\parskip}{1em}   

\begin{center}
\Large
    \textbf{\large Supplementary Material of ``Task-tailored Pre-processing: Fair Supervised Learning"} 
\end{center}

\section{Further Discussion}

\subsection{Triangle inequality of the HGR correlation}
\label{supp:triangle}
\begin{lemma}
\label{lem:triangle}
Let $d(S_1,S_2) = \sqrt{2-2\rho(S_1,S_2)}$. Suppose $S_1$ and $S_2$ are regular and have the finite mean square contingency. Then  $d(S_1,S_2)\leq d(S_1,S_3)+d(S_2,S_3)$ holds if the pivot variable $S_3$ is a binary random variable.
\end{lemma}

$S_1$ and $S_2$ are called regular if the joint distribution of $S_1$ and $S_2$ is absolutely continuous w.r.t. the product distribution of them. The mean square contingency is defined as the square roof of $\int (d\mu_{S_1,S_2}(s_1,s_2) / d\mu_{S_1\otimes S_2}(s_1,s_2)-1)^2 d \mu_{S_1\otimes S_2}(s_1,s_2)$ where $\mu_{S_1,S_2}$ and $\mu_{S_1\otimes S_2}$ are the joint and the product distribution respectively. The condition for the pivot variable could be restrictive in the analysis. It remains nontrivial to extend the result for a generic pivot quantity. Throughout this work, we assume any paired variables in $d$ are regular and have a finite mean square contingency.

\subsection{Extension to separation}
\label{supp:sep}

To begin, we define the conditional HGR correlation as follows. 
\begin{definition}[Conditional HGR correlation]
\label{def:cond_hgr}
For random vectors $\bX$, $\bA$, and a random variable $Y$ the conditional HGR correlation is defined as $\rho_{Y}(\bX,\bA) = \sup_{f,g} \bE[f(\bX,Y)g(\bA,Y)]$ where the supremum is taken over all measurable functions $f$ and $g$ satisfying $\bE[f(\bX,Y)]=\bE[g(\bX,Y)]=0$ and $\bE[f(\bX,Y)^2]=\bE[g(\bA,Y)^2]=1$. 
\end{definition}
As \cite{yu:18} shows, $\rho_{Y}(\bX,\bA) = \sup_{y;P(y)>0} \rho_{Y=y}(\bX,\bA)$ if $Y$ is discrete, where $\rho_{Y=y}(\bX,\bA)=\sup_{f,g} \bE[f(\bX,y)g(\bA,y)]$ in the definition. Clearly, if $\rho_{Y=y}(\tih^*(\btiX),\bA)=0$ for every slice of $Y=y$, then $\tih^*(\btiX)$ is conditionally independent to $\bA$ given $Y$. Thus we replace $\rho$ by $\rho_Y$ in optimization \eqref{opt:min}, to pursue better separation. When applying this conditional HGR correlation, the conditional variable is $Y$, not the modified $\tiY$ since our inferential interest is for the unseen true label of test data. Upon definition, Proposition~\ref{prop:motiv} can be restated under $\rho_Y$, and we define a similar  distance-like metric $d_{Y=y}(S_1,S_2)=\sqrt{2-2\rho_{Y=y}(S_1,S_2)}$ for generic random quantities $S_1,S_2$ with the triangle inequality. Then the next corollary reformulates the upper bound of Theorem~\ref{thm:uplwh0} and Corollary~\ref{cor:uplwh0} for $\tih^*$ under the separation penalty $\rho_Y$.
\begin{corollary}
\label{cor:sep_dominating}
    For any $\delta_{\bX},\delta_Y,\lambda_F\geq 0$, if the output of $h$ is binary, it follows that $\rho_Y(h^*(\bX),\bA) -  \rho_Y(\tih^*(\btiX),\bA) \leq \sup_{y} d_{Y=y}(\tih^*(\btiX),\tiY) + d_{Y=y}(h^*(\bX),\tiY)$. For $\delta_Y=0$, the lower bound for separation is the same as the lower bound in Theorem~\ref{thm:uplwh0}. Similarly, the lower bound for $\delta_Y \geq 0$ resembles the one in Corollary~\ref{cor:uplwh0} by expressing it by $d_{Y=y}$.
\end{corollary}
Hence, the overall argument describing the utility-fairness relationship of $\tih^*$ and in the downstream model $\tih_k^*$ under independence is extended to cover the notion of separation.

\subsection{Dual approximation for separation}
\label{supp:sep-dual}

Similarly, the dual approximation can be found for the conditional HGR from the observation $\rho_{Y=y}(\barh^*(\bbarX),\bA)^2\leq \chi^2(P_{\barh^*(\bbarX),\bA|Y=y},P_{\barh^*(\bbarX)|Y=y}\otimes P_{\bA|Y=y})-1$ for all $y$. Therefore, $||\chi^2(P_{\barh^*(\bbarX),\bA|\cdot},P_{\barh^*(\bbarX)|\cdot}\times P_{\bA|\cdot})||_{\infty,{\cal Y}}$ measures the extent of separation where $||\cdot||_{\infty}$ is the uniform norm on $\calY$. However, for consistency with the independence notion, we borrow the idea of \cite{roma:etal:20} that $\barh^*(\bbarX)\perp \bA|Y$ if and only if $(\barh^*(\bbarX),\bA,Y) \overset{d}{=} (\barh^*(\bbarX),\bA',Y)$ where $\bA'$ independently follows $A|Y$. As a result, the degree of separation is characterized by  $\chi^2(P_{\barh^*(\bbarX),\bA,Y},P_{\barh^*(\bbarX),\bA',Y})$, which is equivalent to $\bE_Y[\chi^2(P_{\barh^*(\bbarX),\bA|Y},P_{\barh^*(\bbarX)|Y}\otimes P_{\bA|Y})]$. For these reasons, we define $R_V(\barh^*(\bbarX),\bA,Y)=\bE_{P_{\barh^*(\bbarX),\bA,Y}}[V(\barh^*(\bbarX),\bA,Y)] - \bE_{P_{\barh^*(\bbarX),Y}\otimes P_{\bA|Y}}[f^*(V(\barh^*(\bbarX),\bA,Y))]$ for separation.

\subsection{Further simulation results}

\subsubsection{Prediction and fairness scores evaluated on {\bf D}}
\label{supp:tab:score_output} 

Table~\ref{tab:basic_score} reports the scores of MLP and RF only because of the limited page. Refer to Table~\ref{tab:performance-metrics-all} to see all other results.  
\begin{table}[ht!]
\centering
\caption{Prediction/fairness scores of all models on ${\bf D}$}
\label{tab:performance-metrics-all}
\resizebox{\textwidth}{!}{%
\begin{tabular}{cccccccccccc}
\toprule
 & & \multicolumn{2}{c}{AUC} & \multicolumn{2}{c}{SP} & \multicolumn{2}{c}{KS-SP} & \multicolumn{2}{c}{EO} & \multicolumn{2}{c}{KS-EO} \\
\cmidrule(lr){3-4} \cmidrule(lr){5-6} \cmidrule(lr){7-8} \cmidrule(lr){9-10} \cmidrule(lr){11-12}
Dataset & Method & $A_{\text{bin}}$ & ${\bf A}_{\text{dis}}$ & $A_{\text{bin}}$ & ${\bf A}_{\text{dis}}$ & $A_{\text{bin}}$ & ${\bf A}_{\text{dis}}$ & $A_{\text{bin}}$& ${\bf A}_{\text{dis}}$ & $A_{\text{bin}}$ & ${\bf A}_{\text{dis}}$ \\
\midrule
\multirow{6}{*}{\texttt{A}} 
 & HGB & 0.927 (0.003) & 0.926 (0.003) & 0.846 (0.015) & 1.597 (0.033) & 0.347 (0.007) & 0.691 (0.008) & 0.345 (0.030) & 0.694 (0.099) & 0.382 (0.025) & 0.897 (0.072) \\
 & KNN & 0.861 (0.001) & 0.861 (0.002) & 0.901 (0.026) & 1.651 (0.027) & 0.310 (0.005) & 0.585 (0.021) & 0.269 (0.051) & 0.522 (0.092) & 0.352 (0.015) & 0.713 (0.054) \\
 & MLP & 0.879 (0.002) & 0.879 (0.006) & 0.832 (0.012) & 1.566 (0.042) & 0.349 (0.009) & 0.697 (0.011) & 0.446 (0.028) & 0.809 (0.104) & 0.429 (0.005) & 0.945 (0.026) \\
 & RF & 0.891 (0.001) & 0.891 (0.002) & 0.759 (0.040) & 1.451 (0.042) & 0.308 (0.007) & 0.610 (0.007) & 0.298 (0.025) & 0.614 (0.068) & 0.387 (0.024) & 0.808 (0.038) \\
 & SVC & 0.907 (0.003) & 0.907 (0.003) & 0.792 (0.031) & 1.550 (0.046) & 0.336 (0.008) & 0.659 (0.002) & 0.335 (0.018) & 0.684 (0.117) & 0.415 (0.025) & 0.884 (0.065) \\
 & Upstream & 0.908 (0.003) & 0.906 (0.003) & 0.848 (0.026) & 1.542 (0.028) & 0.365 (0.010) & 0.687 (0.018) & 0.374 (0.030) & 0.735 (0.101) & 0.418 (0.026) & 0.897 (0.044) \\
\midrule
\multirow{6}{*}{\texttt{E}} 
 & HGB & 0.889 (0.001) & 0.888 (0.001) & 0.704 (0.008) & 1.416 (0.009) & 0.392 (0.002) & 0.793 (0.004) & 0.582 (0.012) & 1.179 (0.011) & 0.749 (0.007) & 1.523 (0.018) \\
 & KNN & 0.849 (0.001) & 0.851 (0.001) & 0.672 (0.006) & 1.348 (0.017) & 0.334 (0.003) & 0.665 (0.007) & 0.550 (0.010) & 1.089 (0.022) & 0.438 (0.007) & 0.879 (0.020) \\
 & MLP & 0.878 (0.002) & 0.880 (0.001) & 0.654 (0.017) & 1.329 (0.033) & 0.367 (0.007) & 0.748 (0.022) & 0.532 (0.025) & 1.091 (0.058) & 0.521 (0.007) & 1.118 (0.052) \\
 & RF & 0.862 (0.001) & 0.865 (0.002) & 0.580 (0.010) & 1.228 (0.047) & 0.340 (0.003) & 0.698 (0.007) & 0.447 (0.011) & 0.963 (0.060) & 0.474 (0.014) & 0.976 (0.028) \\
 & SVC & 0.883 (0.002) & 0.884 (0.001) & 0.708 (0.008) & 1.451 (0.030) & 0.406 (0.004) & 0.823 (0.013) & 0.620 (0.012) & 1.283 (0.049) & 0.682 (0.033) & 1.453 (0.047) \\
 & Upstream & 0.886 (0.001) & 0.886 (0.001) & 0.700 (0.023) & 1.388 (0.030) & 0.387 (0.008) & 0.776 (0.015) & 0.580 (0.029) & 1.144 (0.040) & 0.599 (0.019) & 1.251 (0.043) \\
\midrule
\multirow{6}{*}{\texttt{P}} 
 & HGB & 0.760 (0.003) & 0.755 (0.003) & 1.147 (0.074) & 2.233 (0.183) & 0.550 (0.008) & 1.125 (0.017) & 1.062 (0.040) & 2.128 (0.081) & 0.959 (0.029) & 1.981 (0.055) \\
 & KNN & 0.702 (0.003) & 0.695 (0.003) & 2.331 (0.528) & 4.667 (1.072) & 0.413 (0.007) & 0.837 (0.018) & 1.486 (0.232) & 3.083 (0.534) & 0.650 (0.028) & 1.352 (0.051) \\
 & MLP & 0.725 (0.002) & 0.721 (0.005) & 1.073 (0.175) & 1.709 (0.108) & 0.483 (0.013) & 0.923 (0.046) & 0.918 (0.125) & 1.496 (0.100) & 0.784 (0.029) & 1.522 (0.093) \\
 & RF & 0.727 (0.002) & 0.717 (0.002) & 0.937 (0.022) & 1.964 (0.208) & 0.434 (0.012) & 0.877 (0.017) & 0.870 (0.035) & 1.782 (0.115) & 0.702 (0.033) & 1.438 (0.048) \\
 & SVC & 0.749 (0.003) & 0.745 (0.003) & 1.201 (0.071) & 2.437 (0.217) & 0.545 (0.009) & 1.125 (0.022) & 1.158 (0.038) & 2.409 (0.123) & 0.970 (0.032) & 2.057 (0.051) \\
 & Upstream & 0.748 (0.002) & 0.737 (0.007) & 1.050 (0.098) & 2.136 (0.233) & 0.498 (0.024) & 1.044 (0.068) & 0.932 (0.069) & 1.982 (0.235) & 0.832 (0.055) & 1.815 (0.148) \\
\bottomrule
\end{tabular}%
}
\end{table}

\subsubsection{Trade-off figures}
\label{supp:trade-off-curves} 

This section presents trade-off figures that are omitted in the main manuscript. See Figures~\ref{fig:sp-ks}, \ref{fig:eo}, \ref{fig:eo-ks}, \ref{fig:sp-ratio-mul}, \ref{fig:sp-ks-mul}, \ref{fig:eo-ratio-mul}, and \ref{fig:eo-ks-mul}. 
\begin{figure*}[h] 
    \centering
    \subfloat[Trade-off plots of Adult  with $\delta_{X}\in \{0.01,0.10,0.30\}$, $\delta_Y=0$, and $\lambda_F = 10$]{%
        \includegraphics[width=\linewidth, height=0.15\textheight, keepaspectratio]{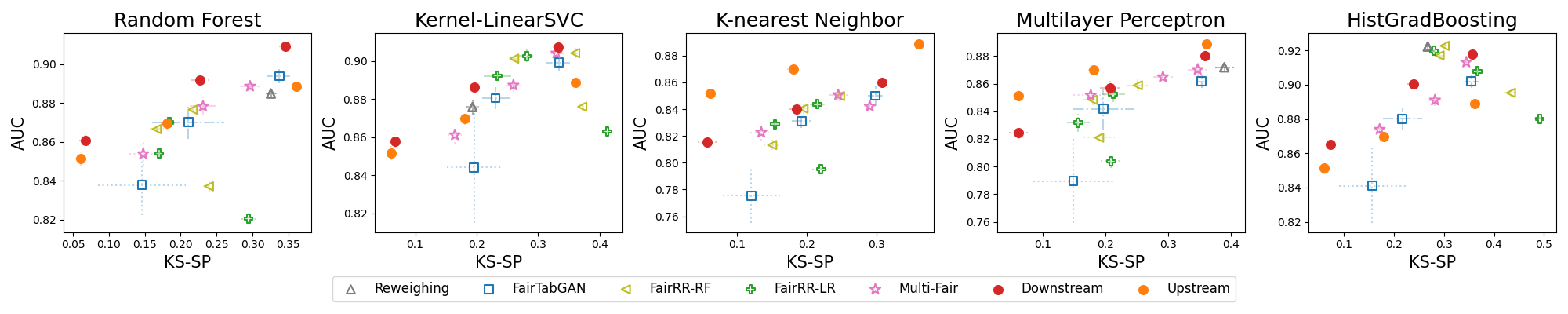}
    }\par
    \subfloat[Trade-off plots of ACSEmployment with $\delta_{X}\in \{0.01,0.05,0.10\}$,  $\delta_Y=0$, and $\lambda_F = 10$]{%
        \includegraphics[width=\linewidth, height=0.15\textheight, keepaspectratio]{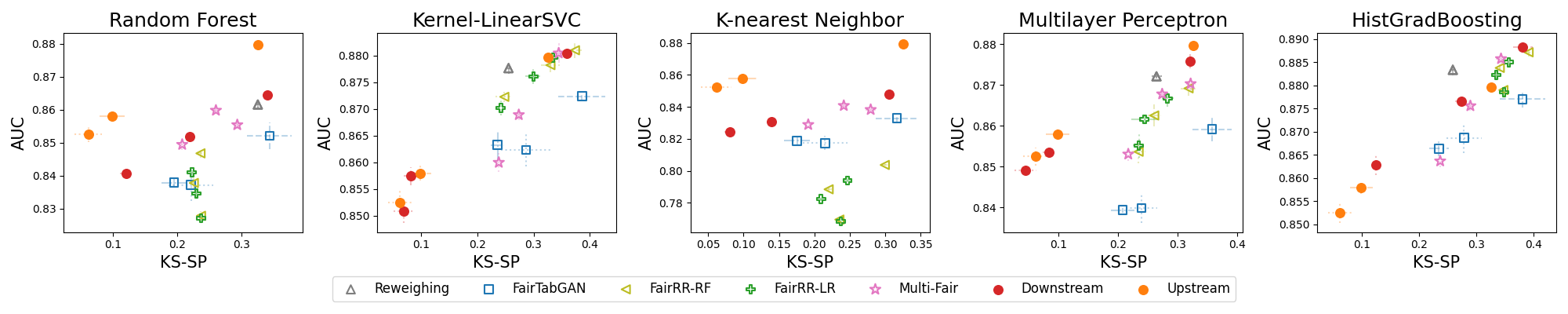}
    }\par
    \subfloat[Trade-off plots of ACSPublicCoverage with $\delta_{X}\in \{0.01,0.05,0.10\}$, $\delta_Y=0.1$, and $\lambda_F = 10$]{%
        \includegraphics[width=\linewidth, height=0.15\textheight, keepaspectratio]{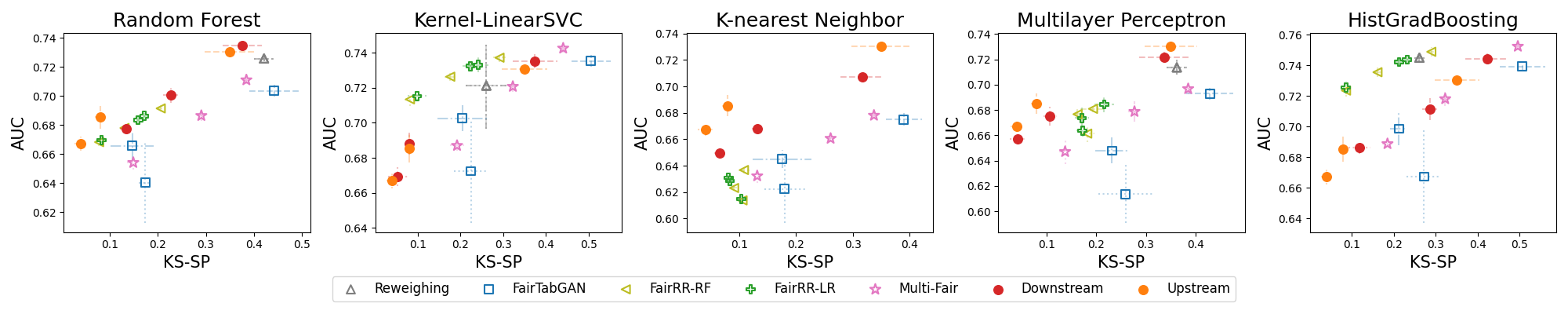}
    }\par
    \caption{Trade-off curves between AUC and KS-SP in handling $A_{\text{bin}}$: Refer to Figure~\ref{fig:sp} for details.}
    \label{fig:sp-ks}
\end{figure*}
\clearpage
\vspace*{\fill}
\begin{figure*}[h] 
    \centering
    \subfloat[Trade-off plots of Adult  with $\delta_{X}\in \{0.01,0.10,0.30\}$, $\delta_Y=0$, and $\lambda_F = 10$]{%
        \includegraphics[width=\linewidth, height=0.15\textheight, keepaspectratio]{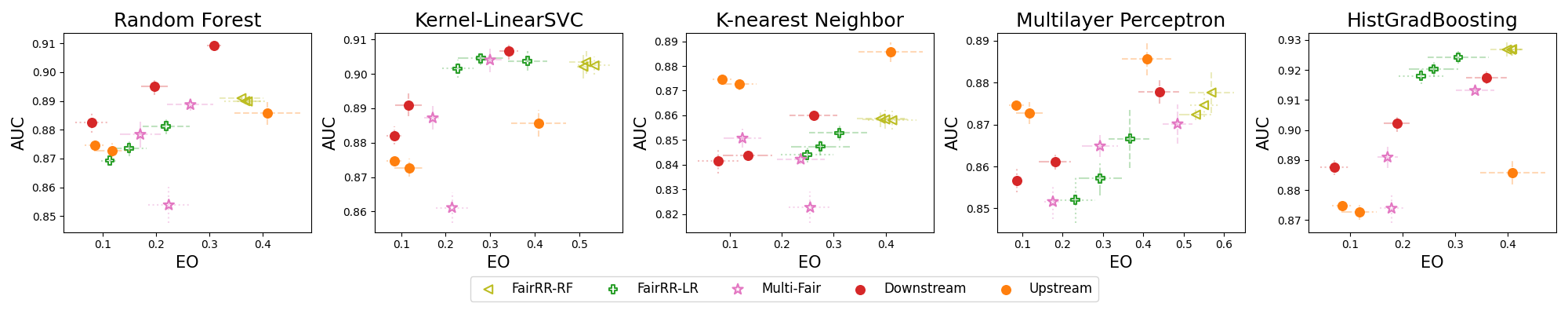}
    }\par
    \subfloat[Trade-off plots of ACSEmployment with $\delta_{X}\in \{0.01,0.025,0.05\}$,  $\delta_Y=0$, and $\lambda_F = 100$]{%
        \includegraphics[width=\linewidth, height=0.15\textheight, keepaspectratio]{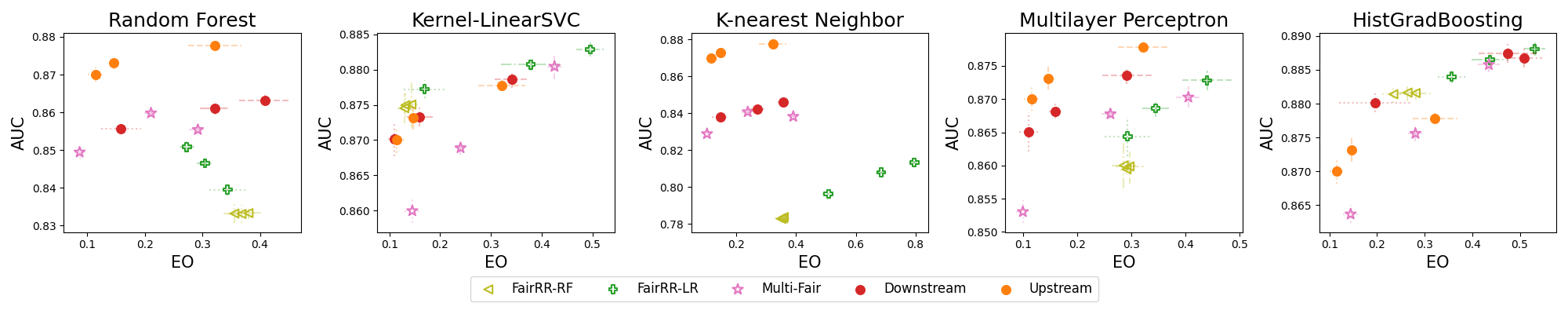}
    }\par
    \subfloat[Trade-off plots of ACSPublicCoverage with $\delta_{X}\in \{0.01,0.05,0.10\}$, $\delta_Y=0.05$, and $\lambda_F = 100$]{%
        \includegraphics[width=\linewidth, height=0.15\textheight, keepaspectratio]{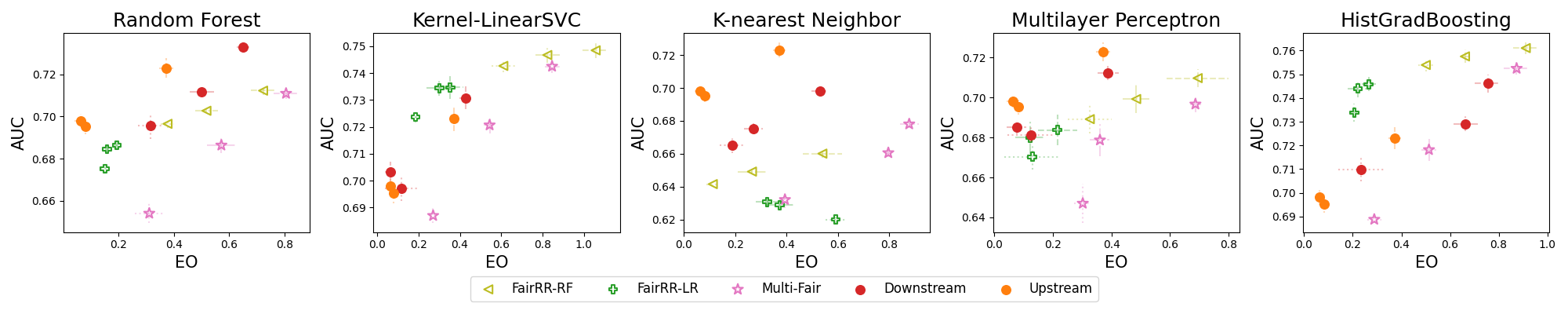}
    }\par
    \caption{Trade-off curves between AUC and EO in handling $A_{\text{bin}}$: Refer to Figure~\ref{fig:sp} for details.}
    \label{fig:eo}
\end{figure*}
\vspace*{\fill}
\clearpage
\vspace*{\fill}
\begin{figure*}[h] 
    \centering
    \subfloat[Trade-off plots of Adult  with $\delta_{X}\in \{0.01,0.10,0.30\}$, $\delta_Y=0$, and $\lambda_F = 10$]{%
        \includegraphics[width=\linewidth, height=0.15\textheight, keepaspectratio]{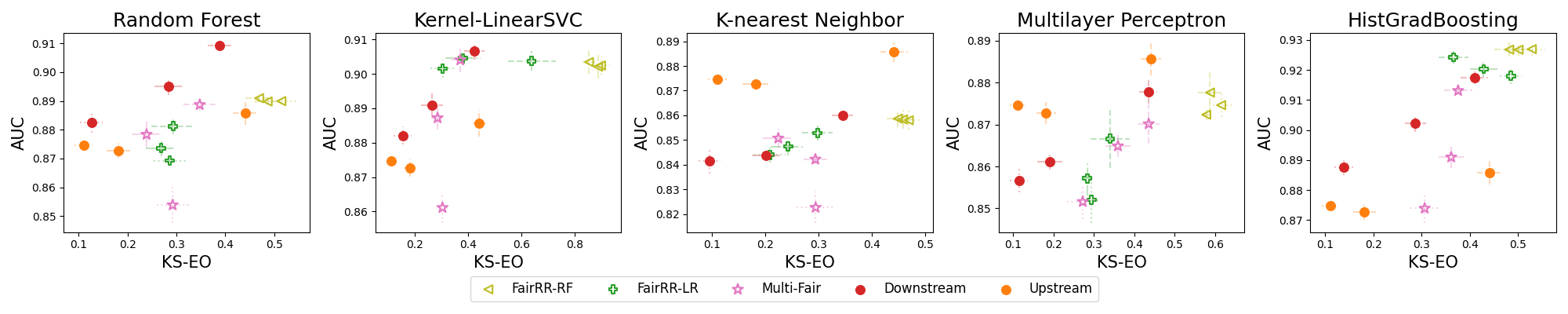}
    }\par
    \subfloat[Trade-off plots of ACSEmployment with $\delta_{X}\in \{0.01,0.025,0.05\}$,  $\delta_Y=0$, and $\lambda_F = 100$]{%
        \includegraphics[width=\linewidth, height=0.15\textheight, keepaspectratio]{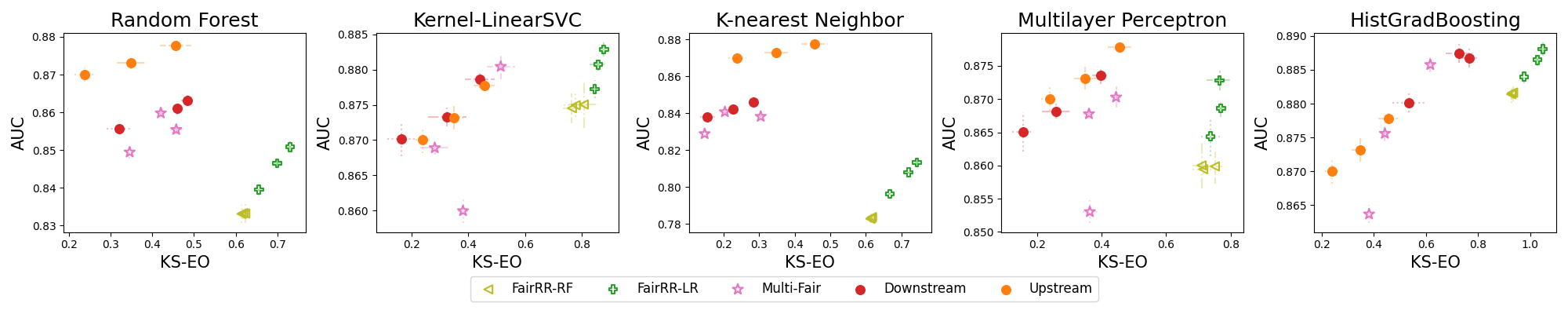}
    }\par
    \subfloat[Trade-off plots of ACSPublicCoverage with $\delta_{X}\in \{0.01,0.05,0.10\}$, $\delta_Y=0.05$, and $\lambda_F = 100$]{%
        \includegraphics[width=\linewidth, height=0.15\textheight, keepaspectratio]{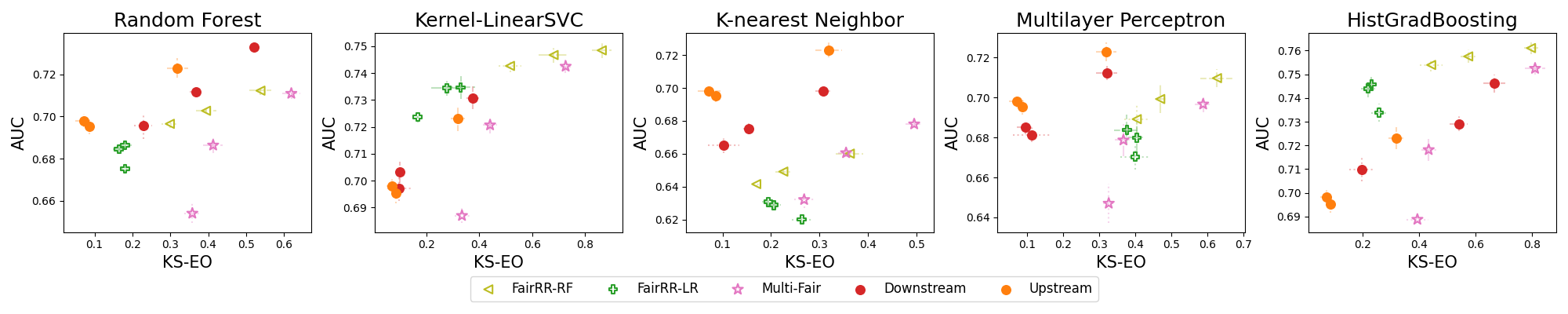}
    }\par
    \caption{Trade-off curves between AUC and KS-EO in handling $A_{\text{bin}}$: Refer to Figure~\ref{fig:sp} for details.}
    \label{fig:eo-ks}
\end{figure*}
\vspace*{\fill}
\clearpage
\vspace*{\fill}
\begin{figure*}[h] 
    \centering
    \subfloat[Trade-off plots of Adult  with $\delta_{X}\in \{0.01,0.10,0.30\}$, $\delta_Y=0$, and $\lambda_F = 10$]{%
        \includegraphics[width=\linewidth]{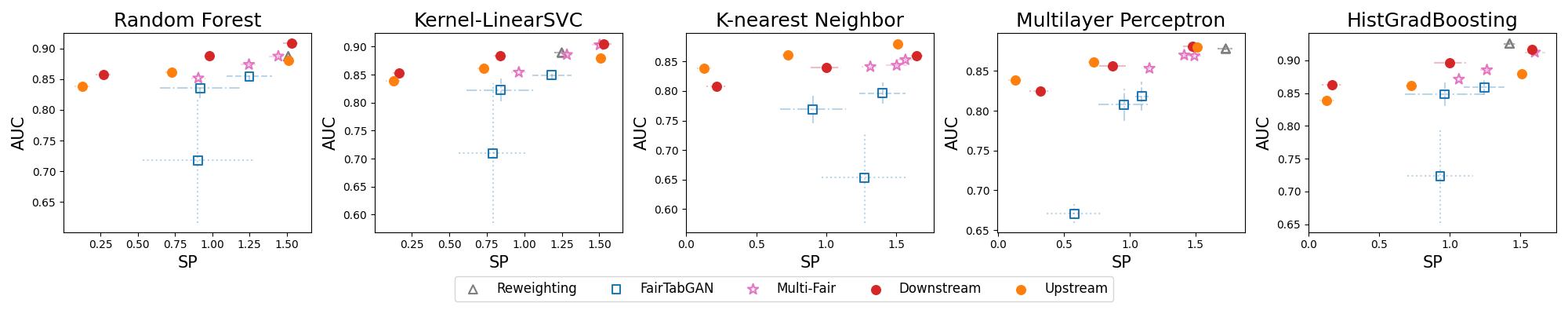}
    }\par
    \subfloat[Trade-off plots of ACSEmployment with $\delta_{X}\in \{0.01,0.05,0.10\}$,  $\delta_Y=0$, and $\lambda_F = 10$]{%
        \includegraphics[width=\linewidth]{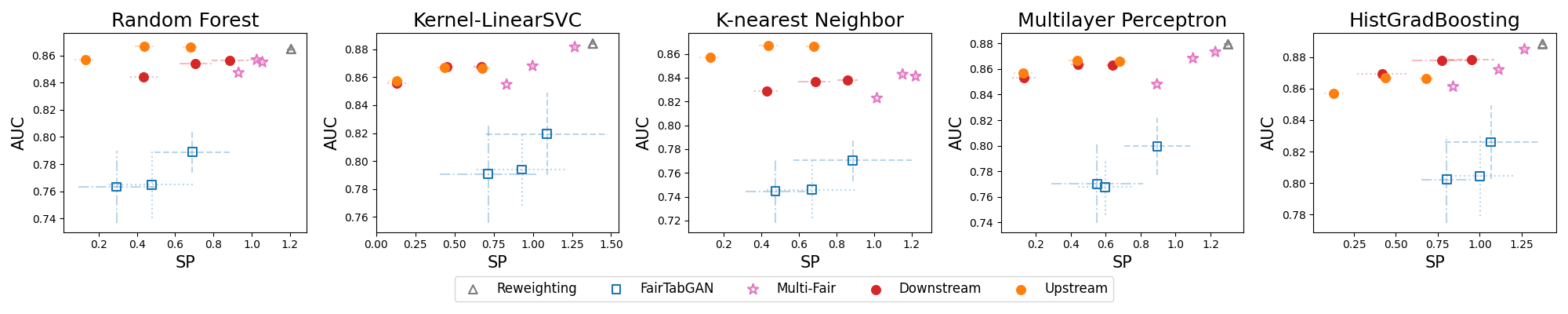}
    }\par
    \subfloat[Trade-off plots of ACSPublicCoverage with $\delta_{X}\in \{0.01,0.05,0.10\}$, $\delta_Y=0.1$, and $\lambda_F = 10$]{%
        \includegraphics[width=\linewidth]{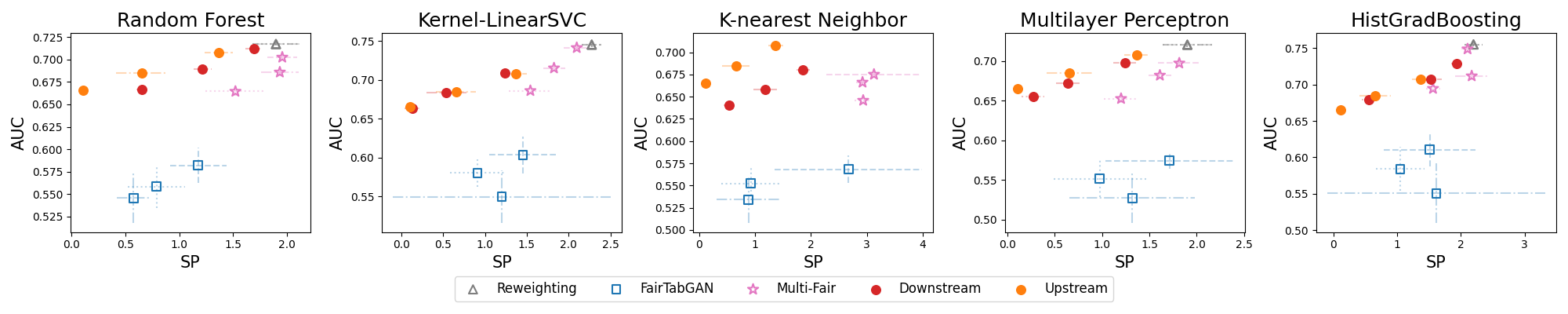}
    }\par
    \caption{Trade-off curves between AUC and SP in handling ${\bf A}_{\text{dis}}$: Refer to Figure~\ref{fig:sp} for details.}
    \label{fig:sp-ratio-mul}
\end{figure*}
\vspace*{\fill}
\clearpage
\vspace*{\fill}
\begin{figure*}[h] 
    \centering
    \subfloat[Trade-off plots of Adult  with $\delta_{X}\in \{0.01,0.10,0.30\}$, $\delta_Y=0$, and $\lambda_F = 10$]{%
        \includegraphics[width=\linewidth]{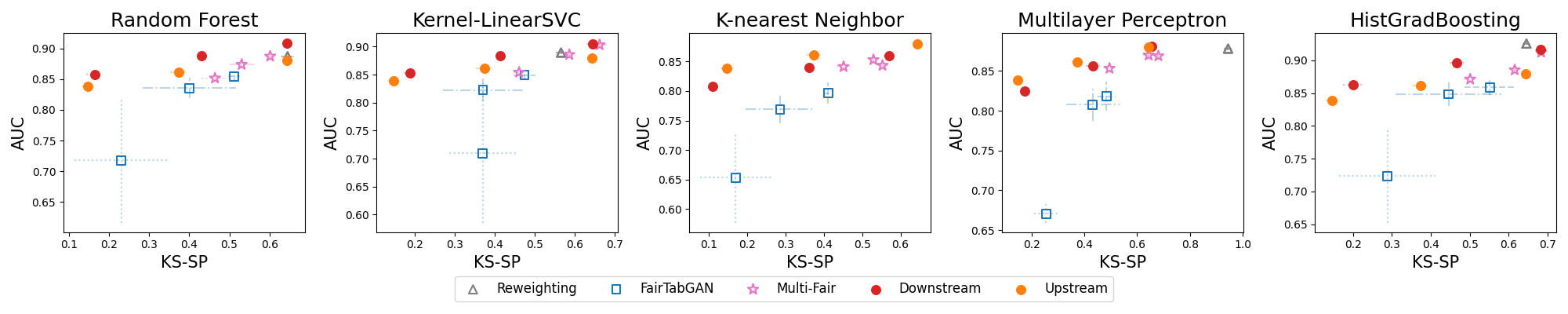}
    }\par
    \subfloat[Trade-off plots of ACSEmployment with $\delta_{X}\in \{0.01,0.05,0.10\}$,  $\delta_Y=0$, and $\lambda_F = 10$]{%
        \includegraphics[width=\linewidth]{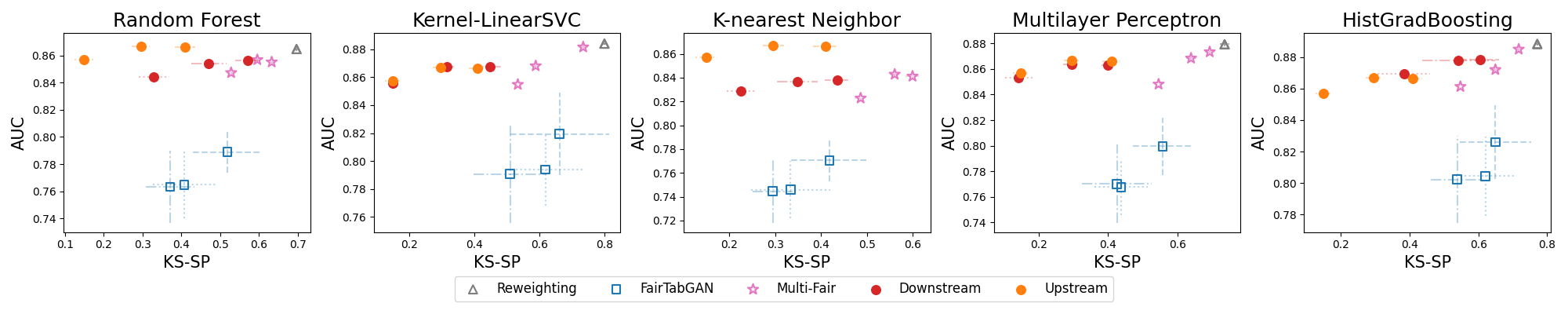}
    }\par
    \subfloat[Trade-off plots of ACSPublicCoverage with $\delta_{X}\in \{0.01,0.05,0.10\}$, $\delta_Y=0.1$, and $\lambda_F = 10$]{%
        \includegraphics[width=\linewidth]{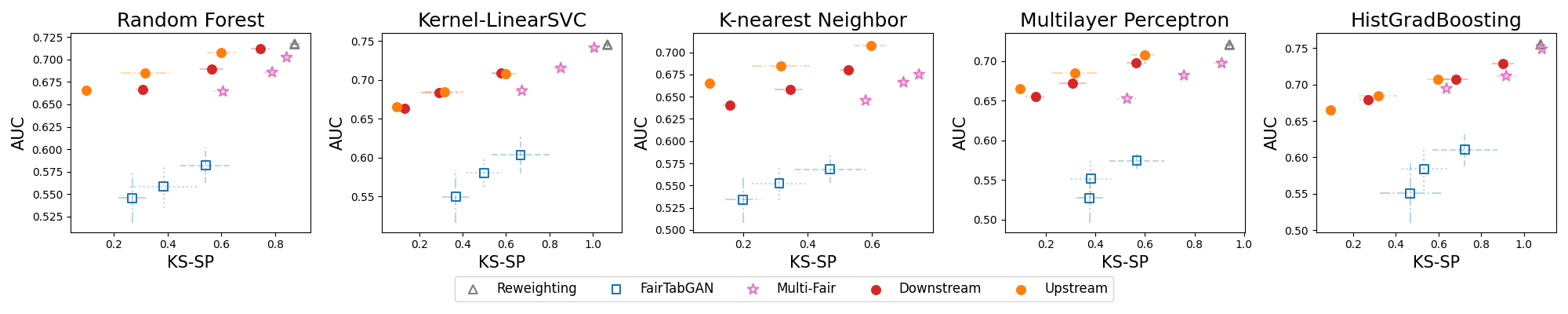}
    }\par
    \caption{Trade-off curves between AUC and KS-SP in handling ${\bf A}_{\text{dis}}$: Refer to Figure~\ref{fig:sp} for details.}
    \label{fig:sp-ks-mul}
\end{figure*}
\vspace*{\fill}
\clearpage
\vspace*{\fill}
\begin{figure*}[h] 
    \centering
    \subfloat[Trade-off plots of Adult  with $\delta_{X}\in \{0.01,0.10,0.30\}$, $\delta_Y=0$, and $\lambda_F = 10$]{%
        \includegraphics[width=\linewidth]{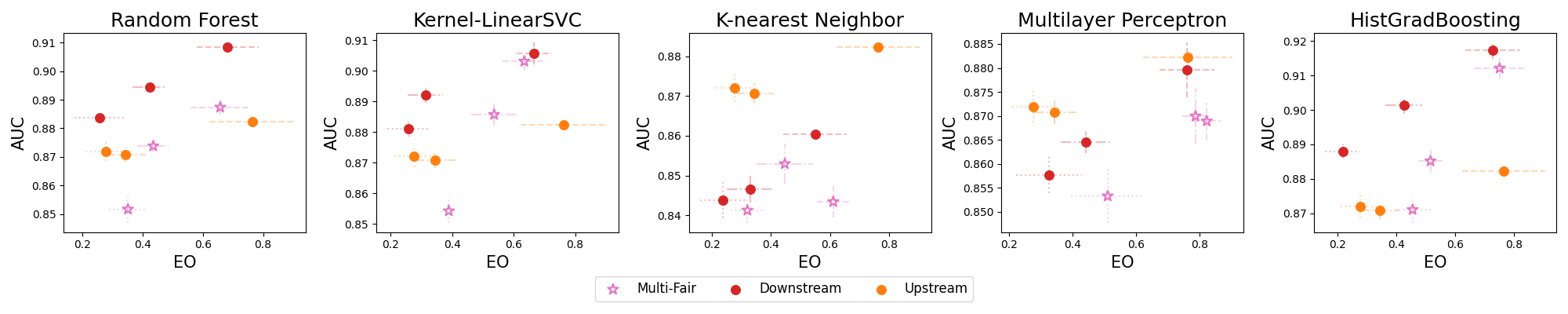}
    }\par
    \subfloat[Trade-off plots of ACSEmployment with $\delta_{X}\in \{0.01,0.025,0.05\}$,  $\delta_Y=0$, and $\lambda_F = 1$]{%
        \includegraphics[width=\linewidth]{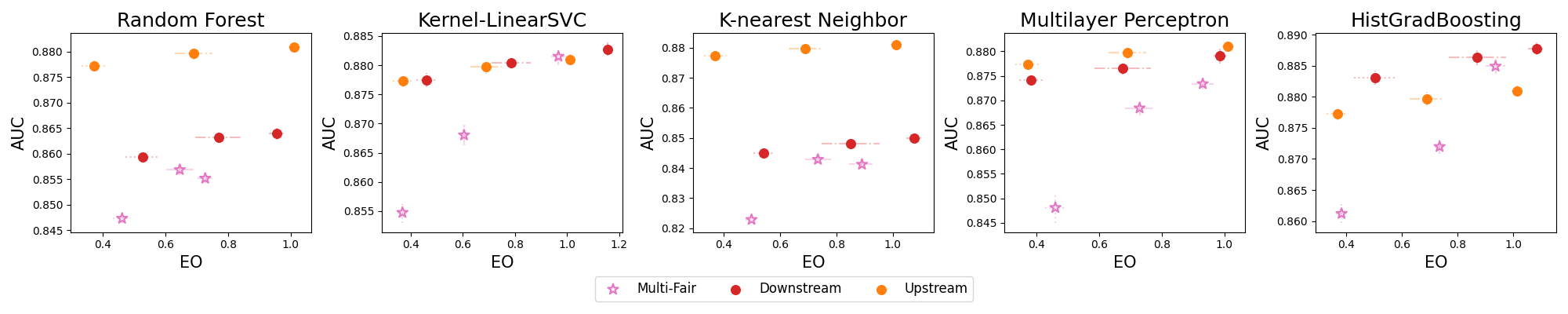}
    }\par
    \subfloat[Trade-off plots of ACSPublicCoverage with $\delta_{X}\in \{0.01,0.05,0.10\}$, $\delta_Y=0.05$, and $\lambda_F = 100$]{%
        \includegraphics[width=\linewidth]{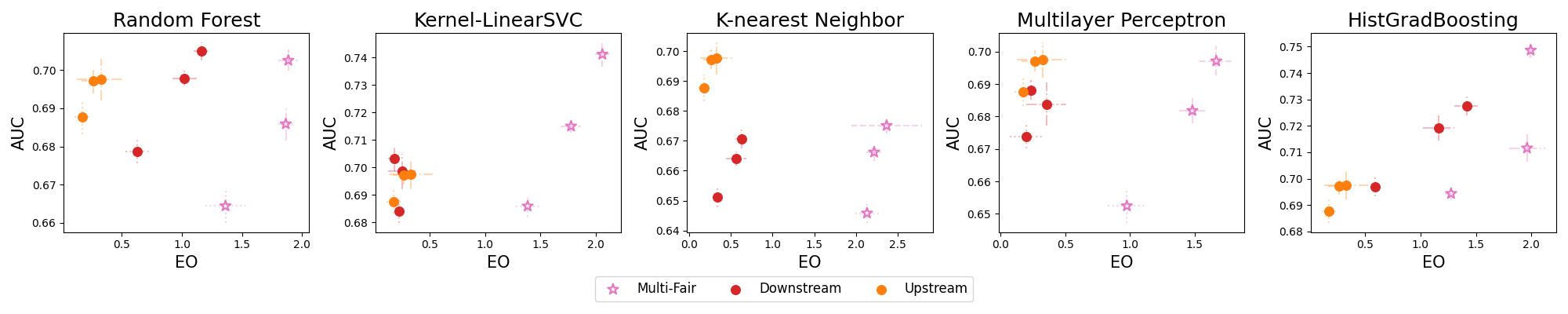}
    }\par
    \caption{Trade-off curves between AUC and EO in handling ${\bf A}_{\text{dis}}$: Refer to Figure~\ref{fig:sp} for details.}
    \label{fig:eo-ratio-mul}
\end{figure*}
\vspace*{\fill}
\clearpage
\vspace*{\fill}
\begin{figure*}[h] 
    \centering
    \subfloat[Trade-off plots of Adult  with $\delta_{X}\in \{0.01,0.10,0.30\}$, $\delta_Y=0$, and $\lambda_F = 10$]{%
        \includegraphics[width=\linewidth]{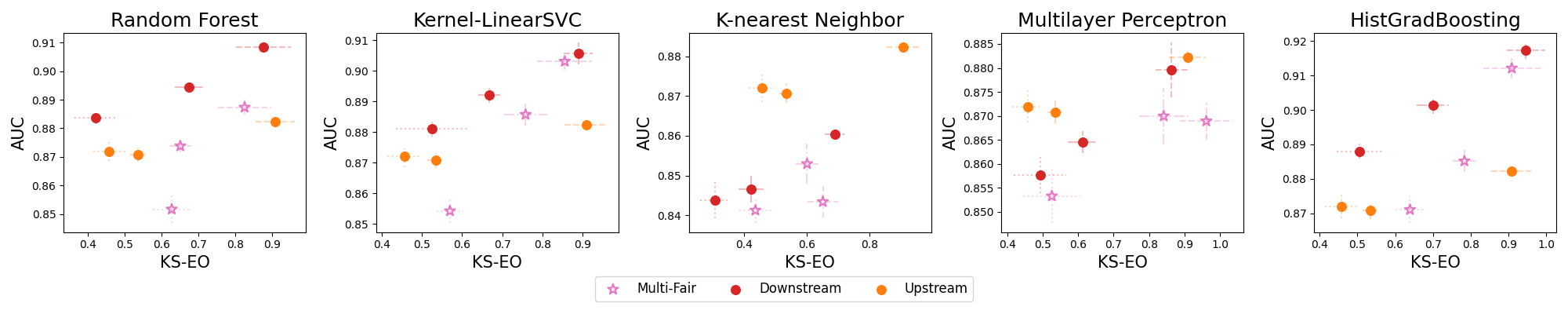}
    }\par
    \subfloat[Trade-off plots of ACSEmployment with $\delta_{X}\in \{0.01,0.025,0.05\}$,  $\delta_Y=0$, and $\lambda_F = 1$]{%
        \includegraphics[width=\linewidth]{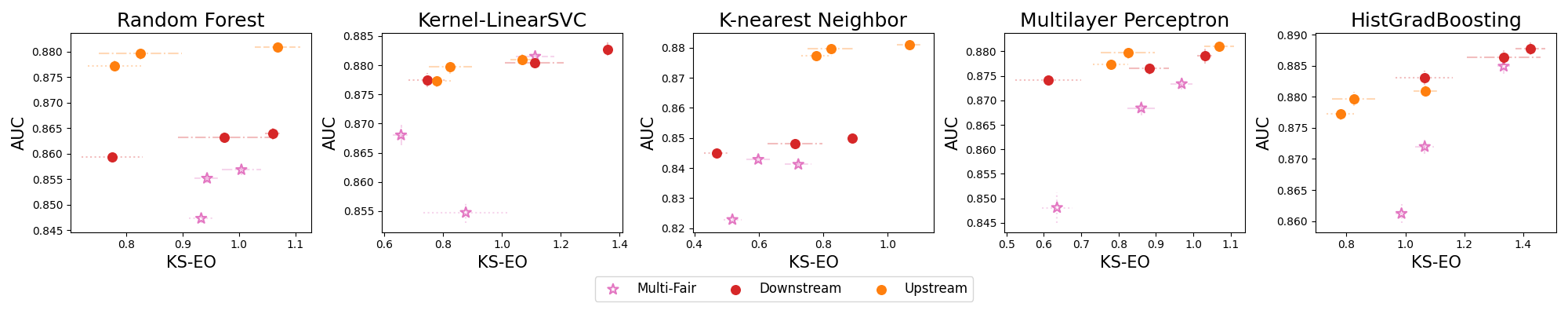}
    }\par
    \subfloat[Trade-off plots of ACSPublicCoverage with $\delta_{X}\in \{0.01,0.05,0.10\}$, $\delta_Y=0.05$, and $\lambda_F = 100$]{%
        \includegraphics[width=\linewidth]{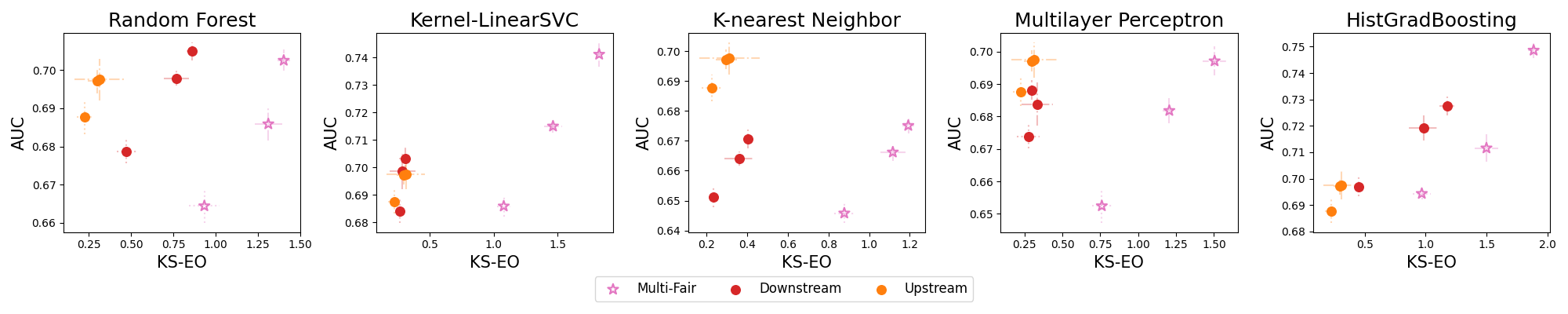}
    }\par
    \caption{Trade-off curves between AUC and KS-EO in handling ${\bf A}_{\text{dis}}$: Refer to Figure~\ref{fig:sp} for details.}
    \label{fig:eo-ks-mul}
\end{figure*}
\vspace*{\fill}
\newpage

\subsubsection{Hypervolume scores for the KS metric} This section presents the barplot for the hypervolume scores when Pareto frontiers are estimated based on AUC and KS-type fairness scores. See Figure~\ref{fig:hv-indicator-ks}. 
\begin{figure*}[h]
    \centering
    \includegraphics[width=1.0\linewidth]{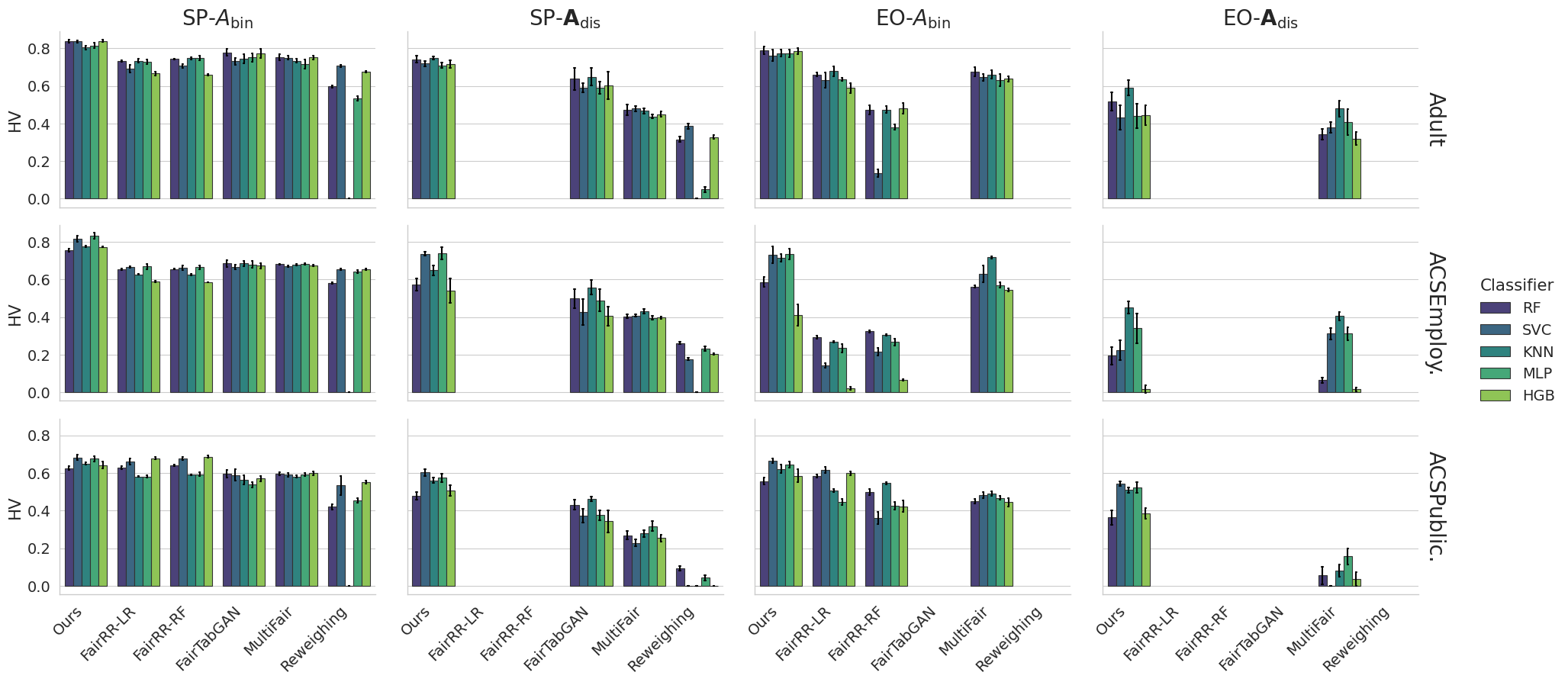}
    \caption{Comparison of HV scores from (AUC, KS-SP/EO): Refer to Figure~\ref{fig:hv-indicator} for details.}
    \label{fig:hv-indicator-ks}
\end{figure*}

\subsection{Further discussion about consistency scores}
\label{supp:const-scores}

Figure~\ref{fig:four_figures_subfloat} displays the consistency scores for all combinations of fairness metrics and sensitive attributes. While ours generally shows slightly better consistency compared to FairRR, some peculiar points in ours are far worse. For instance, the KS-based scores of Budget 2 in Figures~\ref{fig:consistency-sp-bin} and \ref{fig:consistency-eo-bin} in ACSEmployment and ACSPublicCoverage are substantially higher than ones in Budgets 1 and 3. From our observation, the inclusion of the non-Lipschitz boosting model (HGB) brings about substantially higher consistency scores. Let's recall that our downstream analysis, e.g., Theorem~\ref{thm:uplwhk} and Proposition~\ref{prop:coherent}, requires that the downstream models be Lipschitz so that variations among the downstream models can be controlled through the upstream model. As shown in  Figure~\ref{fig:consistency-comparison-hgb}, the consistency scores tend to be improved when HGB is omitted. For instance, the KS-based scores of Budget 2 in ACSEmployment and ACSPublicCoverage are significantly reduced. This short simulation study emphasizes that the Lipschitzness of our theoretical argument potentially plays a crucial role in bounding downstream performance.  
\begin{figure}[h]
    \centering
    
    \subfloat[Statistical parity for $A_{\text{bin}}$. \label{fig:consistency-sp-bin}]{%
        \includegraphics[width=0.48\linewidth]{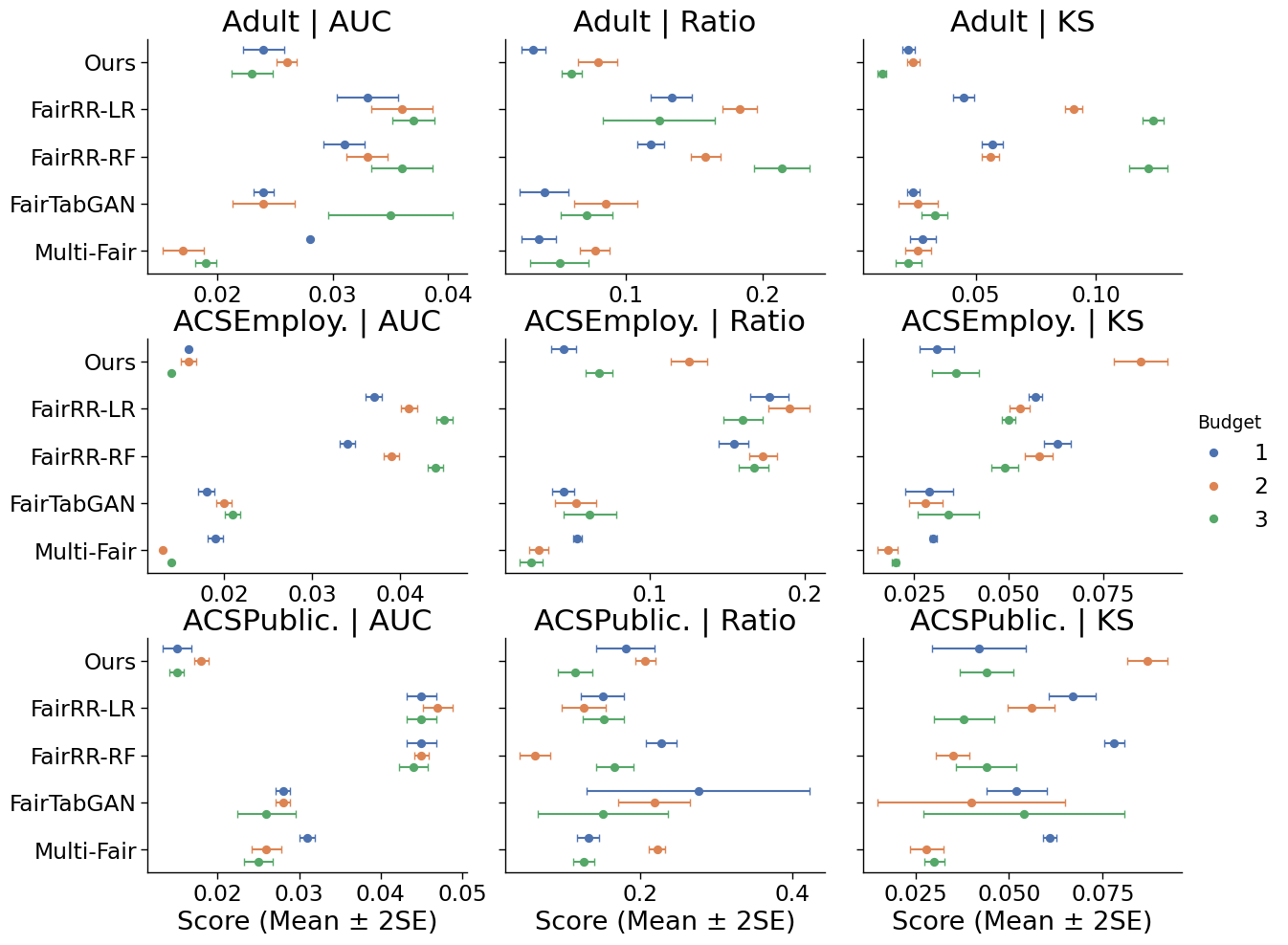}%
    }
    \hfill
    \subfloat[Statistical parity for ${\bf A}_{\text{dis}}$\label{fig:consistency-sp-mul}]{%
        \includegraphics[width=0.48\linewidth]{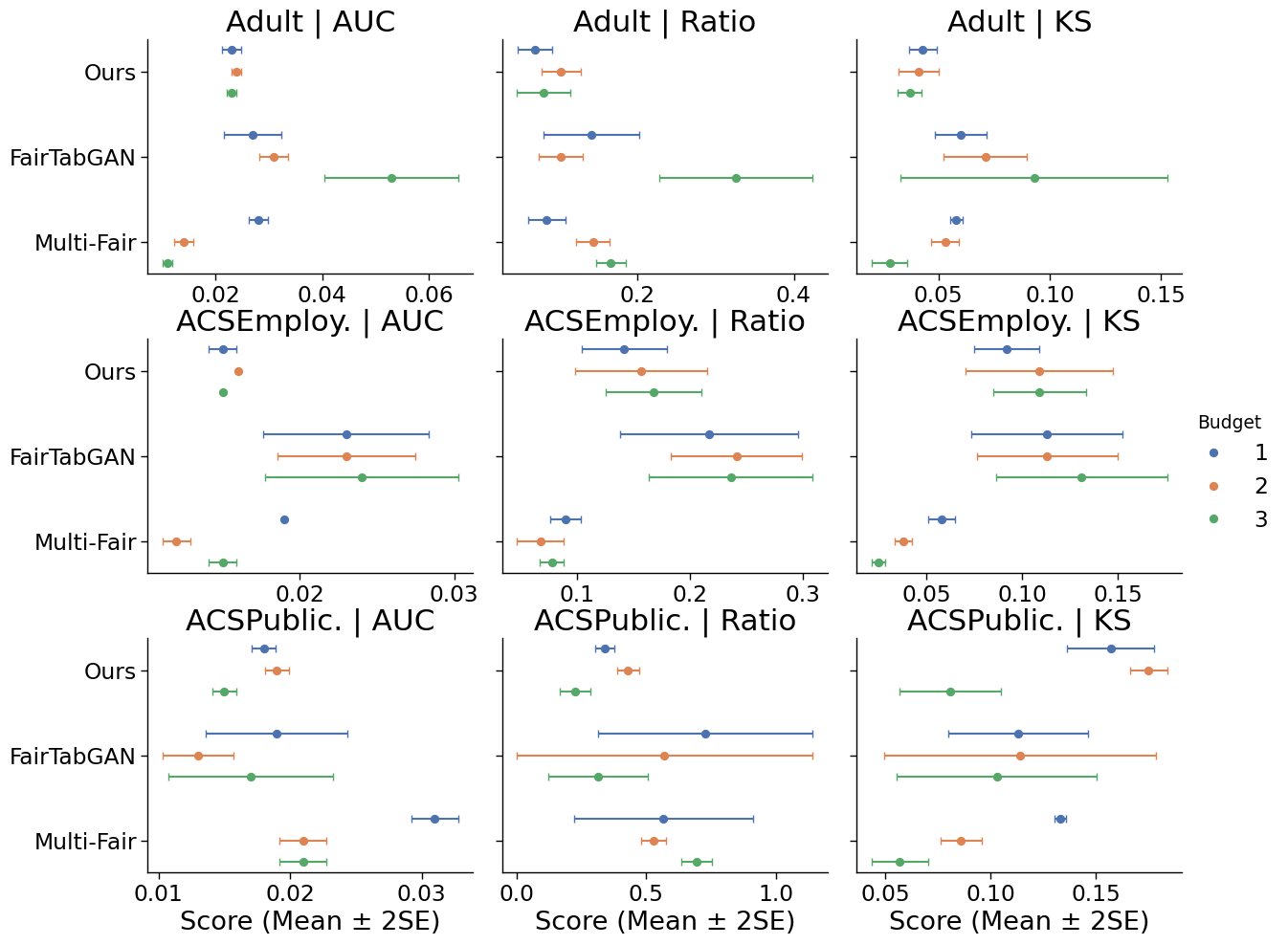}%
    }
    
    \vspace{1em} 
    
    \subfloat[Equalized odds for $A_{\text{bin}}$\label{fig:consistency-eo-bin}]{%
        \includegraphics[width=0.48\linewidth]{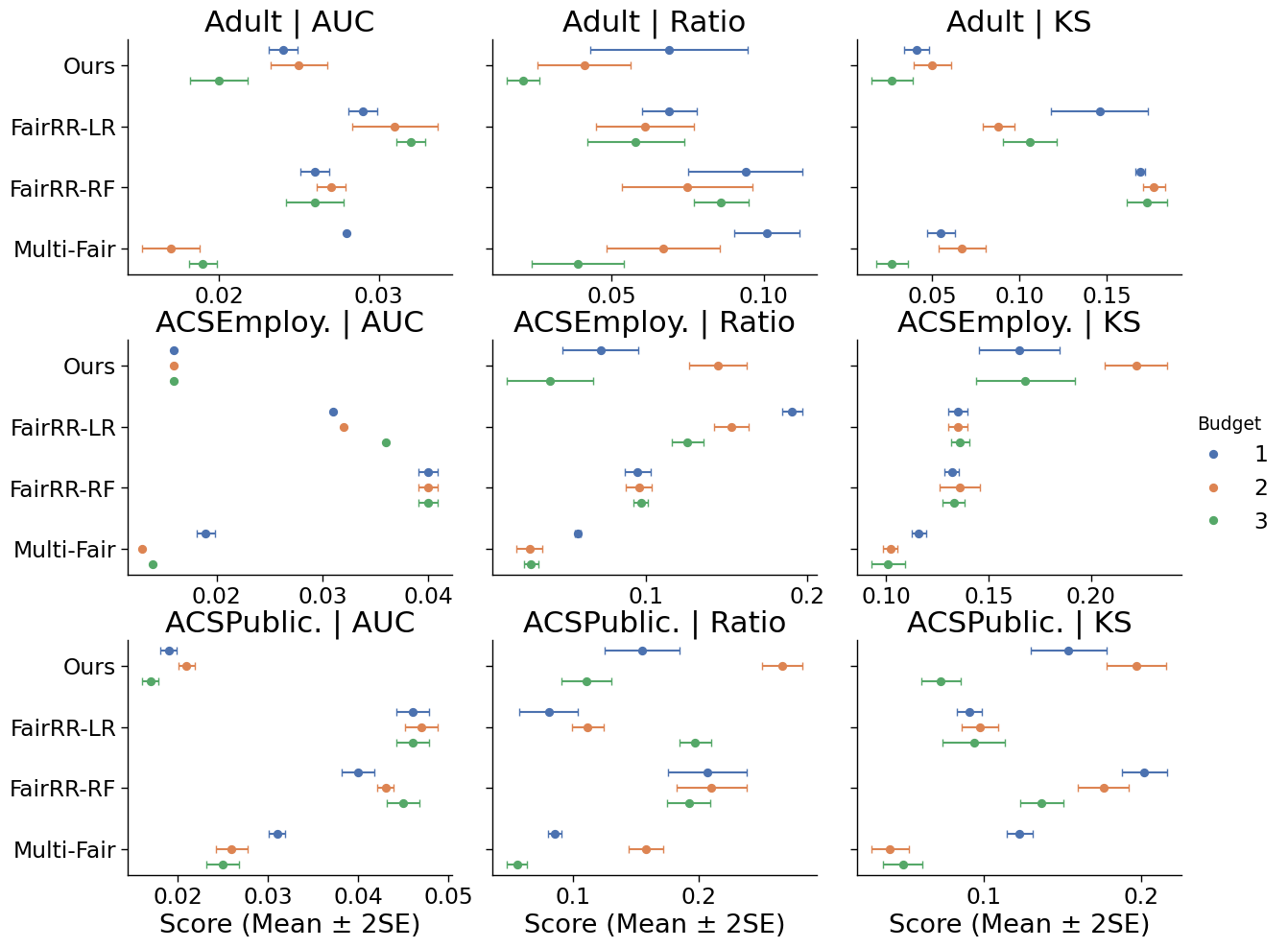}%
    }
    \hfill
    \subfloat[Equalized odds for ${\bf A}_{\text{dis}}$\label{fig:consistency-eo-mul}]{%
        \includegraphics[width=0.48\linewidth]{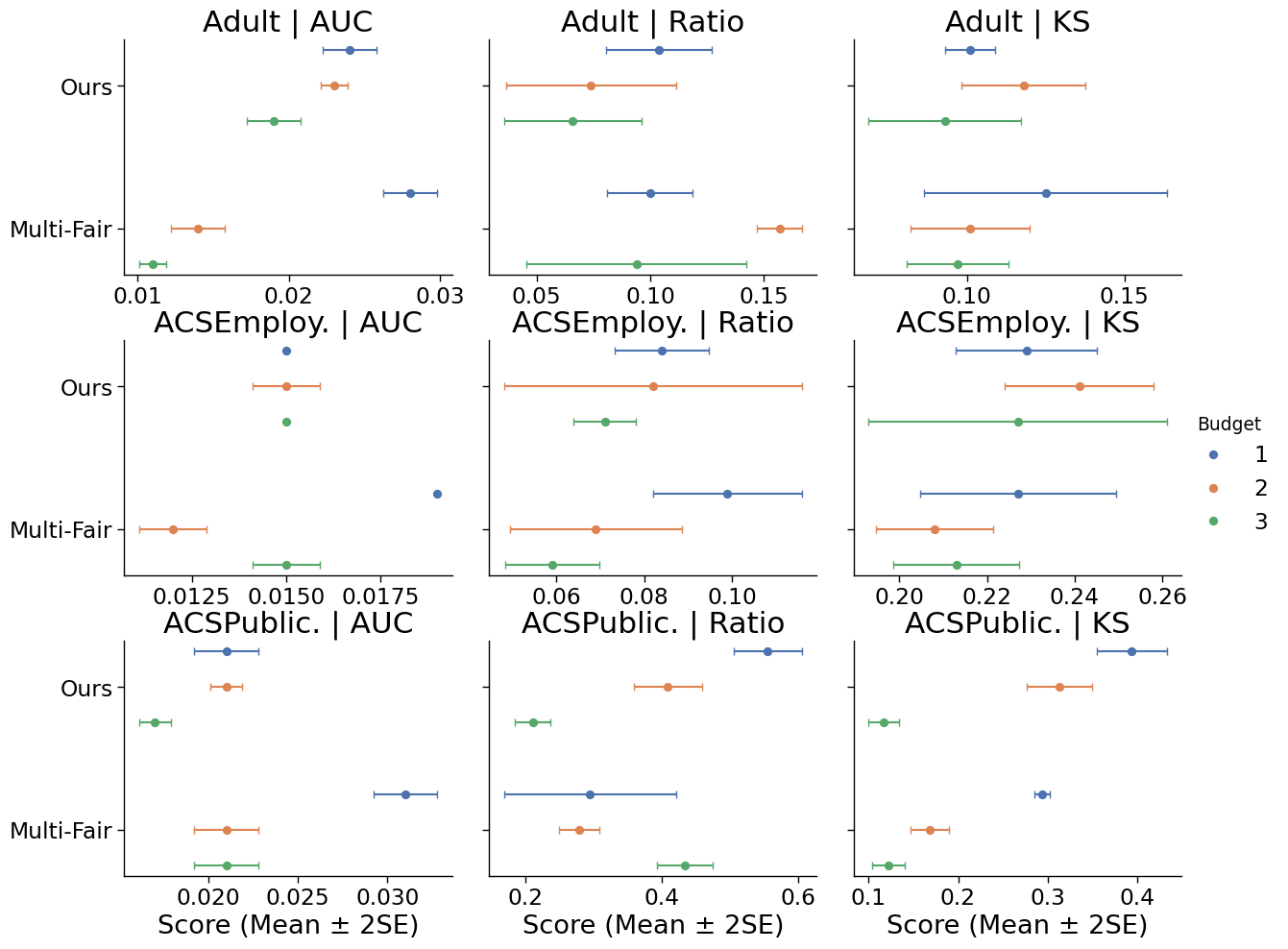}%
    }
    
    \caption{Dots and bars stand for the averages and 2$\times$standard errors. See Section~\ref{simul:quali-quant-analysis} for the definition of the consistency score.}
    \label{fig:four_figures_subfloat}
\end{figure}

\begin{figure}[h]
    \centering
    \subfloat[Consistency scores are based on the case of $A_{\text{bin}}$. \label{fig:consistency-hgb}]{%
        \includegraphics[width=0.9\linewidth]{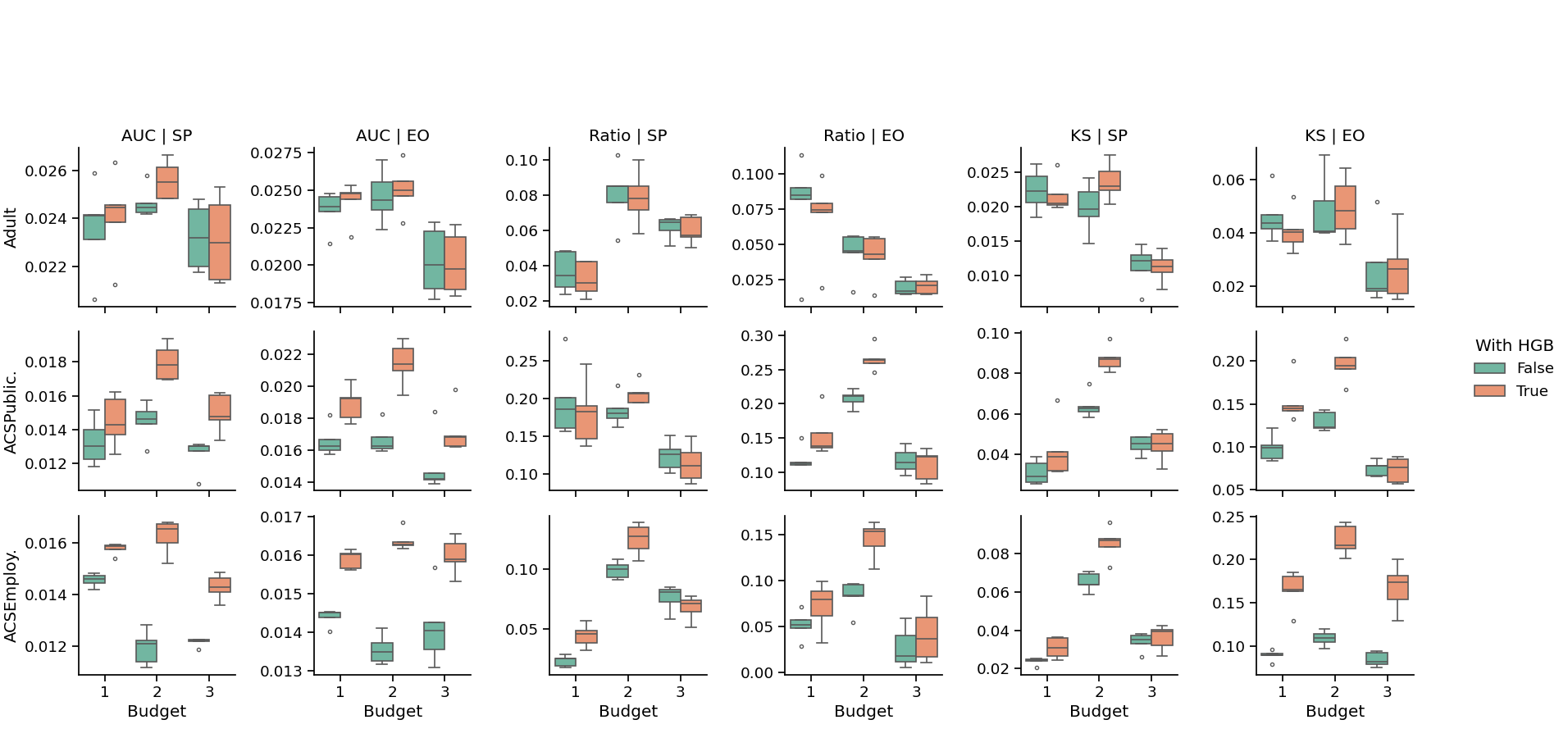}%
    }
    
    \vspace{1em} 
    
    \subfloat[Consistency scores are based on the case of ${\bf A}_{\text{dis}}$. \label{fig:consistency-hgb-mul}]{%
        \includegraphics[width=0.9\linewidth]{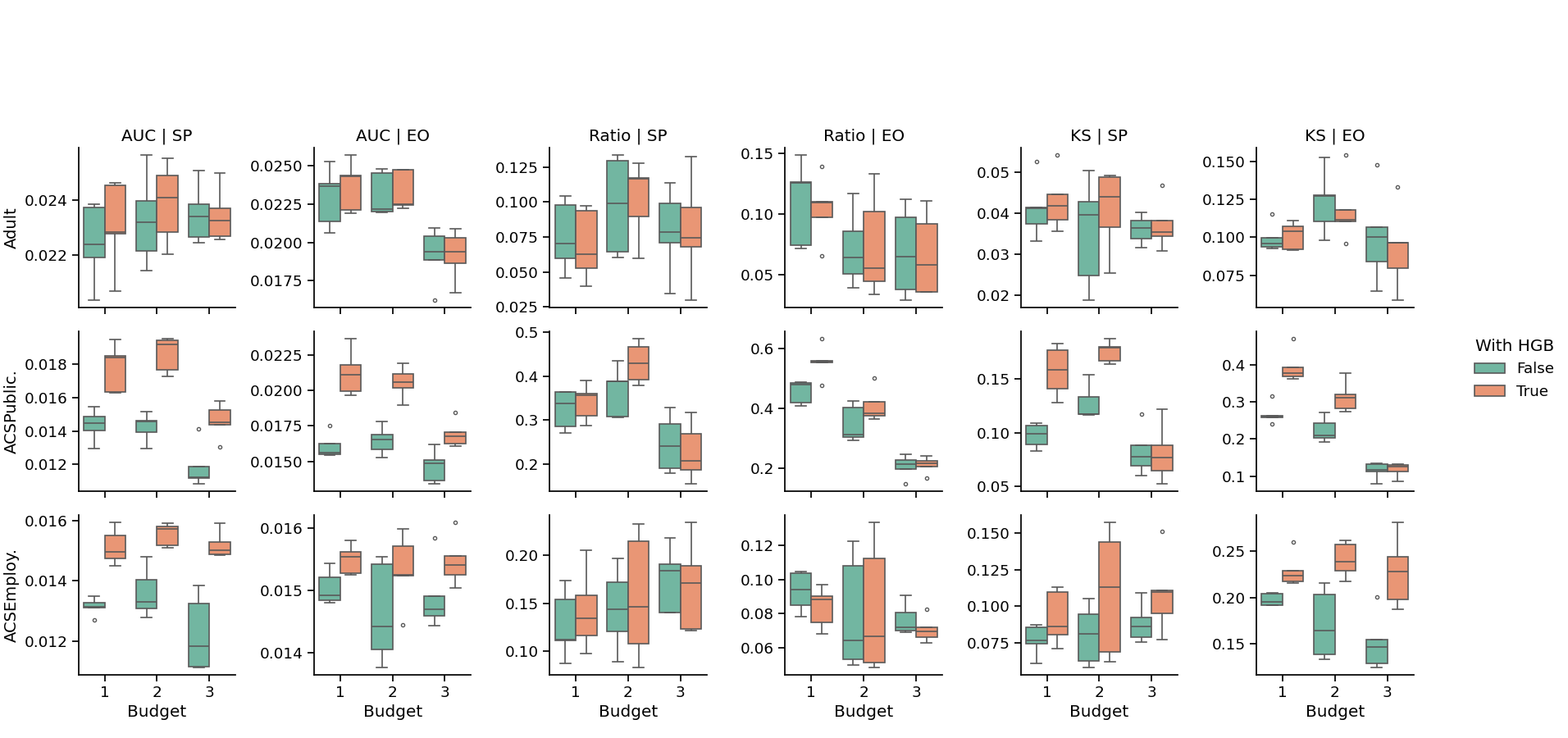}%
    }
    \caption{Influence of the gradient boosting model on consistency scores: The column name,  $\text{AUC (Ratio, KS)}~|~\text{SP}$, represents the consistency score of AUC (Ratio, KS), i.e., the ratio-type and KS-based fairness metrics, when the statistical parity is improved during optimization. Similarly, $\text{AUC (Ratio, KS)}~|~\text{EO}$ stands for the score when the equalized odds is improved.}
    \label{fig:consistency-comparison-hgb}
\end{figure}

\subsection{Comparison with in-processing models}
\label{supp:in-proc}
Figures~\ref{fig:sp-bin-in-proc} and \ref{fig:sp-bin-in-proc-ks} contrast trade-off curves between our upstream/downstream and two in-processing models denoted by HGR \citep{lee:etal:22} and ExpoGrad \citep{agar:etal:18}. The HGR method solves an empirical risk minimization for a neural-net-based classifier with an HGR-correlation-based fairness penalty approximated by neural networks. The classifier adopts the structure of our upstream model for fair comparison. ExpoGrad solves a Lagrangian optimization to find a randomized classifier by sampling a supervised learning model from a pre-specified model class. The figures show that in-processing models tend to have more efficient trade-off curves than our method. We anticipate this outcome, as in-processing methods possess the advantage of directly optimizing the decision boundary of the decision model for both predictive performance and fairness simultaneously. 

Nevertheless, pre-processing has its own fundamental benefits over in-processing. First, pre-processing empowers data distributors to proactively enforce fairness constraints before releasing data. More specifically, data distributors do not need to rely on the data receiver's faithfulness or technical expertise to mitigate discriminative effects, ensuring that the downstream decision process adheres to fairness standards regardless of the downstream user's intent or capability. Secondly, pre-processing methods are model agnostic in the sense of downstream analysis. Pre-processed data can be utilized by any learning algorithm, including black-box legacy systems where modifying the optimization objective or algorithm is systematically impossible. Besides, they facilitate a critical separation of concerns between data governance and model optimization. This allows downstream researchers or developers to focus on predictive performance without needing deep expertise in fairness-constrained optimization, while simultaneously standardizing bias mitigation across heterogeneous model deployments. 
\clearpage
\vspace*{\fill}
\begin{figure*}[h] 
    \centering
    \subfloat[Trade-off plots of Adult  with $\delta_{X}\in \{0.01,0.10,0.30\}$, $\delta_Y=0$, and $\lambda_F = 10$]{%
        \includegraphics[width=\linewidth, height=0.15\textheight, keepaspectratio]{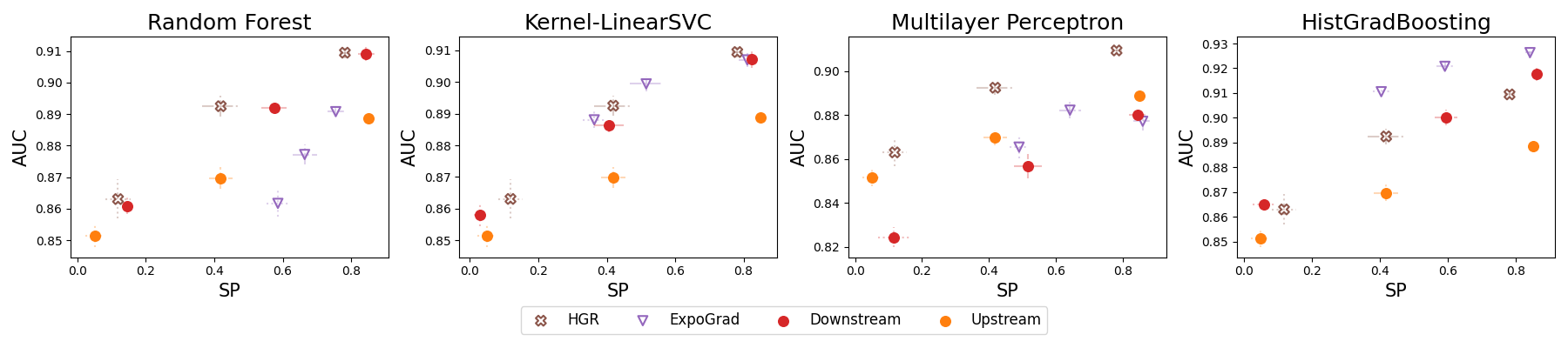}
    }\par
    \subfloat[Trade-off plots of ACSEmployment with $\delta_{X}\in \{0.01,0.025,0.05\}$,  $\delta_Y=0$, and $\lambda_F = 100$]{%
        \includegraphics[width=\linewidth, height=0.15\textheight, keepaspectratio]{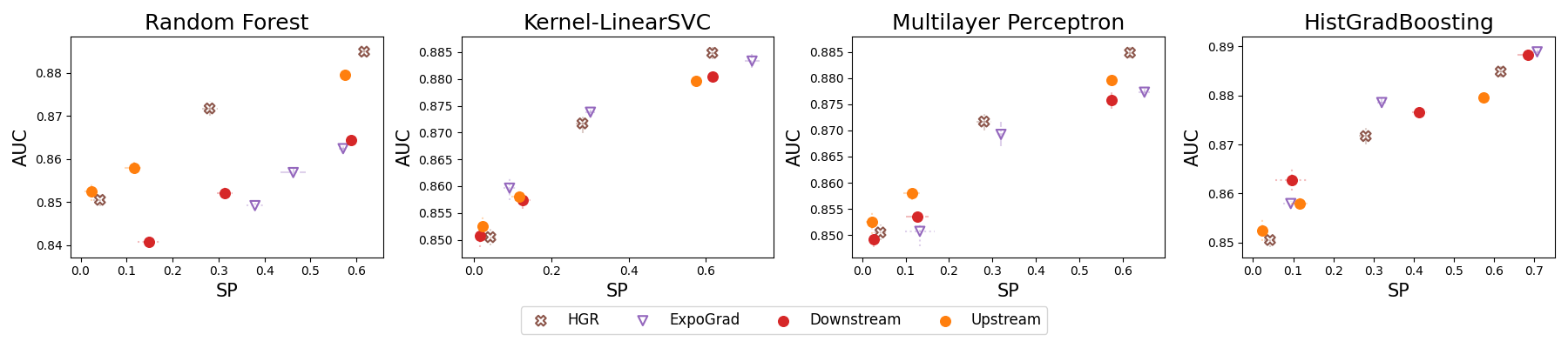}
    }\par
    \subfloat[Trade-off plots of ACSPublicCoverage with $\delta_{X}\in \{0.01,0.05,0.10\}$, $\delta_Y=0.05$, and $\lambda_F = 100$]{%
        \includegraphics[width=\linewidth, height=0.15\textheight, keepaspectratio]{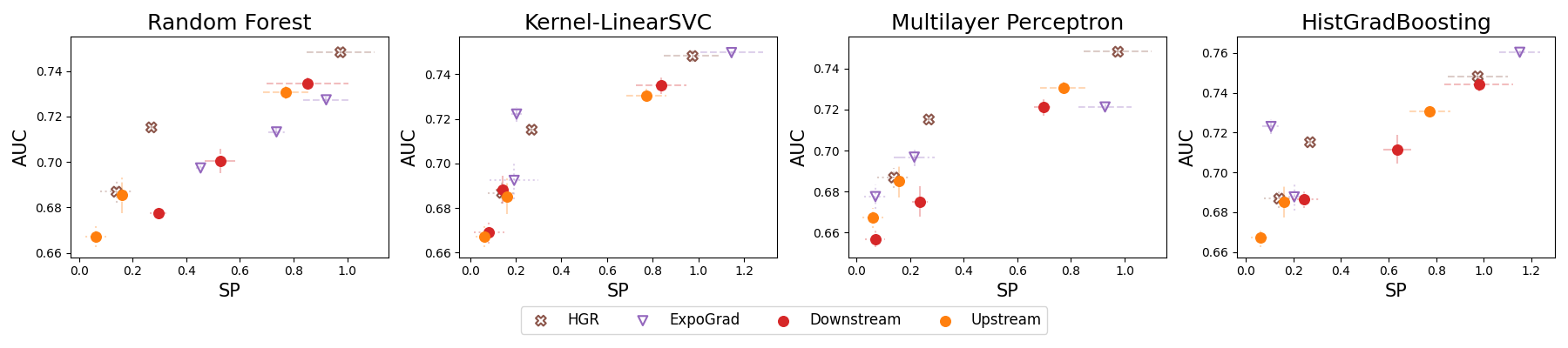}
    }\par
    \caption{Trade-off curves between AUC and SP in handling $A_{\text{bin}}$: HGR and ExpoGrad stand for methods of \cite{lee:etal:22} and \cite{agar:etal:18}, respectively. For HGR, the figures draw the same trade-off curves found from a single neural-network classifier. ExpoGrad draws the curves that each downstream model is used as a target classifier.}
    \label{fig:sp-bin-in-proc}
\end{figure*}
\vspace*{\fill}
\clearpage
\vspace*{\fill}
\begin{figure*}[h] 
    \centering
    \subfloat[Trade-off plots of Adult  with $\delta_{X}\in \{0.01,0.10,0.30\}$, $\delta_Y=0$, and $\lambda_F = 10$]{%
        \includegraphics[width=\linewidth, height=0.15\textheight, keepaspectratio]{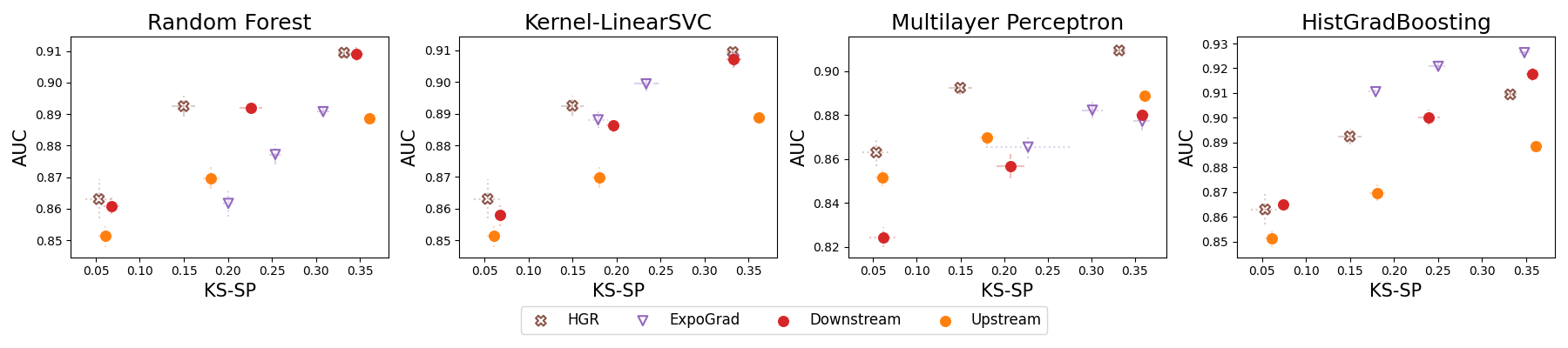}
    }\par
    \subfloat[Trade-off plots of ACSEmployment with $\delta_{X}\in \{0.01,0.025,0.05\}$,  $\delta_Y=0$, and $\lambda_F = 100$]{%
        \includegraphics[width=\linewidth, height=0.15\textheight, keepaspectratio]{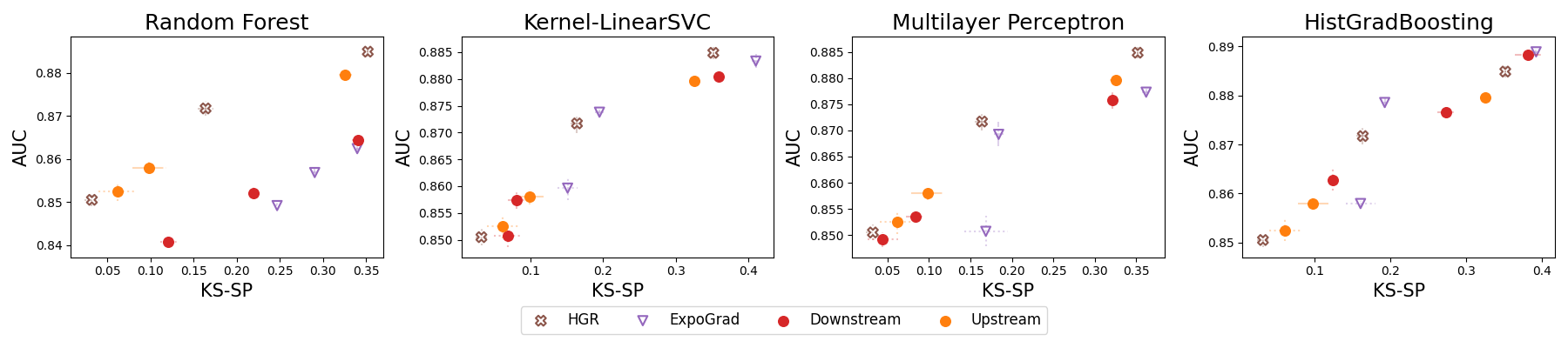}
    }\par
    \subfloat[Trade-off plots of ACSPublicCoverage with $\delta_{X}\in \{0.01,0.05,0.10\}$, $\delta_Y=0.05$, and $\lambda_F = 100$]{%
        \includegraphics[width=\linewidth, height=0.15\textheight, keepaspectratio]{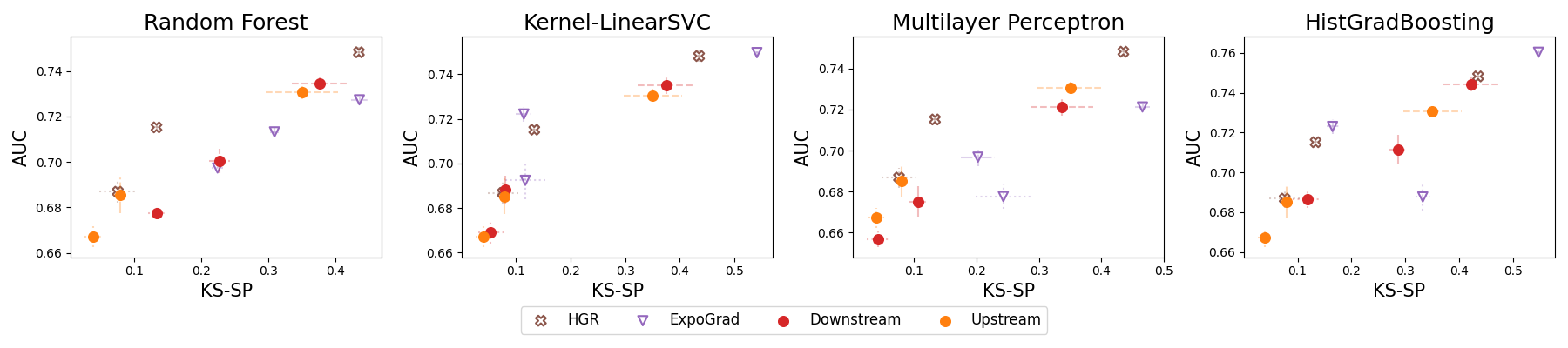}
    }\par
    \caption{Trade-off curves between AUC and KS-SP in handling $A_{\text{bin}}$: HGR and ExpoGrad stand for methods of \cite{lee:etal:22} and \cite{agar:etal:18}, respectively. See Figure~\ref{fig:sp-bin-in-proc} for details.}
    \label{fig:sp-bin-in-proc-ks}
\end{figure*}
\vspace*{\fill}
\clearpage

\section{Proof}
\label{supp:proof}

\subsection{Proposition~\ref{prop:motiv}}

Recall that $\bX$ and $\bA$ are random vectors. We first study the properties of HGR correlation on random vectors necessary for our analysis. First and foremost, the joint independence is defined as follows. 
\begin{definition}[Joint independence]
    Suppose $\bX$ and $\bA$ are finite-dimensional random vectors. The joint independence $\bX \perp \bA$ holds if 
    \begin{align*}
        P(X_1\in E_1,\dots,X_p\in E_p | \bA)=P(X_1\in E_1,\dots,X_p\in E_p)
    \end{align*}
    for all measurable sets $E_1,\dots,E_p$. Then, we say $\bX$ is jointly independent of $\bA$. 
\end{definition}
In our main text, data fairness leverages this joint independence notion. Not surprisingly, this multivariate extension inherits most properties of the HGR correlation shown in \cite{reny:59:2}.
\begin{proposition}
    The multivariate-version HGR correlation satisfies 
    \begin{itemize}
        \item[P1] $0 \leq \rho(\bX,\bA) \leq 1$.
        \item[P2] $\rho(\bX,\bA)=\rho(\bA,\bX)$.
        \item[P3] $\rho(\bX,\bA)=0$ if and only if $\bX$ and $\bA$ are jointly independent.
        \item[P4] $\rho(\bX,\bA)=1$ if there exists some deterministic functional relationship between $\bX$ and $\bA$. 
    \end{itemize}
\end{proposition}
\begin{proof}
    The definition directly shows P1 and P2.  
    For P3, we directly follow \cite{reny:59:2}. Let's consider $f_{E_1,\dots,E_p}$ and $g_{F_1,\dots,F_q}$. 
    \begin{align*}
    f_{E_1,\dots,E_p}(\bX)=
    \begin{cases}
    1, & \mbox{if } X_1 \in E_1, \dots, X_p \in E_p \\
    0, & \mbox{otherwise}, 
    \end{cases}    
    \end{align*}
    and  
    \begin{align*}
        g_{F_1,\dots,F_q}(\bA)=
        \begin{cases}
            1, & \mbox{if } A_1\in F_1, \dots, A_q \in F_q \\
            0, & \mbox{otherwise},
        \end{cases}
    \end{align*}
    where $E_1,\dots,E_p,F_1,\dots,F_q$ are arbitrary Borel sets of such that $0 < P(X_1\in E_1,\dots,X_p \in E_p) < 1$ and $0<P(A_1\in F_1,\dots,A_q\in F_q)<1$. Then, 
    \begin{align*}
        &\text{Corr}(f_{E_1,\dots,E_p}(\bX),g_{F_1,\dots,F_q}(\bA))\\ 
        &=\dfrac{P(X_1\in E_1,\dots,A_q \in F_q)-P(X_1\in E_1,\dots)P(A_1 \in F_1,\dots)}{\sqrt{P(X_1\in E_1,\dots)(1-P(X_1\in E_1,\dots))}\sqrt{P(A_1\in F_1,\dots)(1-P(A_1 \in F_1,\dots))}}.
    \end{align*}
    Therefore, if $\rho(\bX,\bA)=0$, then 
    \begin{align*}
    &P(X_1\in E_1,\dots, X_p\in E_p,A_1\in F_1,\dots,A_q\in F_q)\\
    &=P(X_1\in E_1,\dots,X_p\in E_p)P(A_1\in F_1,\dots A_q \in F_q)    
    \end{align*}
     since $\text{Corr}(f_{E_1,\dots,E_p}(\bX),g_{F_1,\dots,F_q}(\bA))\leq \rho(\bX,\bA)$.
    For P4 as well, we can directly borrow the proof in \cite{reny:59:2}. 
\end{proof}

The proof of Proposition~\ref{prop:motiv} is then straightforward by definition. For a given Borel-measurable function $u:{\calX}\rightarrow \calX^{\dagger}$, the composition set $\{f'\circ u:f' \text{ is Borel measurable on the real line}\}$ is a subset of $\{f:\calX \rightarrow \mathbb{R} \text{ is Borel measurable}\}$. Hence, we conclude $\sup_{f',g}\bE[f'(u(\bX))g(\bA)]\leq \sup_{f,g}\bE[f(\bX)g(\bA)]$.

\subsection{Lemma~\ref{lem:triangle}}

Suppose generic random vectors $X$ and $Y$
are regular and its mean square contingency is finite. By Theorem~2 in \cite{reny:59}, there exist $f^*$ and $g^*$, called maximal correlation functions, such that $\rho(X,Y)=\bE[f^*(X)g^*(Y)]$ holds. For simple notation, we say $X^* = f^*(X)$ and $Y^* = g^*(Y)$ with $\bE[X^*]=0$ and $\bE[X^{*2}]=1$ (respectively for $Y^*$ as well). 

Let's first observe \begin{align*}
        \sqrt{\bE||X^* - Y^*||^2}=\sqrt{\bE[X^{*2}]+\bE[Y^{*2}]-2\bE[X^*Y^*]}=\sqrt{2-2\bE[X^*Y^*]}, 
\end{align*}
and it further implies that 
\begin{align*}
    d'(X^*,Y^*)=\sqrt{2-2|\bE[X^*Y^*]|}=\min\{\sqrt{\bE||X^* - Y^*||^2}, \sqrt{\bE||X^* + Y^*||^2}\}.
\end{align*}
Note $\sqrt{2-2\bE[X^*Y^*]}=\sqrt{2-2|\bE[X^*Y^*]|}=\sqrt{2-2\rho(X,Y)}$ since $\rho$ is positive.

From the result of \cite{chen:etal:23},  given any additional variable $Z^*$,
\begin{align*}
        &\min\{\sqrt{\bE||X^* - Y^*||^2}, \sqrt{\bE||X^* + Y^*||^2}\}  \\ 
        \leq& \min\{\sqrt{\bE||X^* - Z^*||^2 + \bE|| Z^* - Y^*||^2}, \sqrt{\bE||X^* + Z^* ||^2 + \bE|| Z^* + Y^*||^2}, \\ 
        &~~~~~~~~~~~~\sqrt{\bE||X^* - Z^*||^2 + \bE|| Z^* + Y^*||^2}, \sqrt{\bE||X^* + Z^*||^2 + \bE|| Z^* - Y^*||^2} \} \\ 
        \leq& \min\{\sqrt{\bE||X^* - Z^*||^2} + \sqrt{\bE|| Z^* - Y^*||^2}, \sqrt{\bE||X^* + Z^* ||^2} + \sqrt{\bE|| Z^* + Y^*||^2}, \\ 
        &~~~~~~~~~~~~\sqrt{\bE||X^* - Z^*||^2} + \sqrt{\bE|| Z^* + Y^*||^2}, \sqrt{\bE||X^* + Z^*||^2} + \sqrt{\bE|| Z^* - Y^*||^2} \}.
    \end{align*}
where the last line equals to
\begin{align*}
        \min\{ \sqrt{\bE||X^* - Z^*||^2}, \sqrt{\bE||X^* + Z^*||^2}\} + \min\{ \sqrt{\bE||Y^* - Z^*||^2}, \sqrt{\bE||Y^* + Z^*||^2}\},
\end{align*}
so $d'$ satisfies the triangle inequality.

Now, let's consider other pairs $(Z,Y)$ and $(X,Z)$. We denote by $f^*$ and $g^*$ the optimal maximal functions for $X$ and $Y$; $k^*$ and $g'$ for $Z$ and $Y$; $f^{\dagger}$ and $k^{\dagger}$ for $X$ and $Z$. Then, we can observe
\begin{align*}
   \rho(Z,Y)&=\bE[k^*(Z)g'(Y)] \geq \bE[k^*(Z)g^*(Y)]=\tilde{\rho}(Z,Y) \\
    \rho(X,Z)&=\bE[f^{\dagger}(X)k^{\dagger}(Z)] \geq \bE[f^*(X)k^*(Z)]=\tilde{\rho}(X,Z), \\
    \rho(X,Y)&=\bE[f^*(X)g^*(Y)]\geq\bE[f^*(X)g'(Y)]=\tilde{\rho}(X,Y).
\end{align*}
On the one hand, if $Y$ is binary ($P(Y=1)=p$),  there exist $a$ and $b$ such that $g(Y)=aY+b$ for any mapping $g$. If furthermore, $\bE[g(Y)]=0$ and $\bE[g(Y)^2]=1$, then,
\begin{align*}
    V(g(Y))=a^2 p(1-p) = 1 \rightarrow a = \pm \dfrac{1}{\sqrt{p(1-p)}}. \\ 
    \bE[g(Y)]=ap + b = 0 \rightarrow b = \mp \dfrac{\sqrt{p}}{\sqrt{1-p}},
\end{align*}
which means that the optimal maximal function must satisfy  $g^* = g'$ or $g^* = -g'$. This implies that $\tilde{\rho}(X,Y)=\rho(X,Y)$ or $\tilde{\rho}(X,Y)=-\rho(X,Y)$. 

From the triangle inequality $d'(X^*,Z^*)\leq d'(X^*,Y') + d'(Y',Z^*)$, we have 
\begin{align*}
    \sqrt{2-2|\tilde{\rho}(X,Z)|}\leq \sqrt{2-2|\tilde{\rho}(X,Y)|} + \sqrt{2-2|\rho(Y,Z)|}. 
\end{align*}
Finally, since $|\rho(X,Z)| \geq |\tilde{\rho}(X,Z)|$ and $\tilde{\rho}(X,Y)=\rho(X,Y)$, we obtain 
\begin{align*}
    \sqrt{1-|{\rho}(X,Z)|}\leq \sqrt{1-|{\rho}(X,Y)|} + \sqrt{1-|{\rho}(Y,Z)|}. 
\end{align*}
Therefore, the triangle inequality holds if the pivot variable is binary.

\subsection{Theorem~\ref{thm:uplwh0}}

We show the lower bound first. Let $\check{h}:\calX\times \calA \rightarrow \calY$, $h:\calX \rightarrow \calY$, and $G_{\bX}:\calX\times \calA\rightarrow \calX$. Let's consider the class of  composite functions $h\circ G_{\bX}: \calX\times \calA \rightarrow \calY$ under the constraint $\Delta_{\bX}(\bX,\bbarX)\leq \delta_{\bX}$. The size of this class expands as $\delta_{\bX}$ gets larger, i.e., $\delta_{\bX}>0$ leads to a larger composition class than that of $\delta_{\bX}=0$, thus \eqref{opt:min} implies 
\begin{align*}
    \bE[l(Y,\tih^*(\btiX))]+\lambda_F \rho(\tih^*(\btiX),\bA) \leq \bE[l(Y,h^*(\bX))]+\lambda_F \rho(h^*(\bX),\bA),
\end{align*}
equivalently,
\begin{align*}
    \lambda_F (\rho(h^*(\bX),\bA) - \rho(\tih^*(\btiX),\bA))\geq \bE[l(Y,\tih^*(\btiX))] - \bE[l(Y,h^*(\bX))].
\end{align*}
Since $\check{e}=\min_{\check{h}} \bE[l(Y,\check{h}(\bX,\bA))]$ and $e(\bA)=\bE[l(Y,h^*(\bX))]-\check{e}$, we replace $\bE[l(Y,h^*(\bX))]$ by $e(\bA)+\check{e}$.

Now, let's further suppose $\tih^*$ and $h^*$ are binary. Then, 
\begin{align*}
   \rho(h^*(\bX),\bA) - \rho(\tih^*(\btiX),\bA)=\bE[f^*(h^*(\bX))g^*(\bA)]-\bE[\tilde{f}^*(\tih^*(\btiX))\tilde{g}^*(\bA)], 
\end{align*}
where $f^*,g^*,\tilde{f}^*,\tilde{g}^*$ are the functions that achieve the maximal correlations. Since $\tih^*$ and $h^*$ are binary, the potential functions are linear, i.e., 
\begin{align*}
    f^*(x) = a x + b, \quad \tilde{f}^*(x)=\tilde{a} x + \tilde{b},
\end{align*}
and $a=\pm(\sqrt{p(1-p)})^{-1}$, $\tilde{a}=\pm(\sqrt{\tilde{p}(1-\tilde{p})})^{-1}$ where $p=P(h^*(\bX)=1)$ and $\tilde{p}=P(\tih^*(\btiX)=1)$. The definition finds these constants. Now, we replace $\tilde{g}^*$ with $g^*$ to find an upper bound. Then, we have
\begin{align*}
    \bE[f^*(h^*(\bX))g^*(\bA)]-\bE[\tilde{f}^*(\tih^*(\btiX))\tilde{g}^*(\bA)] &\leq \bE[f^*(h^*(\bX))g^*(\bA)]-\bE[\tilde{f}^*(\tih^*(\btiX))g^*(\bA)],\\
    &=\bE[g^*(\bA)(f^*(h^*(\bX))-\tilde{f}^*(\tih^*(\btiX)))].
\end{align*}
Note $a$ and $\tilde{a}$ can be set to having the same sign because $f^*$ and $\tilde{f}^*$ are linear. For instance, suppose that $f^*$ with $a>0$ and $g^*$ are of the ground truth. If $a<0$ is forced, then $-g^*$ becomes the corresponding maximal potential. Therefore, by using the Cauchy-Schwartz inequality, we obtain
\begin{align*}
    \bE[g^*(\bA)(f^*(h^*(\bX))-\tilde{f}^*(\tih^*(\btiX)))] &\leq |\bE[g^*(\bA)(f^*(h^*(\bX))-\tilde{f}^*(\tih^*(\btiX)))]|,\\
    &\leq \sqrt{\bE[g^*(\bA)^2]\bE[(f^*(h^*(\bX))-\tilde{f}^*(\tih^*(\btiX)))^2]}, \\
    &=\sqrt{2-2\bE[f^*(h^*(\bX))\tilde{f}^*(\tih^*(\btiX))]},\\
    &=\sqrt{2-2\text{Corr}(f^*(h^*(\bX)),\tilde{f}^*(\tih^*(\btiX)))},\\
    &=\sqrt{2-2\rho(h^*(\bX),\tih^*(\btiX))}, \\
    &=d(h^*(\bX),\tih^*(\btiX)).
\end{align*}
since $\bE[g^*(\bA)^2]=\bE[f^*(h^*(\bX))^2]=\bE[\tilde{f}^*(\tih^*(\btiX))^2]=1$. Thus, by applying the triangle inequality (Lemma~\ref{lem:triangle}), i.e., $d(h^*(\bX),\tih^*(\btiX))\leq d(h^*(\bX),Y)+d(Y,\tih^*(\btiX))$, the statement is obtained. 

Similarly, we replace $g^*$ by $\tilde{g}^*$.
\begin{align*}
    \bE[f^*(h^*(\bX))g^*(\bA)]-\bE[\tilde{f}^*(\tih^*(\btiX))\tilde{g}^*(\bA)] &\geq \bE[f^*(h^*(\bX))\tilde{g}^*(\bA)]-\bE[\tilde{f}^*(\tih^*(\btiX))\tilde{g}^*(\bA)],\\
        &=\bE[\tilde{g}^*(\bA)(f^*(h^*(\bX))-\tilde{f}^*(\tih^*(\btiX)))],\\
        &=-|\bE[\tilde{g}^*(\bA)(f^*(h^*(\bX))-\tilde{f}^*(\tih^*(\btiX)))]|.
\end{align*}
Then, by applying the Cauchy-Schwartz inequality, we have
\begin{align}
\label{eqn:lb}
    \bE[f^*(h^*(\bX))g^*(\bA)]-\bE[\tilde{f}^*(\tih^*(\btiX))\tilde{g}^*(\bA)] \geq - d(h^*(\bX),\tih^*(\btiX)).
\end{align}

\subsection{Corollary~\ref{cor:uplwh0}}

The lower bound is derived from $d(\tih^*(\btiX),\bA)\geq d(\tiY,\bA)-d(\tih^*(\btiX),\tiY)$ which is the direct result of the triangle inequality. Following the proof of Theorem~\ref{thm:uplwh0}, we replace the upper bound from the triangle inequality by $d(h^*(\bX),\tih^*(\btiX))\leq  d(h^*(\bX),\tiY)+d(\tiY,\tih^*(\btiX))$.

\subsection{Lemma~\ref{lem:util}}

We first define the Lipschitz continuity. 

\begin{definition}
    For the metric space $(\calX,d_X)$ and $(\calY,d_Y)$, a function $h:{\calX}\rightarrow \calY$ is called $L$-Lipschitz if $d_Y(h(x),h(y))\leq L d_X(x,y)$ for all $x,y$. 
\end{definition}

\begin{definition}
    A loss function $L:{\calY}\times \calY \rightarrow \mathbb{R}$ is said to be $K$-Lipschitz, if $|L(y,\hat{y}_1)-L(y,\hat{y}_2)|\leq K d_{Y}(\hat{y}_1,\hat{y}_2)$ for all $\hat{y}_1,\hat{y}_2,y\in \hat{\calY}$, and $|L(y_1,\hat{y})-L(y_2,\hat{y})|\leq K d_{Y}(y_1,y_2)$ for all ${y}_1,{y}_2,\hat{y} \in {\calY}$. 
\end{definition}

Suppose the loss function is $M^l$-Lipschitz. Also, $h_k^*$ and $\tih^*$ are $M_k$ and $\tilde{M}$-Lipschitz respectively. Let $\tih^{\dagger}$ is the optimizer under $(\btiX,Y)$, which means $\bE[l(Y,\tih^{\dagger}(\btiX))]\leq \bE[l(Y,\tih^*(\btiX))]$. Since $\tih_k^*$ is the optimizer under $(\btiX,\tiY)$, 
\begin{align*}
    \bE[l(\tiY,\tih_k^*(\btiX))]-\bE[l(\tiY,\tih^*(\btiX))] \leq \bE[l(\tiY,h_k^*(\btiX))]-\bE[l(\tiY,\tih^*(\btiX))], 
\end{align*}
and it is further characterized as
\begin{align*}
    & = \bE[l(\tiY,h_k^*(\btiX))]-\bE[l(Y,\tih^*(\btiX))]+\underbrace{\bE[l(Y,\tih^*(\btiX))]-\bE[l(\tiY,\tih^*(\btiX))]}_{Q_1^0} \\ 
    & \leq \underbrace{\bE[l(\tiY,h_k^*(\btiX))]- \bE[l(Y,h_k^*(\btiX))]}_{Q_1^k}+\bE[l(Y,h_k^*(\btiX))]-\bE[l(Y,\tih^{\dagger}(\btiX))] + Q_1^0 \\
    &=  \underbrace{\bE[l(Y,h_k^*(\btiX))]-\bE[l(Y,h_k^*(\bX))]}_{Q_2^k}+\bE[l(Y,h_k^*(\bX))]-\bE[l(Y,\tih^{\dagger}(\btiX))] + Q_1^0 + Q_1^k \\ 
    &= \underbrace{\bE[l(Y,h_k^*(\bX))]-\bE[l(Y,h^*(\bX))]}_{\Delta_L^k} + \bE[l(Y,h^*(\bX))]-\bE[l(Y,\tih^{\dagger}(\btiX))]  +  Q_1^0 + Q_1^k + Q_2^k \\ 
    &= \bE[l(Y,h^*(\bX))]-\bE[l(Y,\tih^{\dagger}(\btiX))]  +  Q_1^0 + Q_1^k + Q_2^k + \Delta_L^k \\ 
    &\leq \underbrace{\bE[l(Y,\tih^{\dagger}(\bX))]-\bE[l(Y,\tih^{\dagger}(\btiX))]}_{Q_2^0} +  Q_1^0 + Q_1^k + Q_2^k + \Delta_L^k.
\end{align*}
Hence, we derive
\begin{align*}
    \bE[l(\tiY,\tih_k^*(\btiX))]-\bE[l(\tiY,\tih^*(\btiX))] \leq Q_1^0 + Q_1^k + Q_2^0 + Q_2^k + \Delta_L^k.
\end{align*}
Also, let's further suppose that $\Delta_{\bX}$ and $\Delta_Y$ correspond to the metrics $d_{X}$ and $d_Y$. Then we obtain
\begin{align*}
    &|Q_1^0|, |Q_1^k| \leq M^l \bE[d_Y(Y,\tiY)]=M^l \Delta_Y \leq M^l \delta_Y \\
    &|Q_2^k| \leq M^l \bE[d_{\hat{Y}}(h_k^*(\bX),h_k^*(\btiX))]\leq M^l M_k \bE[d_X(\bX,\btiX)]=M^l M_k \Delta_{\bX}\leq M^l M_k\delta_{\bX} \\ 
    &|Q_2^0| \leq M^l \bE[d_{\hat{Y}}(\tih^{\dagger}(\bX),\tih^{\dagger}(\btiX))]\leq M^l \tilde{M} \bE[d_X(\bX,\btiX)]= M^l \tilde{M} \Delta_{\bX}\leq M^l \tilde{M}\delta_{\bX}.
\end{align*}
Therefore, we conclude 
\begin{align*}
    \bE[l(\tiY,\tih^*(\btiX))]\leq \bE[l(\tiY,\tih_k^*(\btiX))] \leq \bE[l(\tiY,\tih^*(\btiX))] + 2M^l \delta_Y + M^l(M_k+\tilde{M}) \delta_{\bX} + \Delta_L^k.
\end{align*}

Let $(\btiX',\bA',Y')$ be the test data and $\btiD'=(\btiX',Y')$ be the transformed of it. Note the label of $\btiD'$ is not transformed. Evaluating the risk of $\tih_k^*$ on $\btiD'=(\btiX',Y')$ where $\btiD' \perp \bD$ is of great interest since it accounts for the predictive error. That is, $\tih_k^*$ is learned on $\btiD$ and evaluates $\bE[l(Y',\tih_k^*(\btiX'))]$.  We assume there is no domain shift between $(\bX,\bA,Y)$ and $(\bX',\bA',Y')$. In the population perspective, therefore,   $\bE[l(Y',\tih_k^*(\btiX'))]$ can be read as $\bE[l(Y,\tih_k^*(\btiX))]$, denoting by $\btiD'=(\btiX,Y)$. Then we observe the following inequality
\begin{align*}
    &|\bE[l(Y,\tih_k^*(\btiX))]-\bE[l(\tiY,\tih^*(\btiX))]|\leq \\
    &|\bE[l(Y,\tih_k^*(\btiX))]-\bE[l(\tiY,\tih_k^*(\btiX))]|+|\bE[l(\tiY,\tih_k^*(\btiX))]-\bE[l(\tiY,\tih^*(\btiX))]|.
\end{align*}
Hence, by following the same proof technique, it is straightforward to see that 
\begin{align*}
    |\bE[l(Y,\tih_k^*(\btiX))]-\bE[l(\tiY,\tih^*(\btiX))]| \leq  3M^l \delta_Y + M^l(M_k+\tilde{M}) \delta_{\bX} + \Delta_L^k.
\end{align*}

\subsection{Theorem~\ref{thm:uplwhk}}

For effective discussion, we say $F(h_k^*;\bD)=\lambda_F\rho(h_k^*(\bX),\bA)$, $F(\tih_k^*;\btiD)=\lambda_F\rho(\tih_k^*(\btiX),\bA)$, and $F(\tih^*;\btiD)=\lambda_F\rho(\tih^*(\btiX),\bA)$. For the lower bound, the fairness of $\tih_k^*$ is decomposed as 
\begin{align*}
    &F(h_k^*;\bD) - F(\tih_k^*;\btiD)\\
    &= F(h^*;\bD) - \lambda_F \Delta_F^k  - F(\tih_k^*;\btiD), \\
    &= -\lambda_F \Delta_F^k + F(h^*;\bD) - F(\tih^*;\btiD) +  F(\tih^*;\btiD)- F(\tih_k^*;\btiD), \\
    &=\lambda_F \tilde{\Delta}_F - \lambda_F \Delta_F^k + F(\tih^*;\btiD)- F(\tih_k^*;\btiD), \\
    &=\lambda_F \tilde{\Delta}_F - \lambda_F \Delta_F^k + F(\tih^*;\btiD) + L(\tih^*;\btiD) - L(\tih^*;\btiD) - F(\tih_k^*;\btiD).
\end{align*}
By Lemma~\ref{lem:util}, the last line is further bounded below by
\begin{align*}
    &\geq \lambda_F\tilde{\Delta}_F - \lambda_F\Delta_F^k + F(\tih^*;\btiD) + L(\tih^*;\btiD) - L(\tih^*_k;\btiD) - F(\tih_k^*;\btiD), \\ 
     &\geq \lambda_F\tilde{\Delta}_F - \lambda_F\Delta_F^k -2M^l \delta_Y - M^l(M_k+M) \delta_{\bX} - \Delta_L^k -(F(\tih^*;\btiD)- F(\tih_k^*;\btiD)),
     \end{align*}
and by applying \eqref{eqn:lb}, we have
\begin{align*}
     &\geq \lambda_F \tilde{\Delta}_F - \lambda_F \Delta_F^k -2M^l \delta_Y - M^l(M_k+M) \delta_{\bX} - \Delta_L^k - \lambda_F d(\tih^*(\btiX),\tih_k^*(\btiX)).
\end{align*}

To derive the upper bound of $F(h_k^*;\bD)-F(\tih_k^*;\btiD)$, we observe 
\begin{align*}
    F(\tih_k^*;\btiD)-F(h_k^*;\bD)&=F(\tih_k^*;\btiD) + L(\tih_k^*;\btiD) - L(\tih_k^*;\btiD)-F(h_k^*;\bD), 
\end{align*}
Since $F(\tih_k^*;\btiD) + L(\tih_k^*;\btiD)\geq F(\tih^*;\btiD) + L(\tih^*;\btiD)$, the right-hand side is bounded by 
\begin{align*}
    &\geq F(\tih^*;\btiD) + L(\tih^*;\btiD)- L(\tih_k^*;\btiD)-F(h_k^*;\bD)\\ 
    &=F(\tih^*;\btiD) - F(h^*;
    \bD)+ L(\tih^*;\btiD) - L(\tih_k^*;\btiD) + F(h^*;
    \bD) -F(h_k^*;\bD)\\
    &=-\lambda_F \tilde{\Delta}_F+ L(\tih^*;\btiD) - L(\tih_k^*;\btiD) + \lambda_F \Delta_F^k.
\end{align*}
By multiplying -1 to both sides, it follows that
\begin{align*}
    F(h_k^*;\bD) - F(\tih_k^*;\btiD) &\leq \lambda_F \tilde{\Delta}_F - \lambda_F \Delta_F^k +L(\tih_k^*;\btiD)- L(\tih^*;\btiD),\\ 
    &\leq \lambda_F \tilde{\Delta}_F - \lambda_F \Delta_F^k +L(\tih_k^*;\btiD) + F(\tih_k^*;\btiD) - F(\tih_k^*;\btiD) - L(\tih^*;\btiD), \\ 
    &\leq \lambda_F \tilde{\Delta}_F - \lambda_F \Delta_F^k +L(\tih_k^*;\btiD) + F(\tih_k^*;\btiD) - F(\tih^*;\btiD) - L(\tih^*;\btiD).
\end{align*}
Moreover, by referring the proof in Theorem~\ref{thm:uplwh0}, we derive 
\begin{align*}
    \rho(\tih_k^*(\btiX),\bA)-\rho(\tih^*(\btiX),\bA)\leq d(\tih_k^*(\btiX),\tiY)+d(\tih^*(\btiX),\tiY).
\end{align*}
In consequence, by Lemma~\ref{lem:util}, we characterize
\begin{align*}
     &= L(\tih_k^*;\btiD)-L(\tih^*;\btiD) + F(\tih_k^*;\btiD) - F(\tih^*;\btiD), \\ 
     &= L(\tih_k^*;\btiD)-L(\tih^*;\btiD) + \lambda_F(\rho(\tih_k^*(\btiX),\bA)-\rho(\tih^*(\btiX),\bA)), \\
     &\leq 2M^l \delta_Y + M^l(M_k+M) \delta_{\bX} + \Delta_L^k + \lambda_F d(\tih_k^*(\btiX),\tih^*(\btiX)), \\ 
     &\leq  2M^l \delta_Y + M^l(M_k+M) \delta_{\bX} + \Delta_L^k + \lambda_F (d(\tih_k^*(\btiX),\tiY)+d(\tih^*(\btiX),\tiY)).
\end{align*}

\subsection{Proposition~\ref{prop:coherent}}

The end users randomly pick an arbitrary model among $K$ function classes, 
hence, the downstream loss of the $i$th user would evaluate follows an i.i.d. bounded random variable satisfies 
\begin{align*}
L(\tih^*;\btiD)\leq L(\tih_{k(i)}^*;\btiD) \leq L(\tih^*;\btiD) +2 M^l \delta_Y + M^l(M^*+\tilde{M}) \delta_{\bX} + \Delta_L^*.
\end{align*}
Proposition ~\ref{prop:coherent} holds due to this bound and the Popoviciu's inequality on variances.

\subsection{Corollary~\ref{cor:sep_dominating}}

Recall the conditional maximal correlation in Definition~\ref{def:cond_hgr}, i.e., $\rho_A(X,Y)=\sup_{a\in{\cal A};P(a)>0}\rho_{A=a}(X,Y)$. It is straightforward to see that $d_{A=a}(X,Y)=\sqrt{1-\rho_{A=a}(X,Y)}$ inherits Lemma~\ref{lem:triangle} for the binary pivot variable. Hence, it follows that 
\begin{align*}
   \rho_Y(h^*(\bX),A) - \rho_Y(\tih^*(\btiX),A)&=\sup_y \rho_{Y=y}(h^*(\bX),A)- \sup_y \rho_{Y=y}(\tih^*(\btiX),A), \\
   &\leq \sup_{y} \rho_{Y=y}(h^*(\bX),A) - \rho_{Y=y}(\tih^*(\btiX),A), 
\end{align*}
Therefore, following the same proof w.r.t. $\rho_{Y=y}(h^*(\bX),A) - \rho_{Y=y}(\tih^*(\btiX),A)$ in Theorem~\ref{thm:uplwh0}, we state 
\begin{align*}
    \rho_Y(h^*(\bX),A) - \rho_Y(\tih^*(\btiX),A) &\leq \sup_{Y=y} d_{Y=y}(\tih^*(\btiX),h^*(\bX)), \\ 
    &\leq \sup_{Y=y} d_{Y=y}(\tih^*(\btiX),\tiY) + d_{Y=y}(h^*(\bX),\tiY).
\end{align*}

\section{Details of Simulation}
\label{sup:simul}
\subsection{Section~\ref{sec:method_intro}}
\subsubsection*{Figure~\ref{fig:ex1}}
\label{sup:fig_ex1}

First of all, we generate 10 independent replicates of the toy data. There are 4,500/760 training/test data points. The implementation of each method is elaborated as follows. The three downstream models are trained on the transformed training data and evaluated on the transformed test data. For readers who want to know much more details, please refer to the shared code script. 

\noindent {\bf Implementation of \cite{feld:etal:15}.} We implement the geometric repair varying $0\leq\lambda \leq 1$ where $\lambda$ controls the degree of fairness. The 5 equally spaced points of $\lambda$ are considered. Figure~\ref{fig:ex1} draws $\lambda \in \{0.25,0.5,0.75\}$ for clear visualization. As the authors addressed in their work, the label is not transformed. For fair evaluation, we transform both the training and test data.

\noindent {\bf Implementation of \cite{gord;etal;19}.} We refer to the author's open-source and follow their experiment scheme where $X$ of test data is transformed only. The random repair scheme is used with 5 equally spaced points of $\lambda\in [0,1]$. Figure~\ref{fig:ex1} draws $\lambda \in \{0.25,0.5,0.75\}$ for clear visualization, and the label is not also transformed as well.

\noindent {\bf Implementation of ours.} The following structures of neural networks are used. 
\begin{itemize}
    \item $h$ consists of 4 hidden layers each of which has a dense block with 64 hidden nodes and BN-ReLU-Dropout layers in order. The final layer has only a dense layer having the same dimension with $Y$ without any further layers. 
    \item $G_{\bX}$ consists of 3 hidden layers each of which has a dense block with 64 hidden nodes and BN-ReLU-Dropout layers in order. The final layer has only a dense layer having the same dimension with $X$ without any further layers. 
    \item $G_Y$ consists of 3 hidden layers each of which has a dense block with 64 hidden nodes and BN-ReLU-Dropout layers. The final layer has only a dense layer having the same dimension with $Y$ without any further layers. 
    \item $V$  consists of 3 hidden layers each of which has a dense block with 64 hidden nodes and ReLU-Dropout layers in order. The final layer has only a dense layer whose output size is 1 without any further layers. 
\end{itemize}
Algorithm~\ref{alg} is implemented for 500 epochs with 200 batch size. For all the neural networks, the Adam optimizer with $\beta_1=0$ and $\beta_2=0.999$ is employed. The marginal constraints $\Delta_X$ and $\Delta_Y$ take the mean-square error (MSE) with $\delta_X=0.05$ and $\delta_Y \in \{1,2,4\}$, and $l$ is also given MSE. We also put $\lambda_F=10$. We do not carefully find the best hyperparameter setting for this illustration. 

\subsection{Section~\ref{sec:simul}}
\label{supp:simulation}


Benchmark data sets are publicly available. This work imports Adult data from the open-source platform\footnote{https://github.com/UCLA-Trustworthy-AI-Lab/FairRR\label{foot:fairr}} that reproduces the FairRR method. ACSEmployment and ACSPublicCoverage in California are imported from the Python library \texttt{folktables}\footnote{https://github.com/socialfoundations/folktables}. For each benchmark data, we generate 5 pairs of training and validation data sets. Note that 20\% of the original data are randomly selected for validation data. Similar to the previous section, downstream models are fitted on the transformed train data and then evaluated on the transformed validation data. 
The implementation of each method is elaborated as follows. 

\noindent {\bf Implementation of Reweighing \cite{kami:cald:12}.} We refer to the \texttt{AIF360} library in Python to implement Reweighing and calculate the fairness and utility metrics shown in the main text. 

\noindent {\bf Implementation of FairRR \cite{zeng:etal:24}.} We refer to the author's open-source\footref{foot:fairr}. Since the author's work does not handle ACSPublicCoverage and ACSEmployment, we find the \texttt{level} parameters by ourselves such that the FairRR methods with logistic regression (LR) and random forest (RF) in the main text draw trade-off plots that are comparable to ours. 
\begin{itemize}
    \item For Adult, we set 
    \begin{itemize}
        \item \texttt{DP}: level $\in \{0.0,0.1,0.14\}$ for LR and  $\in \{0.0,0.1,0.14\}$ for RF.
        \item \texttt{EO}: level $\in \{0.0,0.02,0.06\}$ for LR and  $\in \{0.0,0.01,0.02\}$ for RF.
    \end{itemize}
    \item For ACSPublicCoverage, we set 
    \begin{itemize}
        \item \texttt{DP}: level $\in \{0.0,0.14,0.18\}$ for LR and  $\in \{0.0,0.06,0.14\}$ for RF.
        \item \texttt{EO}: level $\in \{0.0,0.14,0.18\}$ for LR and  $\in \{0.0,0.06,0.14\}$ for RF.
    \end{itemize}
    \item For ACSEmployment, we set 
    \begin{itemize}
        \item \texttt{DP}: level $\in \{0.02,0.1,0.18\}$ for LR and  $\in \{0.02,0.1,0.18\}$ for RF.
        \item \texttt{EO}: level $\in \{0.02,0.1,0.18\}$ for LR and  $\in \{0.001,0.002,0.003\}$ for RF.
    \end{itemize}
\end{itemize}
To see further details, please refer to the public code of the original work. 

\noindent {\bf Implementation of FairTabGAN \citep{raja:etal:22}.} For fair comparison, the generator of FairTabGAN adopts the network structure of our converter, where the input noise dimension is in $\mathbb{R}^{32}$, and the discriminator uses the same structure as the upstream network just by adapting the input dimension. We set 200/100/100 epochs for Adule/ACSEmployment/ACSPublicCoverage with 500 batch size. For each data, the 10\% of the epochs is used for the post-optimization for fairness. The penalty parameter for fairness is set as follows: 
\begin{itemize}
    \item For Adult/ACSPublicCoverage/ACSEmployment, we set the trade-off tuning parameter
    \begin{itemize}
        \item $\lambda\in \{0.01,1.0,5.0\}$ for $A_{\text{bin}}$;
        \item $\lambda\in \{0.01,0.1,0.5\}$ for ${\bf A}_{\text{dis}}$.
    \end{itemize}
\end{itemize}
These configurations are specified so that the training process can successfully find the trade-off curves for comparison with other methods. 

\noindent {\bf Implementation of MultiFair \citep{tian:etal:24}.} We implement the `Mixup via interpolation scheme' that finds a neutral data instance $x_m:=\lambda x+(1-\lambda)x_c$ where $x$ and $x_c$ stand for the original data instance and the interpolated instance across different sensitive attributes under their scheme. To see more details about the framework, refer to \cite{tian:etal:24}. To randomly pick the mixing weight $\lambda$, suggested by the work, we employ $\lambda\sim {\rm Beta}(1,\alpha)$. Provided that $\bE[\lambda]=1/(\alpha+1)$, a higher $\alpha$ tends to encourages stronger fairness restriction. For all data sets, we choose $\alpha \in \{0.1,1,10\}$. To handle categorical variables, we first transform them into one-hot encoded vectors and apply the mixup scheme. When creating a whole data set, the most probable category is randomly selected from the interpolated probability vector as the one-hot encoding form for fair comparison with others.

\noindent {\bf Implementation of ours.} 
The specific configuration of solving \eqref{opt:minmax} is as follows. The auxiliary model $h$ is trained by the generalized cross-entropy loss. We set $T'=1$ in Algorithm~\ref{alg}. The metrics $\Delta_{\bX}$ and $\Delta_{Y}$ are defined as the mean-absolute error and the categorical hinge loss, for continuous and categorical variables respectively. The fairness budget $\delta_X=\delta_{X_i}=\cdots=\delta_{X_{d_{\bX}}}$ is assigned to each variable $i=1,\dots,d_{\bX}$, i.e., each $X_i$ has $\Delta_{X_i}(X_i,\tiX_i)\leq \delta_{X_i}$ and $\delta_{X_i}$ is the same for all variables. For each data, we consider three fairness budgets with respect to $\delta_{X}$; Budgets 1,2, and 3 in ascending order of $\delta_X$. Supported by Remark~\ref{remark:ex}, $\delta_{Y}>0$ is used for PublicCoverage since it has relatively weak predictive accuracy as shown in Table~\ref{tab:basic_score}. We choose $\lambda_F\in \{0.1,1,10,100\}$ that achieves a better fairness-utility trade-off that varies with the size of $\delta_{X}$.

Let $G_{\bX,j}$ be the output of $G_{\bX}(\bX,\bA)$ corresponding to the $j$th variable. For instance, suppose $\bX=(X_1,X_2,X_3)^{\top}$ where $X_1,X_2$ are 1-dimensional continuous variables and $X_3$ is a one-hot encoded discrete variable with 3 factors, and its converted version $\btiX=(\tilde{X}_1,\tilde{X}_2,\tilde{X}_3)$. Then we say $\tilde{X}_j=G_{\bX,j}$ and $\tilde{X}_1,\tilde{X}_2 \in \mathbb{R},\tilde{X}_3 \in \mathbb{R}^3$. Note that all the benchmark tabular data sets consist of continuous and one-hot encoded variables. To generate such a mix of discrete and continuous variables, we set $G_{\bX,j} = (G_{L,j} \circ G_I)(\bX,\bA)$ where all $G_{\bX,j}$ share $G_I$ but each has their own output layer $G_{L,j}$. For instance, discrete variables are processed using the Gumbel-Softmax layer \cite{jang:etal:16}, while continuous variables undergo an affine transformation.

The following structures of neural networks are used. 
\begin{itemize}
    \item Neural networks for this example use a residual block that follows: Input-Dense-BatchNorm-Activation-Dropout-Dense-BatchNorm-Output-Add(Input,Output)-Activation.
    \item $h$ consists of 4 residual blocks and a dense layer for the last output with the softmax activation. 
    \item $G_{I}$ consists of 4 residual blocks with dropout. Then each variable has the last dense layer $G_{L,j}$. 
    \item $V$ consists of 4 residual blocks and a dense layer for the last output. 
\end{itemize}
Algorithm~\ref{alg} is implemented for 400/200/150 epochs for Adult/PublicCoverage/Employment data sets with 500 batch size respectively. For all the neural networks, the Adam optimizer with $\beta_1=0$ and $\beta_2=0.999$ is employed. For the marginal constraint $\Delta_X$, discrete and continuous variables adopt categorical hinge loss and mean absolute error, respectively. The binary outcome variable also uses the categorical hinge loss. 

\noindent {\bf Implementation of HGR \citep{lee:etal:22}.} We refer to the implementation provided in https://github.com/jwsohn612/fairSBP. For fair comparison, we set the classification model as our upstream model. We set 200/100/100 epochs for Adule/ACSEmployment/ACSPublicCoverage with 500 batch size and put $\lambda\in\{0.1,0.6,0.9\}$ weighting $\lambda \text{EmpiricalLoss} + (1-\lambda)\text{FairnessLoss}$ to draw the trade-off curve. 

\noindent {\bf Implementation of ExpoGrad \cite{agar:etal:18}.} We refer to the \texttt{FairLearn} library in Python to implement ExponentiatedGradientReduction. The levels of the ratio constraints are specified as follows. For all data sets, we change the $\lambda:=\text{`ratio\_bound'}$ of \texttt{DemographicParity} for fairness control by $\lambda\in\{0.2,0.7,0.9\}$.

\newpage

\subsection{Experiments on CelebA}
\label{supp:simul:celeba}

For our simulation study, we import $128\times 128 \times 3$ CelebA images from Kaggle\footnote{https://www.kaggle.com/datasets/jessicali9530/celeba-dataset}. Since the original data has a sufficiently large number of training images, we separate 20000 randomly selected images for validation data. For network structures, we use well-known models: ResNet18 for $h$ and U-Net autoencoder for $G_{\bX}$. $V$ uses the same structure used in the tabular data example. We adopted the early stopping procedure based on 2000 validation images when training the downstream ResNet18 model. 

\noindent {\bf Implementation of ours.} In this example, $\delta_y=0$ is set because the original data produces the ResNet18 model with relatively higher accuracy. We run Algorithm~\ref{alg} for 30 epochs with 64 batch size.

\noindent {\bf Implementation of FAAP \cite{wang:etal:22}.} We modify Algorithm~\ref{alg} by replacing $\bar{h}^{(t,t')}$ with pre-trained $h^*$, and this pre-trained classification model is fixed during the training of other networks. This scheme is conceptually equivalent to FAAP although the penalty construction is different. We believe that this is a fair comparison because the competing methods have the same architectural backgrounds except for the core penalization step. 

\subsection{A variant of Algorithm~\ref{alg} with variational loss}
\label{supp:variational loss}

In this work, we consider the following optimization with the $f$-divergence $D_{f}$: 
\begin{align*}
    \min_{G \in {\cal G}}\max_{V\in {\cal V}} \bE[l(\barY,\barh^*(\bbarX))] + \lambda_F R_V(\tih^*(\bbarX),\bA) + \kappa D_{f}(\bX,\bbarX),
\end{align*}
which can be seen as 
\begin{align*}
    \min_{G \in {\cal G}}\max_{V\in {\cal V}, E \in {\cal E}} \bE[l(\barY,\barh^*(\bbarX))] + \lambda_F R_V(\tih^*(\bbarX),\bA) + \kappa R_E(\bX,\bbarX), 
\end{align*}
where $R_E(\bX,\btiX) = \sup_{E} \bE[E(\bX)] - \bE[f^*(E(\btiX))]$. To write the algorithm for the case of $\delta_{Y}=0$, we denote by $L(G_{\bX};G_Y,\bar{h}^*,V,E,\lambda_{\bX})=\bE[l(\barY,\barh^*(\bbarX))] + \lambda_F R_V(\bar{h}^*(\bbarX),\bA)+\lambda_{\bX}\tau(\Delta_{\bX}(\bX,\bbarX)-\delta_{\bX})+\kappa R_E(\bX,\bbarX)$ and $L(G_Y;G_X,\bar{h}^*,\lambda_Y)=\bE[l(\barY,\barh^*(\bbarX))]+\lambda_Y\tau(\Delta_{Y}(Y,\barY)-\delta_{Y})$ the target objective for $G_{\bX}$ and $G_Y$ respectively where $\tau(x)=\max\{x,0\}$. Algorithm~\ref{supp:alg_with_vari} shows the modified Algorithm~\ref{alg} with the extra variational loss. 

\begin{algorithm}[ht!]
\caption{Task-tailored Pre-processing with the Extra Loss}
\label{supp:alg_with_vari}
\begin{algorithmic}[1]
\REQUIRE Data $(\mathbf{X}, \mathbf{A}, Y)$; initialized neural nets $\bar{h}^{(0)}$, $V^{(0)}$, $G^{(0)} = (G_{\mathbf{X}}^{(0)}, G_Y^{(0)})$; variables $\lambda_{\mathbf{X}}^{(0)} \gets 0$, $\lambda_Y^{(0)} \gets 0$; learning rate $r_{\cdot} > 0$ for each component; total number of iterations $T$.
\FOR{$t = 1$ to $T$}
    \STATE $(\bar{\mathbf{X}}^{(t)}, \bar{Y}^{(t)}) = (G_{\mathbf{X}}^{(t)}(\mathbf{X}, \mathbf{A}), G_Y^{(t)}(\mathbf{X}, \mathbf{A}, Y))$
    \vspace{0.2cm}
    \STATE $\bar{h}^{(t+1)} \gets \bar{h}^{(t)} - r_h \partial {\bE}[L(\bar{Y}^{(t)}, \bar{h}(\bar{\mathbf{X}}^{(t)}))]/\partial \bar{h}^{(t)}$
    \vspace{0.2cm}
    \STATE $V^{(t+1)} \gets V^{(t)} + r_V \partial R_{V^{(t)}}(\bar{h}^{(t+1)}(\bar{\mathbf{X}}^{(t)}), \mathbf{A})/\partial V^{(t)}$
    \vspace{0.2cm}
    \STATE $E^{(t+1)} \gets E^{(t)} + r_E \partial R_{E^{(t)}}(\bX,\bar{\mathbf{X}}^{(t)}))/\partial E^{(t)}$
    \vspace{0.2cm}
    \STATE $\lambda_{\mathbf{X}}^{(t+1)} \gets \lambda_{\mathbf{X}}^{(t)} + r_{\mathbf{X}} \tau(\Delta_{\mathbf{X}}(\mathbf{X}, \bar{\mathbf{X}}^{(t)}) - \delta_{\mathbf{X}})$
    \vspace{0.2cm}
    \STATE $\lambda_Y^{(t+1)} \gets \lambda_Y^{(t)} + r_Y \tau(\Delta_Y(Y, \bar{Y}^{(t)}) - \delta_Y)$
    \vspace{0.2cm}
    \STATE $ G_{\mathbf{X}}^{(t+1)} \gets G_{\mathbf{X}}^{(t)} - r_G \partial L(G_{\bX}^{(t)};G_Y^{(t)},\bar{h}^{(t+1)},V^{(t+1)},E^{(t+1)},\lambda_{\bX}^{(t+1)})/\partial  G_{\mathbf{X}}^{(t)}$ 
    \vspace{0.2cm}
    \STATE $G_Y^{(t+1)} \gets G_Y^{(t)} - r_G\partial L(G_Y^{(t)};G_{\mathbf{X}}^{(t)},\bar{h}^{(t+1)},\lambda_Y^{(t+1)})/\partial G_Y^{(t)}$
    \vspace{0.2cm}
\ENDFOR
\RETURN $G^{(T)}$
\end{algorithmic}
\end{algorithm}

We illustrate the results of Algorithm~\ref{supp:alg_with_vari} in Figure~\ref{fig:celeba_eo2}. The figure compares the generated images from Algorithm~\ref{alg}, highlighting that using the variational loss during the training can improve the quality of images. Note Figure~\ref{fig:celeba_eo} shows that male images are more distorted than female images, so we present male images in this figure.

\begin{figure*}[ht!] 
    \centering
    \subfloat[Results of Algorithm~\ref{alg}: $Y=1$ and $A=1$]{%
        \includegraphics[width=0.9\linewidth]{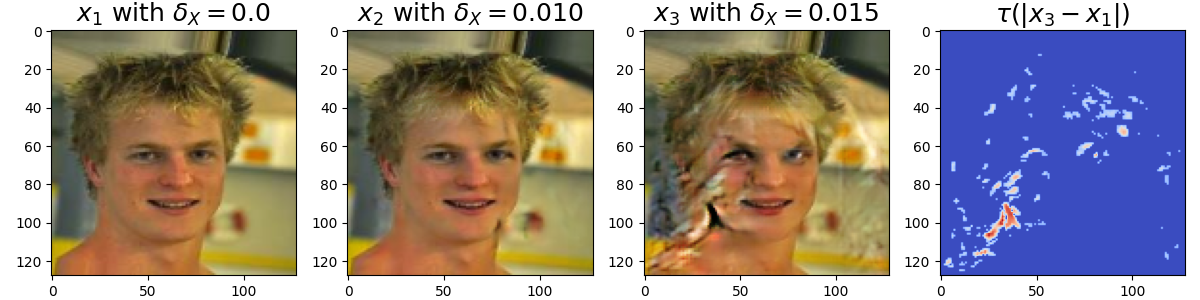}
    }\hfill
    \subfloat[Results of Algorithm~\ref{supp:alg_with_vari}: $Y=1$ and $A=1$]{%
        \includegraphics[width=0.9\linewidth]{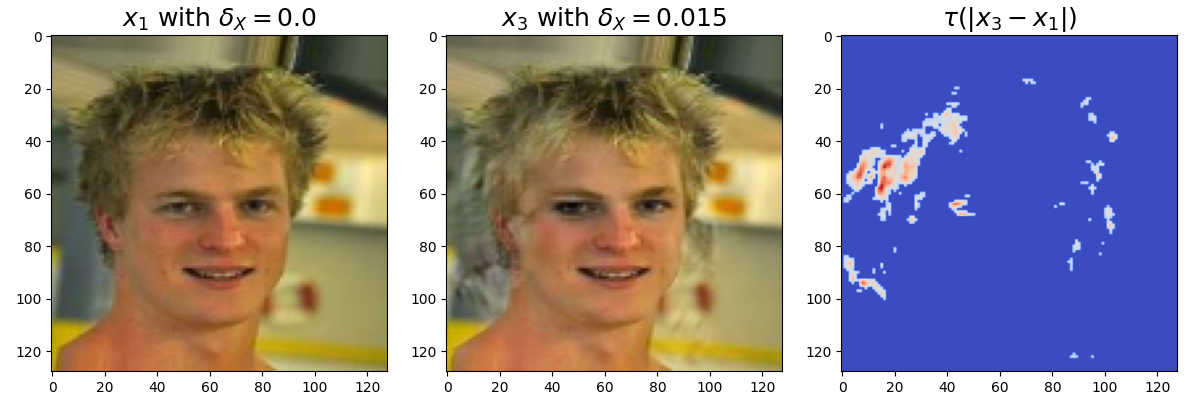}
    }\hfill
    \subfloat[Results of Algorithm~\ref{alg}: $Y=0$ and $A=1$]{%
        \includegraphics[width=0.9\linewidth]{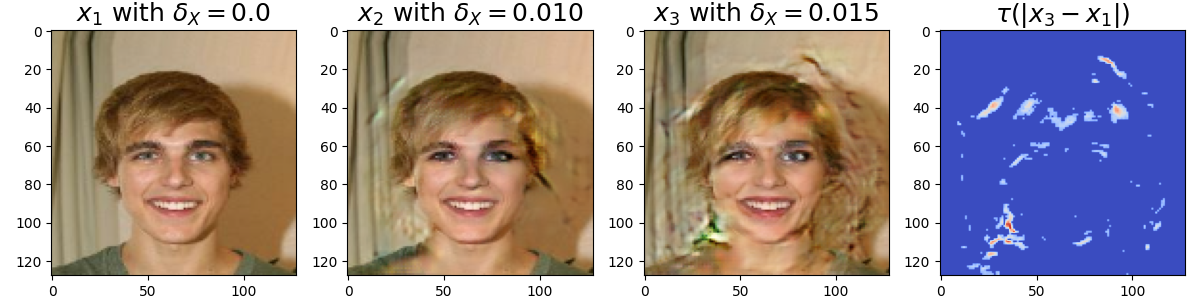}
    }\hfill
    \subfloat[Results of Algorithm~\ref{supp:alg_with_vari}: $Y=0$ and $A=1$]{%
        \includegraphics[width=0.9\linewidth]{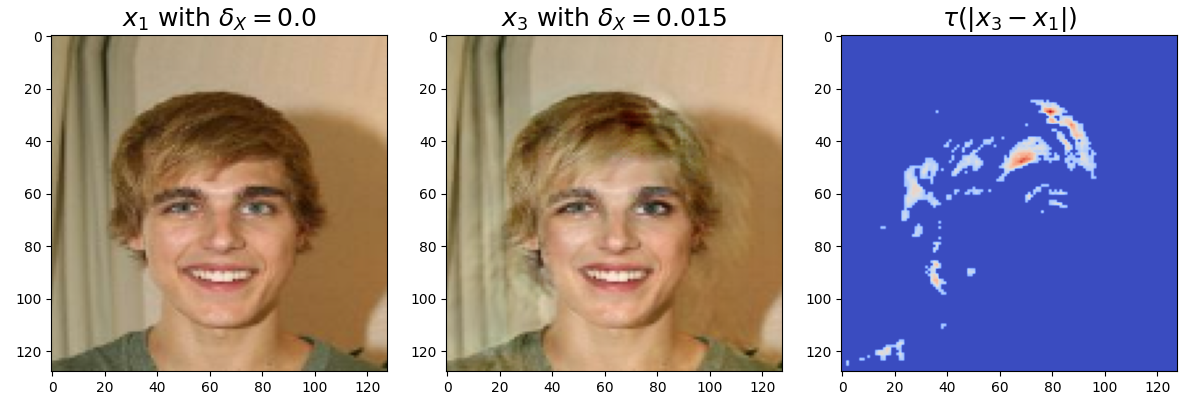}
    }\hfill
    \caption{Comparison between Algorithms~\ref{alg} and \ref{supp:alg_with_vari}}
    \label{fig:celeba_eo2}
\end{figure*}

    \renewcommand{\thesection}{\arabic{section}}
    \renewcommand{\theequation}{\arabic{equation}}
    \renewcommand{\thefigure}{\arabic{figure}}
    \renewcommand{\thetable}{\arabic{table}}
\fi

\end{document}